\documentclass{article}

\usepackage[final,nonatbib]{neurips_2022}

\usepackage[utf8]{inputenc}
\usepackage{float}
\usepackage{stmaryrd}
\usepackage{amsmath,amsfonts,bm,amssymb,amsthm}
 \usepackage{mathtools}

\usepackage{optidef}
\usepackage{breqn}
\usepackage{placeins}
\usepackage{url}
\usepackage{xfrac}
\usepackage{soul}
\usepackage{makecell}
\usepackage{xspace}
\usepackage[dvipsnames]{xcolor}
\usepackage[utf8]{inputenc} 
\usepackage[T1]{fontenc}    
\usepackage[pagebackref=true,colorlinks=true,linkcolor=cyan,citecolor=blue]{hyperref}
\renewcommand*{\backrefalt}[4]{%
    \ifcase #1 \footnotesize{(Not cited.)}%
    \or        \footnotesize{(p.~#2)}%
    \else      \footnotesize{(pp.~#2)}%
    \fi}

\usepackage{url}            
\usepackage{booktabs}       
\usepackage{amsfonts}       
\usepackage{nicefrac}       
\usepackage{microtype}      
\usepackage{graphicx}
\usepackage{rotating}
\usepackage{booktabs}
\usepackage{float}
\usepackage{wrapfig}
\usepackage[export]{adjustbox}
\usepackage{optidef}
\usepackage{stmaryrd}
\usepackage{cleveref}
\usepackage{caption}
\usepackage{subcaption}
\usepackage{wrapfig}
\usepackage{xfrac}
\usepackage{multirow}
\usepackage{pifont}
\newcommand{\cmark}{\ding{51}}%
\newcommand{\xmark}{\ding{55}}%
\usepackage{csquotes}
\usepackage{pifont}
\usepackage{amsthm}
\usepackage[ruled]{algorithm}
\usepackage[noend]{algorithmic}
\usepackage[page]{appendix} 

\newtheorem{theorem}{Theorem}
\newtheorem{remark}{Remark}
\newtheorem*{theorem*}{Theorem}
\newtheorem{assumption}{Assumption}
\newtheorem*{assumption*}{Assumption}
\newtheorem{lemma}{Lemma}
\newtheorem*{lemma*}{Lemma}
\newtheorem{proposition}{Proposition}
\newtheorem*{proposition*}{Proposition}

\Crefname{assumption}{Assumption}{Assumptions}
\Crefname{lemma}{Lemma}{Lemmas}
\Crefname{definition}{Definition}{Definitions}
\Crefname{proposition}{Proposition}{Propositions}

\newcommand*\X{\mathcal{X}}

\newcommand*\U{\mathcal{U}}

\newcommand*\Deltab{\bar{\Delta}}

\newcommand*\biasb{\mathrm{bias}}
\newcommand*\varb{\mathrm{var}}
\newcommand*\covb{\mathrm{cov}}
\newcommand*\M{\mathcal{M}}

\newcommand*{\eg}{e.g.,\@\xspace}
\newcommand*{\versus}{vs.\@\xspace}
\newcommand*{\sut}{s.t.\@\xspace}
\newcommand*{\ie}{i.e.,\@\xspace}
\newcommand*{\iid}{IID\@\xspace}
\newcommand*{\ood}{OOD\@\xspace}
\newcommand*{\wrt}{w.r.t.\@\xspace}
\newcommand*{\aka}{a.k.a.\@\xspace}

\newcommand*{\etc}{etc.\@\xspace}

\let\originalleft\left
\let\originalright\right
\renewcommand{\left}{\mathopen{}\mathclose\bgroup\originalleft}
\renewcommand{\right}{\aftergroup\egroup\originalright}

\def\eqref#1{eq.~\ref{#1}}

\def\1{\bm{1}}

\def\X{\mathcal{X}}
\def\Y{\mathcal{Y}}
\def\U{\mathcal{U}}

\DeclareMathAlphabet{\mathsfit}{\encodingdefault}{\sfdefault}{m}{sl}
\SetMathAlphabet{\mathsfit}{bold}{\encodingdefault}{\sfdefault}{bx}{n}

\newcommand{\reb}[1]{#1} 
\makeatletter
\newcommand{\mytag}[2]{%
  \text{#1}%
  \@bsphack
  \begingroup
    \@onelevel@sanitize\@currentlabelname
    \edef\@currentlabelname{%
      \expandafter\strip@period\@currentlabelname\relax.\relax\@@@%
    }%
    \protected@write\@auxout{}{%
      \string\newlabel{#2}{%
        {#1}%
        {\thepage}%
        {\@currentlabelname}%
        {\@currentHref}{}%
      }%
    }%
  \endgroup
  \@esphack
}
\makeatother

\makeatletter
\renewcommand\maketitle{\par
  \begingroup
    \@maketitle
    \@thanks
    \@notice
    \thispagestyle{empty}
  \endgroup
  \setcounter{footnote}{0}%
}
\makeatother

\begin{document}
\title{Diverse Weight Averaging for Out-of-Distribution Generalization}
%
%

\author{%
  Alexandre~Ramé\textsuperscript{1,*}, Matthieu~Kirchmeyer\textsuperscript{1,2,*} \\ \textbf{Thibaud~Rahier\textsuperscript{2}, Alain~Rakotomamonjy\textsuperscript{2,4}, Patrick~Gallinari\textsuperscript{1,2}, Matthieu~Cord\textsuperscript{1,3}}\\
  \textsuperscript{1}Sorbonne Université, CNRS, ISIR, F-75005 Paris, France
  ~~\textsuperscript{2}Criteo AI Lab, Paris, France \\
  \textsuperscript{3}Valeo.ai, Paris, France ~~
  \textsuperscript{4}Université de Rouen, LITIS, France \\
  \textsuperscript{*}Equal contribution\\
}

%
\maketitle
\begin{abstract}
    Standard neural networks struggle to generalize under distribution shifts in computer vision.
\reb{Fortunately, combining multiple networks can consistently improve out-of-distribution generalization. In particular, weight averaging (WA) strategies were shown to perform best on the competitive DomainBed benchmark; they directly average the weights of multiple networks despite their nonlinearities.}
In this paper, we propose Diverse Weight Averaging (DiWA), \reb{a new WA strategy whose} main motivation is to increase the functional diversity across averaged models.
To this end, DiWA averages weights obtained from several independent training runs:
indeed, models obtained from different runs are more diverse than those collected along a single run thanks to differences in hyperparameters and training procedures.
We motivate the need for diversity 
by a new bias-variance-covariance-locality decomposition of the expected error, exploiting similarities between WA and standard functional ensembling.
Moreover, this decomposition highlights that WA succeeds when the variance term dominates, which we show occurs when the marginal distribution changes at test time.
Experimentally, DiWA consistently improves the state of the art on DomainBed without inference overhead.
\end{abstract}
\section{Introduction}
Learning robust models that generalize well is critical for many real-world applications \cite{Zech2018,degrave2021ai}.
Yet, the classical Empirical Risk Minimization (ERM) lacks robustness to distribution shifts \cite{hendrycks2018benchmarking,NEURIPS2020_6cfe0e61,d2020underspecification}.
To improve out-of-distribution (\ood) generalization in classification, several recent works proposed to train models simultaneously on multiple related but different domains \cite{muandet2013domain}.
Though theoretically appealing, domain-invariant approaches \cite{peters2016causal} either underperform \cite{arjovsky2019invariant,krueger2020utofdistribution} or only slightly improve \cite{coral216aaai,rame_fishr_2021} ERM on the reference DomainBed benchmark \cite{gulrajani2021in}.
The state-of-the-art strategy on DomainBed is currently to average the weights obtained along a training trajectory \cite{izmailov2018}.
\cite{cha2021wad} argues that this weight averaging (WA) succeeds in \ood because it finds solutions with flatter loss landscapes.

In this paper, we show the limitations of this flatness-based analysis and provide a new explanation for the success of WA in \ood.
It is based on WA's similarity with ensembling \cite{Lakshminarayanan2017}, a well-known strategy to improve robustness \cite{Ovadia2019,ashukha2020pitfalls}, that averages the predictions from various models.
Based on \cite{ueda1996generalization}, we present 
a bias-variance-covariance-locality decomposition of WA's expected error.
It contains four terms:
\textit{first} the bias that we show increases under shift in label posterior distributions (\ie correlation shift \cite{ye2021odbench});
\textit{second}, the variance that we show increases under shift in input marginal distributions (\ie diversity shift \cite{ye2021odbench});
\textit{third}, the covariance that decreases when models are diverse;
\textit{finally}, a locality condition on the weights of averaged models.

Based on this analysis, we aim at obtaining diverse models whose weights are averageable with our Diverse Weight Averaging (DiWA) approach.
In practice, DiWA averages in weights the models obtained from independent training runs that share the same initialization. The motivation is that those models are more diverse than those obtained along a single run \cite{Fort2019DeepEA,gontijolopes2022no}.
Yet, averaging the weights of independently trained networks with batch normalization \cite{batchnorm2015} and ReLU layers \cite{agarap2018deep} may be counter-intuitive.
Such averaging is efficient especially when models can be connected linearly in the weight space via a low loss path. Interestingly, this linear mode connectivity property \cite{Frankle2020} was empirically validated when the runs start from a shared pretrained initialization \cite{Neyshabur2020}.
This insight is at the heart of DiWA but also of other recent works \cite{Wortsman2022robust,mergefisher21,Wortsman2022ModelSA}, as discussed in \Cref{sec:related_work}.

In summary, our main contributions are the following:
\begin{itemize}
    \item We propose a new theoretical analysis of WA for \ood based on a bias-variance-covariance-locality decomposition of its expected error (\Cref{sec:theory}). By relating correlation shift to its bias and diversity shift to its variance, we show that WA succeeds under diversity shift.
    \item We empirically tackle the covariance term by increasing the diversity across models averaged in weights. In our DiWA approach, we decorrelate their training procedures: in practice, these models are obtained from independent runs (\Cref{sect:dwa}).
    We then empirically validate that diversity improves \ood performance (\Cref{sect:analysis}) and show that DiWA is state of the art on all real-world datasets from the DomainBed benchmark \cite{gulrajani2021in} (\Cref{sec:domainbed}).%
\end{itemize}%
\section{Theoretical insights}
\label{sec:theory}
Under the setting described in \Cref{subsec:notations}, we introduce WA in \Cref{sec:limithessian} and decompose its expected \ood error in \Cref{subsec:biasvariance}.
Then, we separately consider the four terms of this bias-variance-covariance-locality decomposition in \Cref{subsec:analysisbvc}.
This theoretical analysis will allow us to better understand when WA succeeds, and most importantly, how to improve it empirically in \Cref{sect:dwa}.%
\subsection{Notations and problem definition}%
\label{subsec:notations}
\paragraph{Notations.}%
We denote $\X$ the input space of images, $\Y$ the label space
and $\ell:\Y^2 \rightarrow \mathbb{R}_+$ a loss function.
$S$ is the training (source) domain with distribution $p_S$, and $T$ is the test (target) domain with distribution $p_T$.
For simplicity, we will indistinctly use the notations $p_S$ and $p_T$ to refer to the joint, posterior and marginal distributions of $(X, Y)$.
We note $f_S, f_T:\X\rightarrow\Y$ the source and target labeling functions.
We assume that there is no noise in the data: then $f_S$ is defined on $\X_S \triangleq \{x\in \X / p_S(x)>0\}$ by $\forall (x, y) \sim p_S,f_S(x)=y$ and similarly $f_T$ is defined on $\X_T \triangleq \{x\in \X / p_T(x)>0\}$ by $\forall (x, y) \sim p_T, f_T(x)=y$.%
\paragraph{Problem.}
We consider a neural network (NN) $f(\cdot, \theta):\X\rightarrow\Y$ made of a fixed architecture $f$ with weights $\theta$.
We seek $\theta$ minimizing the target generalization error:
\begin{equation} \label{eq:gen_error}
    \mathcal{E}_T(\theta)= \mathbb{E}_{(x,y) \sim p_T}[\ell(f(x, \theta),y)].
\end{equation}
$f(\cdot, \theta)$ should approximate $f_T$ on $\X_T$. However, this is complex in the \ood setup because we only have data from domain $S$ in training, related yet different from $T$.
The differences between $S$ and $T$ are due to distribution shifts (\ie the fact that
$p_S(X,Y) \neq p_T(X,Y)$) which are decomposed per \cite{ye2021odbench} into
\textbf{diversity shift} (\aka covariate shift), when marginal distributions differ (\ie $p_S(X) \neq p_T(X)$), and
\textbf{correlation shift} (\aka concept shift), when posterior distributions differ (\ie $p_S(Y|X) \neq p_T(Y|X)$ and $f_S \neq f_T$).
The weights are typically learned on a training dataset $d_S$ from $S$ (composed of $n_S$ i.i.d. samples from $p_S(X,Y)$) with a configuration $c$, which contains all other sources of randomness in learning (\eg initialization, hyperparameters, training stochasticity, epochs, \etc). 
We call $l_S=\{d_S, c\}$ a learning procedure on domain $S$, and explicitly write $\theta(l_S)$ to refer to the weights obtained after stochastic minimization of $1/n_S \sum_{(x,y)\in d_S} \ell(f(x, \theta), y)$ \wrt $\theta$ under $l_S$.
\subsection{Weight averaging for OOD and limitations of current analysis}%
\label{sec:limithessian}%
\paragraph{Weight averaging.}%
We study the benefits of combining $M$ individual member weights $\{\theta_m\}_{m=1}^M \triangleq \{\theta(l_S^{(m)})\}_{m=1}^M$ obtained from $M$ (potentially correlated) identically distributed (i.d.) learning procedures $L_S^M\triangleq\{l_S^{(m)}\}_{m=1}^M$.
Under conditions discussed in \Cref{subsec:sharedinithpws}, these $M$ weights can be averaged despite nonlinearities in the architecture $f$.
Weight averaging (WA) \cite{izmailov2018}, defined as:
\begin{equation} \label{eq:f_wa}
    f_{\text{WA}} \triangleq f(\cdot, \theta_{\text{WA}}), \text{~where~} \theta_{\text{WA}}\triangleq\theta_{\text{WA}}(L_S^M) \triangleq 1/M \sum\nolimits_{m=1}^{M} \theta_m,%
\end{equation}
is the state of the art \cite{cha2021wad,arpit2021ensemble} on DomainBed \cite{gulrajani2021in} when the weights $\{\theta_m\}_{m=1}^M$ are sampled along a single training trajectory (a description we refine in \Cref{remark:identical_assumption} from \Cref{app:proof_bvc}).
\paragraph{Limitations of the flatness-based analysis.}
To explain this success, Cha \textit{et al.} \cite{cha2021wad} argue that flat minima generalize better; indeed, WA flattens the loss landscape.
Yet, as shown~in~\Cref{app:limit_proof_swad}, this analysis does not fully explain WA's spectacular results on DomainBed.
First, flatness does not act on distribution shifts thus the \ood error is uncontrolled with their upper bound (see \Cref{app:flatness_distribution_shifts}).
Second, this analysis does not clarify why WA outperforms \reb{Sharpness-Aware Minimizer (SAM) \cite{foret2021} for \ood generalization, even though SAM} directly optimizes flatness (see \Cref{app:mav_better_than_sam}).
Finally, it does not justify why combining WA and SAM succeeds in \iid \cite{questionsflatminima22} yet fails in \ood (see \Cref{app:mav_sam_failure}).
These observations motivate a new analysis of WA; we propose one below that better explains these~results. %
\subsection{Bias-variance-covariance-locality decomposition}
\label{subsec:biasvariance}
We now introduce our bias-variance-covariance-locality decomposition which extends the bias-variance decomposition \cite{kohavi1996bias} to WA.
In the rest of this theoretical section, $\ell$ is the Mean Squared Error for simplicity: yet, our results may be extended to other losses as in \cite{domingos2000unified}.
In this case, the expected error of a model with weights $\theta(l_S)$ \wrt the learning procedure $l_S$ was decomposed in \cite{kohavi1996bias} into:
\begin{equation} \tag{BV}\label{eq:b_v}
    \mathbb{E}_{l_S}\mathcal{E}_T(\theta(l_S)) = \mathbb{E}_{(x,y)\sim p_T}[\biasb^2(x, y)+\varb(x)],
\end{equation}
where $\biasb(x,y), \varb(x)$ are the bias and variance of the considered model \wrt a sample $(x,y)$, defined later in \Cref{eq:b_var_cov}.
To decompose WA's error, we leverage the similarity (already highlighted in \cite{izmailov2018}) between WA and functional ensembling (ENS) \cite{Lakshminarayanan2017,dietterich2000ensemble}, a more traditional way to combine a collection of weights.
More precisely, ENS averages the predictions, $f_{\text{ENS}} \triangleq f_{\text{ENS}}(\cdot, \{\theta_m\}_{m=1}^M) \triangleq 1/M \sum_{m=1}^{M} f(\cdot, \theta_m)$.
\Cref{lemma:wa_ensembling} establishes that $f_{\text{WA}}$ is a first-order approximation of $f_{\text{ENS}}$ when $\{\theta_m\}_{m=1}^M$ are close in the weight space.
\begin{lemma}[WA and ENS. Proof in \Cref{app:wa_loss}. Adapted from \cite{izmailov2018,Wortsman2022ModelSA}.] \label{lemma:wa_ensembling}
    Given $\{\theta_m\}_{m=1}^M$ with learning procedures $L_S^M\triangleq\{l_S^{(m)}\}_{m=1}^M$. Denoting $\Delta_{L_S^M}=\max_{m=1}^{M}\left\|\theta_m-\theta_{\text{WA}}\right\|_2$, $\forall (x,y) \in \X \times \Y$:%
    \begin{equation*}
        f_{\text{WA}}(x) = f_{\text{ENS}}(x) + O(\Delta^2_{L_S^M}) \text{~and~} \ell\left(f_{\text{WA}}(x), y\right) = \ell\left(f_{\text{ENS}}(x) , y\right) + O(\Delta_{L_S^M}^2).%
    \end{equation*}%
\end{lemma}%
This similarity is useful since \Cref{eq:b_v} was extended into a bias-variance-covariance decomposition for ENS in \cite{ueda1996generalization,brown2005between}.
We can then derive the following decomposition of WA's expected test error. To take into account the $M$ averaged weights, the expectation is over the joint distribution describing the $M$ identically distributed (i.d.) learning procedures $L_S^M\triangleq\{l_S^{(m)}\}_{m=1}^M$.
\begin{proposition}[Bias-variance-covariance-locality decomposition of the expected generalization error of WA in \ood. Proof in \Cref{app:proof_bvc}.]
    \label{prop:b_var_cov}
    Denoting $\bar{f}_S\left(x\right) = \mathbb{E}_{l_S} \left[f\left(x,\theta\left(l_S\right)\right)\right]$, under identically distributed learning procedures $L_S^M\triangleq\{l_S^{(m)}\}_{m=1}^M$, the expected generalization error on domain $T$ of $\theta_{\text{WA}}(L_S^M)\triangleq\frac{1}{M} \sum\nolimits_{m=1}^{M} \theta_m$ over the joint distribution of $L_S^M$ is:%
    \begin{equation}
        \begin{aligned}
             \mathbb{E}_{L_S^M}\mathcal{E}_T(\theta_{\text{WA}}(L_S^M)) &= \mathbb{E}_{(x,y)\sim p_T}\Big[\biasb^2(x, y)+\frac{1}{M} \varb(x)+\frac{M-1}{M} \covb(x)\Big] + O(\Deltab^2), \\
             \text{where~}\biasb(x,y)&=y-\bar{f}_S\left(x\right), \\
             \text{and~}\varb(x)&=\mathbb{E}_{l_S}\left[\left(f(x, \theta(l_S)) - \bar{f}_S\left(x\right)\right)^{2}\right], \\
             \text{and~}\covb(x)&= \mathbb{E}_{l_S,l_S'}\left[\left(f(x,\theta(l_S))-\bar{f}_S\left(x\right)\right)\left(f(x,\theta(l_S')))-\bar{f}_S\left(x\right)\right)\right], \\
             \text{and~}\Deltab^2&=\mathbb{E}_{L_S^M}\Delta_{L_S^M}^2 \text{~with~} \Delta_{L_S^M}=\max_{m=1}^{M}\left\|\theta_m-\theta_{\text{WA}}\right\|_2.%
        \end{aligned} \tag{BVCL}\label{eq:b_var_cov}%
    \end{equation}%
    $\covb$ is the prediction covariance between two member models whose weights are averaged.
    The locality term $\Deltab^2$ is the expected squared maximum distance between weights and their average.
\end{proposition}
\Cref{eq:b_var_cov} decomposes the \ood error of WA into four terms.
The bias is the same as that of each of its i.d. members.
WA's variance is split into the variance of each of its i.d. members divided by $M$ and a covariance term.
The last locality term constrains the weights to ensure the validity of our approximation.
In conclusion, combining $M$ models divides the variance by $M$ but introduces the covariance and locality terms which should be controlled along bias to guarantee low \ood error.
\subsection{Analysis of the bias-variance-covariance-locality decomposition}%
\label{subsec:analysisbvc}
We now analyze the four terms in \Cref{eq:b_var_cov}.
We show that bias dominates under correlation shift (\Cref{subsec:expression_ood_bias}) and variance dominates under diversity shift (\Cref{subsec:expression_ood_var}).
Then, we discuss a trade-off between covariance, reduced with diverse models (\Cref{subsec:expression_cov_div}), and the locality term, reduced when weights are similar (\Cref{subsec:expression_loc_lir}).
This analysis shows that \textit{WA is effective against diversity shift when $M$ is large and when its members are diverse but close in the weight space}.
\subsubsection{Bias and correlation shift (and support mismatch)}
\label{subsec:expression_ood_bias}
We relate \ood bias to correlation shift \cite{ye2021odbench} 
under \Cref{ass:no_bias_iid}, where $\bar{f}_{S}\left(x\right) \triangleq \mathbb{E}_{l_S} \left[f\left(x,\theta\left(l_S\right)\right)\right]$.
As discussed in \Cref{app:bias_correlation_ass_noiidbias}, \Cref{ass:no_bias_iid} is reasonable for a large NN trained on a large dataset representative of the source domain $S$. 
It is relaxed in \Cref{prop:app_bias_full} from \Cref{app:bias_correlation}.
\begin{assumption}[Small \iid bias]%
    $\exists \epsilon > 0 \text{~small~s.t.~} \forall x\in \X_S, |f_{S}\left(x\right)-\bar{f}_{S}\left(x\right)|\leq \epsilon$.%
    \label{ass:no_bias_iid}%
\end{assumption}%
\begin{proposition}[\ood bias and correlation shift. Proof in \Cref{app:bias_correlation}] \label{prop:bias}
    With a bounded difference between the labeling functions $f_T-f_S$ on $\X_T \cap \X_S$, under \Cref{ass:no_bias_iid}, the bias on domain $T$ is:
    \begin{equation}\label{eq:bias}
        \begin{aligned}
            \mathbb{E}_{(x,y)\sim p_T}[\biasb^2(x, y)] &= \text{Correlation shift} + \text{Support mismatch} + O(\epsilon), \\
            \text{where~} \text{Correlation shift} &= \int_{\X_T \cap \X_S} \left(f_T(x)-f_S(x)\right)^{2} p_T(x) dx, \\
            \text{and~} \text{Support mismatch} &= \int_{\X_T \setminus \X_S} \left(f_T(x)-\bar{f}_{S}\left(x\right)\right)^{2} p_T(x) dx.%
        \end{aligned}%
    \end{equation}
\end{proposition}%
We analyze the first term by noting that $f_T(x) \triangleq \mathbb{E}_{p_T}[Y|X=x]$ and $f_S(x) \triangleq \mathbb{E}_{p_S}[Y|X=x],$ $\forall x \in \X_T \cap \X_S$.
This expression confirms that our correlation shift term measures shifts in posterior distributions between source and target, as in \cite{ye2021odbench}.
It increases in presence of spurious correlations: \eg on ColoredMNIST \cite{arjovsky2019invariant} where the color/label correlation is reversed at test time.
The second term is caused by support mismatch between source and target.
It was analyzed in \cite{ruan2022optimal} and shown irreducible in their \enquote{No free lunch for learning representations for DG}.
Yet, this term can be tackled if we transpose the analysis in the feature space rather than the input space. This motivates encoding the source and target domains into a shared latent space, \eg by pretraining the encoder on a task with minimal domain-specific information as in \cite{ruan2022optimal}.

This analysis explains why WA fails under correlation shift, as shown on ColoredMNIST in \Cref{app:failure_corr_shift}.
Indeed, combining different models does \textit{not} reduce the bias.
\Cref{subsec:expression_ood_var} explains that WA is however efficient against diversity shift.%
\subsubsection{Variance and diversity shift}
\label{subsec:expression_ood_var}
Variance is known to be large in \ood \cite{d2020underspecification} and to cause a phenomenon named  underspecification, when models behave differently in \ood despite similar test \iid accuracy.
We now relate \ood variance to diversity shift \cite{ye2021odbench} in a simplified setting.
We fix the source dataset $d_S$ (with input support $X_{d_S}$), the target dataset $d_T$ (with input support $X_{d_T}$) and the network's initialization.
We get a closed-form expression for the variance of $f$ over all other sources of randomness under \Cref{ass:infinite_width,ass:inter_sample}.
\begin{assumption}[Kernel regime] $f$ is in the kernel regime \cite{Jacot2018,daniely2017sgd}.%
\label{ass:infinite_width}%
\end{assumption}%
This states that $f$ behaves as a Gaussian process (GP); it is reasonable if $f$ is a wide network \cite{Jacot2018,Lee2018DeepNN}.
The corresponding kernel $K$ is the neural tangent kernel (NTK) \cite{Jacot2018} depending only on the initialization.
GPs are useful because their variances have a closed-form expression (\Cref{app:nns_as_gps}).
To simplify the expression of variance, we now make \Cref{ass:inter_sample}.
\begin{assumption}[Constant norm and low intra-sample similarity on $d_S$]
    $\exists (\lambda_S, \epsilon) \text{~with~}\linebreak 0 \leq \epsilon \ll \lambda_S \text{~such~that~} \forall x_S \in X_{d_S}, K(x_S,x_S)=\lambda_S \text{~and~} \forall x_S^\prime \neq x_S \in X_{d_S},|K(x_S,x_S^\prime)|\leq \epsilon$.%
    \label{ass:inter_sample}%
\end{assumption}%
This states that training samples have the same norm (following standard practice \cite{Lee2018DeepNN,ah2010normalized,ghojogh2021reproducing,rennie2005}) and weakly interact \cite{He2020The,seleznova2022neural}. This assumption is further discussed and relaxed in \Cref{app:discussion_inter_sample}.
We are now in a position to relate variance and diversity shift when~$\epsilon \to 0$.%
\clearpage
\begin{proposition}[\ood variance and diversity shift. Proof in \Cref{app:var_diversity}] \label{prop:var}%
    Given $f$ trained on source dataset $d_S$ (of size $n_S$) with NTK $K$, under \Cref{ass:infinite_width,ass:inter_sample}, the variance on dataset $d_T$ is:%
    \begin{equation} \label{eq:int_var}%
        \mathbb{E}_{x_T\in X_{d_T}}[\varb(x_T)] = \frac{n_S}{2\lambda_S}\text{MMD}^{2}(X_{d_S}, X_{d_T}) + \lambda_T - \frac{n_S}{2\lambda_S}\beta_T + O(\epsilon),%
    \end{equation}%
    where $\text{MMD}$ is the empirical Maximum Mean Discrepancy in the RKHS of $K^2(x,y)=(K(x,y))^2$;$\lambda_T \triangleq \mathbb{E}_{x_T \in X_{d_T}} K\left(x_T, x_T\right)$ and $\beta_T \triangleq \mathbb{E}_{(x_T,x_T^\prime)\in X_{d_T}^2, x_T \neq x_T^\prime} K^2\left(x_T, x_T^{\prime}\right)$ are the empirical mean similarities respectively measured between identical (\wrt $K$) and different (\wrt $K^2$) samples averaged over $X_{d_T}$.%
\end{proposition}
The MMD empirically estimates shifts in input marginals, \ie between $p_S(X)$ and $p_T(X)$. Our expression of variance is thus similar to the diversity shift formula in \cite{ye2021odbench}: MMD replaces the $L_1$ divergence used in \cite{ye2021odbench}.
The other terms, $\lambda_T$ and $\beta_T$, both involve internal dependencies on the target dataset $d_T$: they are constants \wrt $X_{d_T}$ and do not depend on distribution shifts.
At fixed $d_T$ and under our assumptions, \Cref{eq:int_var} shows that variance on $d_T$ decreases when $X_{d_S}$ and $X_{d_T}$ are closer (for the MMD distance defined by the kernel $K^2$) and increases when they deviate. Intuitively, the further $X_{d_T}$ is from $X_{d_S}$, the less the model's predictions on $X_{d_T}$ are constrained after fitting $d_S$.

This analysis shows that WA reduces the impact of diversity shift as combining $M$ models divides the variance per $M$. This is a strong property achieved \textit{without requiring data from the target domain}.
\subsubsection{Covariance and diversity}
\label{subsec:expression_cov_div}
The covariance term increases when the predictions of $\{f(\cdot, \theta_m)\}_{m=1}^M$ are correlated.
In the worst case where all predictions are identical, covariance equals variance and WA is no longer beneficial.
On the other hand, the lower the covariance, the greater the gain of WA over its members; this is derived by comparing
\Cref{eq:b_v,eq:b_var_cov}, as detailed in \Cref{app:proof_wa_ind}.
It motivates tackling covariance by encouraging members to make different predictions, thus to be functionally diverse.
Diversity is a widely analyzed concept in the ensemble literature \cite{Lakshminarayanan2017}, for which numerous measures have been introduced \cite{kuncheva2003measures,aksela2003comparison,DBLP:journals/corr/abs-1905-00414}.
\reb{In \Cref{sect:dwa}, we aim at decorrelating the learning procedures to increase members' diversity and reduce the covariance term}.

\subsubsection{Locality and linear mode connectivity}
\label{subsec:expression_loc_lir}
To ensure that WA approximates ENS, the last locality term $O(\Deltab^2)$ constrains the weights to be close.
Yet, the covariance term analyzed in \Cref{subsec:expression_cov_div} is antagonistic, as it motivates functionally diverse models.
Overall, to reduce WA's error in \ood, we thus seek a good trade-off between diversity and locality.
In practice, we consider that the main goal of this locality term is to ensure that the weights are averageable despite the nonlinearities in the NN such that WA's error does not explode.
This is why in \Cref{sect:dwa}, we empirically relax this locality constraint and simply require that the weights are linearly connectable in the loss landscape, as in the linear mode connectivity \cite{Frankle2020}. We empirically verify later in \Cref{fig:home0_samediffruns_net_soup} that the approximation $f_{\text{WA}} \approx f_{\text{ENS}}$ remains valid even in this case.

\section{DiWA: Diverse Weight Averaging}%
\label{sect:dwa}%
\subsection{Motivation: weight averaging from different runs for more diversity}%
\label{subsec:diversitybydifferentruns}%
\paragraph{Limitations of previous WA approaches.}
Our analysis in \Cref{subsec:expression_ood_bias,subsec:expression_ood_var} showed that the bias and the variance terms are mostly fixed by the distribution shifts at hand.
In contrast, the covariance term can be reduced by enforcing diversity across models (\Cref{subsec:expression_cov_div}) obtained from learning procedures $\{l_S^{(m)}\}_{m=1}^M$.
Yet, previous methods \cite{cha2021wad,arpit2021ensemble} only average weights obtained along a single run.
This corresponds to highly correlated procedures sharing the same initialization, hyperparameters, batch orders, data augmentations and noise, that only differ by the number of training steps. The models are thus mostly similar: this does not leverage the full potential of WA.

\paragraph{DiWA.}
Our Diverse Weight Averaging approach seeks to reduce the \ood expected error in \Cref{eq:b_var_cov} by decreasing covariance across predictions: DiWA decorrelates the learning procedures $\{l_S^{(m)}\}_{m=1}^M$.
Our weights are obtained from $M\gg 1$ different runs, with diverse learning procedures: these have different hyperparameters (learning rate, weight decay and dropout probability), batch orders, data augmentations (\eg random crops, horizontal flipping, color jitter, grayscaling), stochastic noise and number of training steps.
Thus, the corresponding models are more diverse on domain $T$ per \cite{gontijolopes2022no} and reduce the impact of variance when $M$ is large.
However, this may break the locality requirement analyzed in \Cref{subsec:expression_loc_lir} if the weights are too distant. Empirically, we show that DiWA works under two conditions: shared initialization and mild hyperparameter ranges.%
\subsection{Approach: shared initialization, mild hyperparameter search and weight selection}%
\label{subsec:sharedinithpws}%
\paragraph{Shared initialization.}
The shared initialization condition follows \cite{Neyshabur2020}: when models are fine-tuned from a shared pretrained model, their weights can be connected along a linear path where error remains low \cite{Frankle2020}.
Following standard practice on DomainBed \cite{gulrajani2021in}, our encoder is pretrained on ImageNet \cite{krizhevsky2012imagenet}; this pretraining is key as it controls the bias (by defining the feature support mismatch, see \Cref{subsec:expression_ood_bias}) and variance (by defining the kernel $K$, see \Cref{remark:mmdk}).
Regarding the classifier initialization, we test two methods.
The first is the random initialization, which may distort the features \cite{kumar2022finetuning}.
The second is Linear Probing (LP) \cite{kumar2022finetuning}: it first learns the classifier (while freezing the encoder) to serve as a shared initialization.
Then, LP fine-tunes the encoder and the classifier together in the $M$ subsequent runs; the locality term is smaller as weights remain closer (see \cite{kumar2022finetuning}).%

\paragraph{Mild hyperparameter search.}
As shown in \Cref{fig:home0_locality_requirement}, extreme hyperparameter ranges lead to weights whose average may perform poorly.
Indeed, weights obtained from extremely different hyperparameters may not be linearly connectable; they may belong to different regions of the loss landscape.
In our experiments, we thus use the mild search space defined in \Cref{tab:hyperparam}, first introduced in SWAD \cite{cha2021wad}.
These hyperparameter ranges induce diverse models that are averageable in weights.

\paragraph{Weight selection.}
The last step of our approach (summarized in \Cref{alg:pseudo-code}) is to choose which weights to average among those available.
We explore two simple weight selection protocols, as in \cite{Wortsman2022ModelSA}.
The first \textit{uniform} equally averages all weights; it is practical but may underperform when some runs are detrimental.
The second \textit{restricted} (\textit{greedy} in \cite{Wortsman2022ModelSA}) solves this drawback by restricting the number of selected weights: weights are ranked in decreasing order of validation accuracy and sequentially added only if they improve DiWA's validation accuracy.

In the following sections, we experimentally validate our theory.
First, \Cref{sect:analysis} confirms our findings on the OfficeHome dataset \cite{venkateswara2017deep} where diversity shift dominates \cite{ye2021odbench} (see \Cref{app:analysis_pacs} for a similar analysis on PACS \cite{li2017deeper}). Then, \Cref{sec:domainbed} shows that DiWA is state of the art on DomainBed~\cite{gulrajani2021in}.%
\begin{figure}
    \vskip -0.1in
    \begin{algorithm}[H]
        \caption{DiWA Pseudo-code} \label{alg:pseudo-code}
        \begin{algorithmic}
            \REQUIRE $\theta_0$ pretrained encoder and initialized classifier; $\{h_m\}_{m=1}^H$ hyperparameter configurations.
            \STATE \hspace*{-1.25em} \textit{\underline{Training}:} $\forall m=1$ to $H$, $\theta_m \triangleq \mathrm{FineTune}(\theta_0, h_m)$
            \STATE \hspace*{-1.25em} \textit{\underline{Weight selection}:}
            \STATE \hspace*{-0.5em} \textit{Uniform:} $\M=\{1, \cdots, H\}$.
            \STATE \hspace*{-0.5em} \textit{Restricted:} Rank $\{\theta_m\}_{m=1}^H$ by decreasing $\mathrm{ValAcc}(\theta_m)$. $\M\leftarrow\emptyset$.
            \FOR{$m=1$ to $H$}
            \STATE \textbf{If} $\mathrm{ValAcc}(\theta_{\M \cup \{m\}})\geq\mathrm{ValAcc}(\theta_\M)$
            \STATE $\M \leftarrow \M\cup \{m\}$
            \ENDFOR
            \STATE \hspace*{-1.25em} \textit{\underline{Inference}}: with $f(\cdot, \theta_\M)$, where $\theta_\M=\sum_{m\in\M}\theta_m/|\M|$.
        \end{algorithmic}
    \end{algorithm}
    \vskip -0.3in
\end{figure}%
\section{Empirical validation of our theoretical insights}
\label{sect:analysis}
We consider several collections of weights $\{\theta_{m}\}_{m=1}^M$ ($2\leq M < 10$) trained on the \enquote{Clipart}, \enquote{Product} and \enquote{Photo} domains from OfficeHome \cite{venkateswara2017deep} with a shared random initialization and mild hyperparameter ranges.
These weights are first indifferently sampled from a single run (every $50$ batches) or from different runs.
They are evaluated on \enquote{Art}, the fourth domain from OfficeHome.%
\begin{figure}
\begin{minipage}{.45\textwidth}%
  \includegraphics[width=\linewidth]{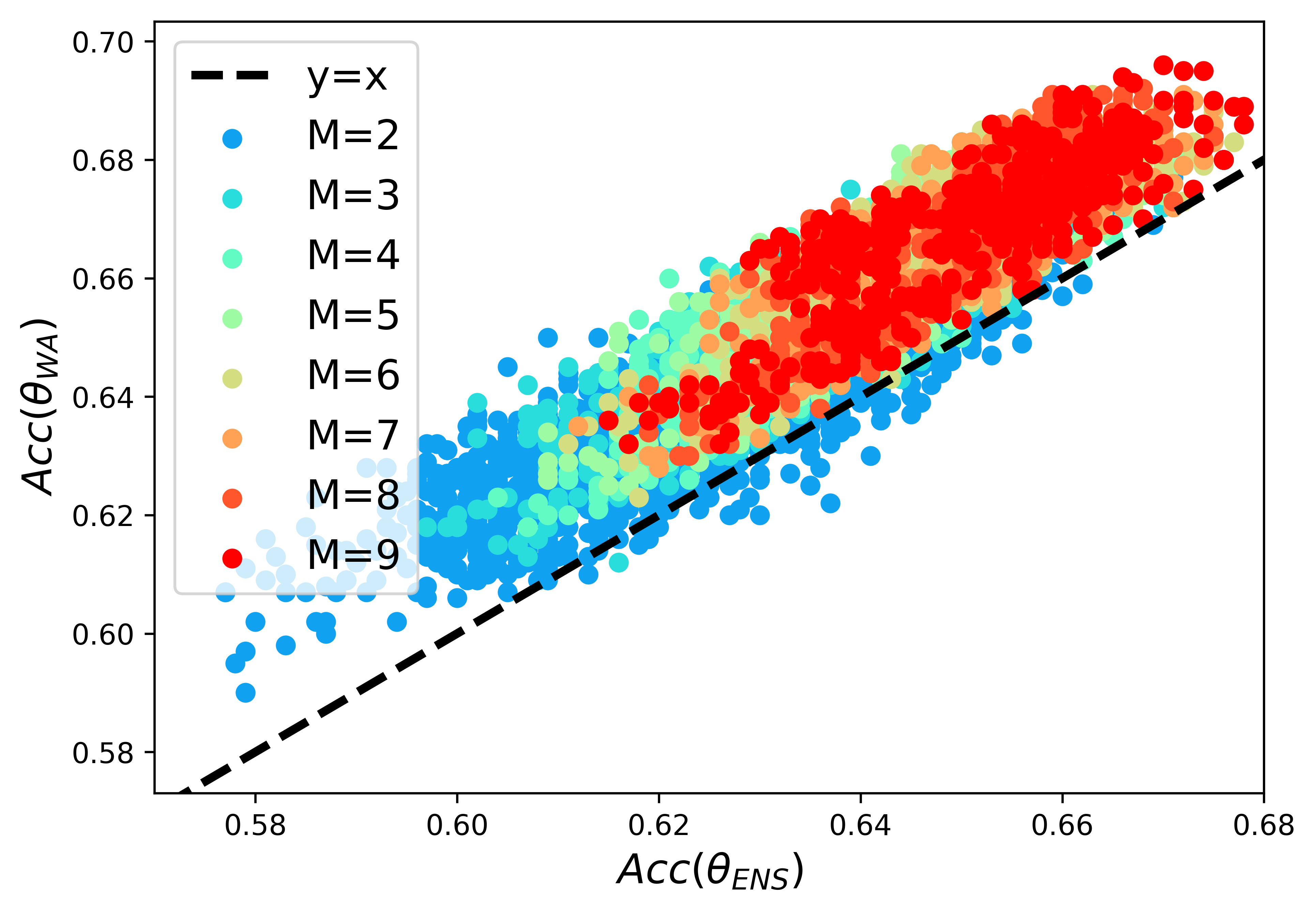}
  \captionof{figure}{Each dot displays the accuracy ($\uparrow$) of weight averaging (WA) \versus accuracy ($\uparrow$) of prediction averaging (ENS) for $M$ models.}
  \label{fig:home0_samediffruns_net_soup}
\end{minipage}%
\hskip 4ex
\begin{minipage}{.45\textwidth}%
  \includegraphics[width=\linewidth]{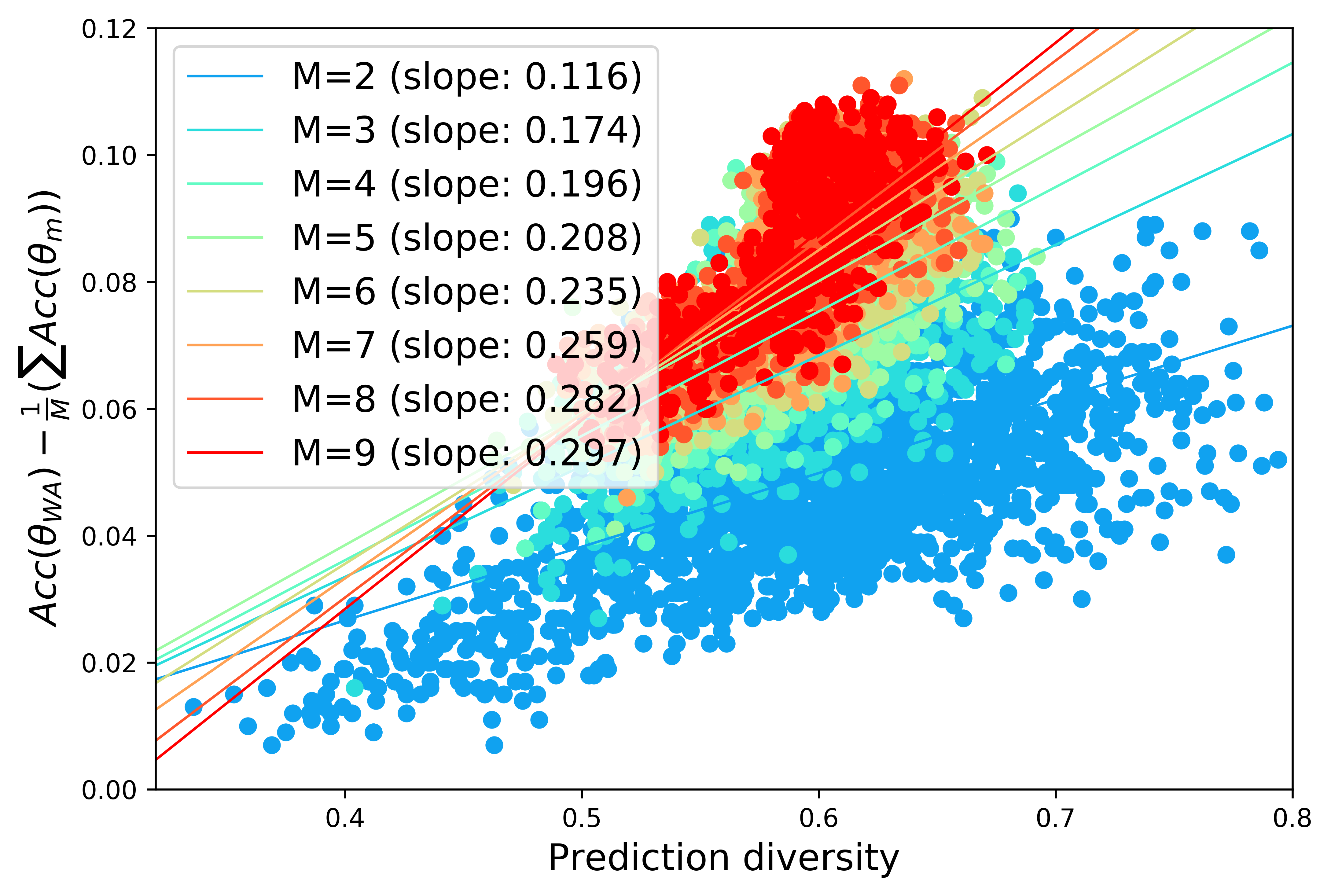}
  \captionof{figure}{Each dot displays the accuracy ($\uparrow$) gain of WA over its members \versus the prediction diversity \cite{aksela2003comparison} ($\uparrow$) for $M$ models.}
  \label{fig:home0_dr_soup-netm}
\end{minipage}
\end{figure}

\paragraph{WA \versus ENS.}
\Cref{fig:home0_samediffruns_net_soup} validates \Cref{lemma:wa_ensembling} and that \textbf{$f_{\text{WA}} \approx f_{\text{ENS}}$}. 
More precisely, $f_{\text{WA}}$ slightly but consistently improves $f_{\text{ENS}}$: we discuss this in \Cref{app:wa_vs_ens_analysis}.
Moreover, a larger $M$ improves the results; in accordance with \Cref{eq:b_var_cov}, this motivates averaging as many weights as possible.
In contrast, large $M$ is computationally impractical for ENS at test time, requiring $M$ forwards.%

\paragraph{Diversity and accuracy.}
We validate in \Cref{fig:home0_dr_soup-netm} that $f_{\text{WA}}$ benefits from diversity.
Here, we measure diversity with the ratio-error \cite{aksela2003comparison}, \ie the ratio $N_{\text{diff}}/N_{\text{simul}}$ between the number of different errors $N_{\text{diff}}$ and of simultaneous errors $N_{\text{simul}}$ in test for a pair in $\{f\left(\cdot, \theta_{m}\right)\}_{m=1}^M$.
A higher average over the $\binom{M}{2}$ pairs means that members are less likely to err on the same inputs.
Specifically, the gain of $\operatorname{Acc}(\theta_{\text{WA}})$ over the mean individual accuracy $\frac{1}{M}\sum_{m=1}^M\operatorname{Acc}\left(\theta_m\right)$ increases with diversity.
Moreover, this phenomenon intensifies for larger $M$: the linear regression's slope (\ie the accuracy gain per unit of diversity) increases with $M$. This is consistent with the $(M-1)/M$ factor of $\covb\left(x\right)$ in \Cref{eq:b_var_cov}, as further highlighted in \Cref{app:soup:m_slope}.
Finally, in \Cref{app:analysis_office_features}, we show that the conclusion also holds with CKAC \cite{DBLP:journals/corr/abs-1905-00414}, another established diversity measure.

\paragraph{Increasing diversity thus accuracy via different runs.}
Now we investigate the difference between sampling the weights from a single run or from different runs.
\Cref{fig:home0_dr_frequency} \textit{first} shows that diversity increases when weights come from different runs.
\textit{Second}, in \Cref{fig:home0_m_vs_acc}, this is reflected on the accuracies in \ood.
Here, we rank by validation accuracy the $60$ weights obtained (1) from $60$ different runs and (2) along $1$ well-performing run.
We then consider the WA of the top $M$ weights as $M$ increases from $1$ to $60$.
Both have initially the same performance and improve with $M$; yet,
WA of weights from different runs gradually outperforms the single-run WA.
\textit{Finally}, \Cref{fig:home0_locality_requirement} shows that this holds only for mild hyperparameter ranges and with a shared initialization.
Otherwise, when hyperparameter distributions are extreme (as defined in \Cref{tab:hyperparam}) or when classifiers are not similarly initialized, DiWA may perform worse than its members due to a violation of the locality condition.
These experiments confirm that \textit{diversity is key as long as the weights remain averageable}.%
\begin{figure}[b]
\begin{minipage}{.32\textwidth}%
    \includegraphics[width=1.0\textwidth]{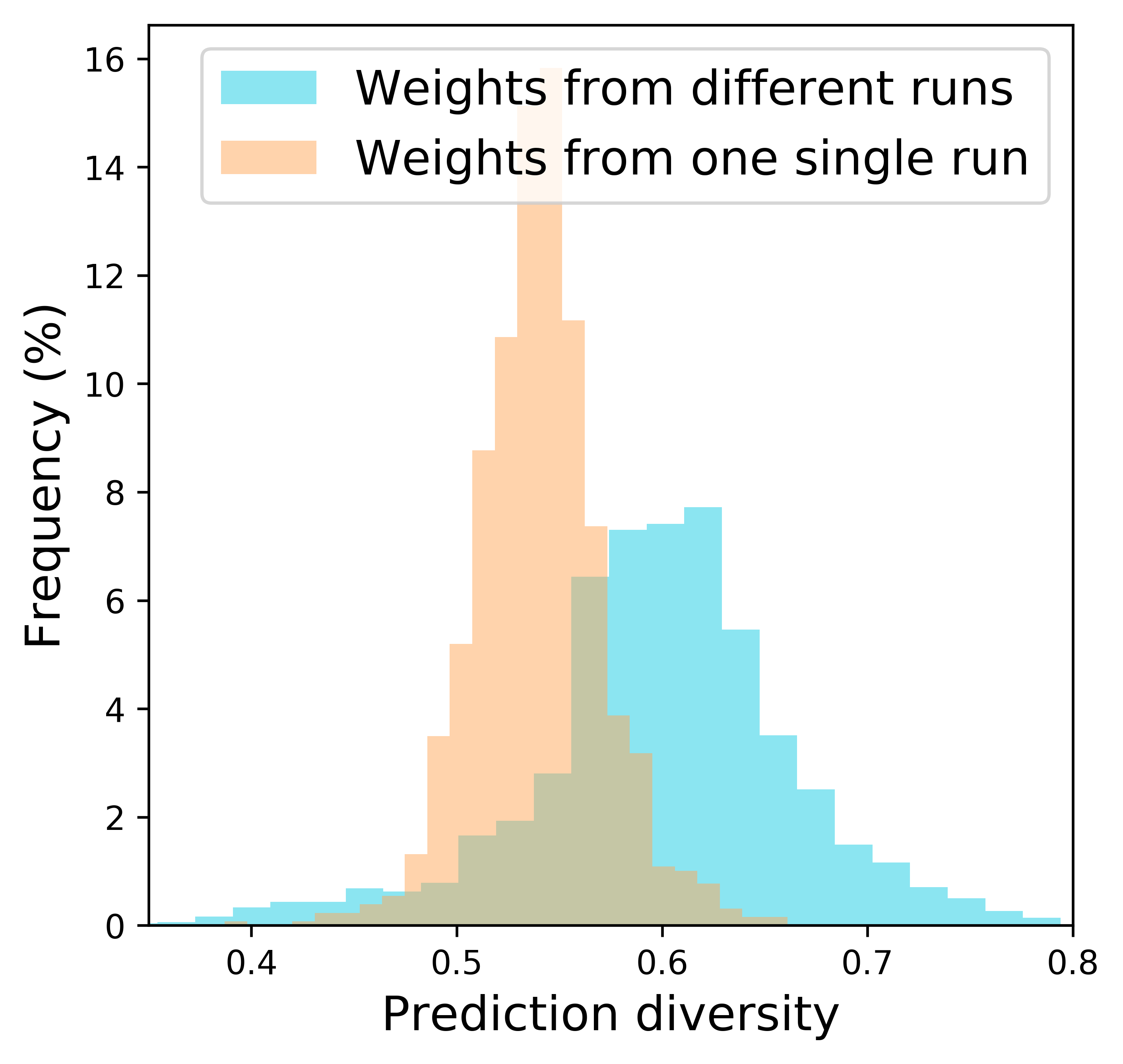}%
    \captionof{figure}{Frequencies of prediction diversities ($\uparrow$) \cite{aksela2003comparison} across $2$ weights obtained along a single run or from different runs.}%
    \label{fig:home0_dr_frequency}
\end{minipage}%
\hskip 1.5ex
\begin{minipage}{.32\textwidth}%
    \includegraphics[width=1.0\textwidth]{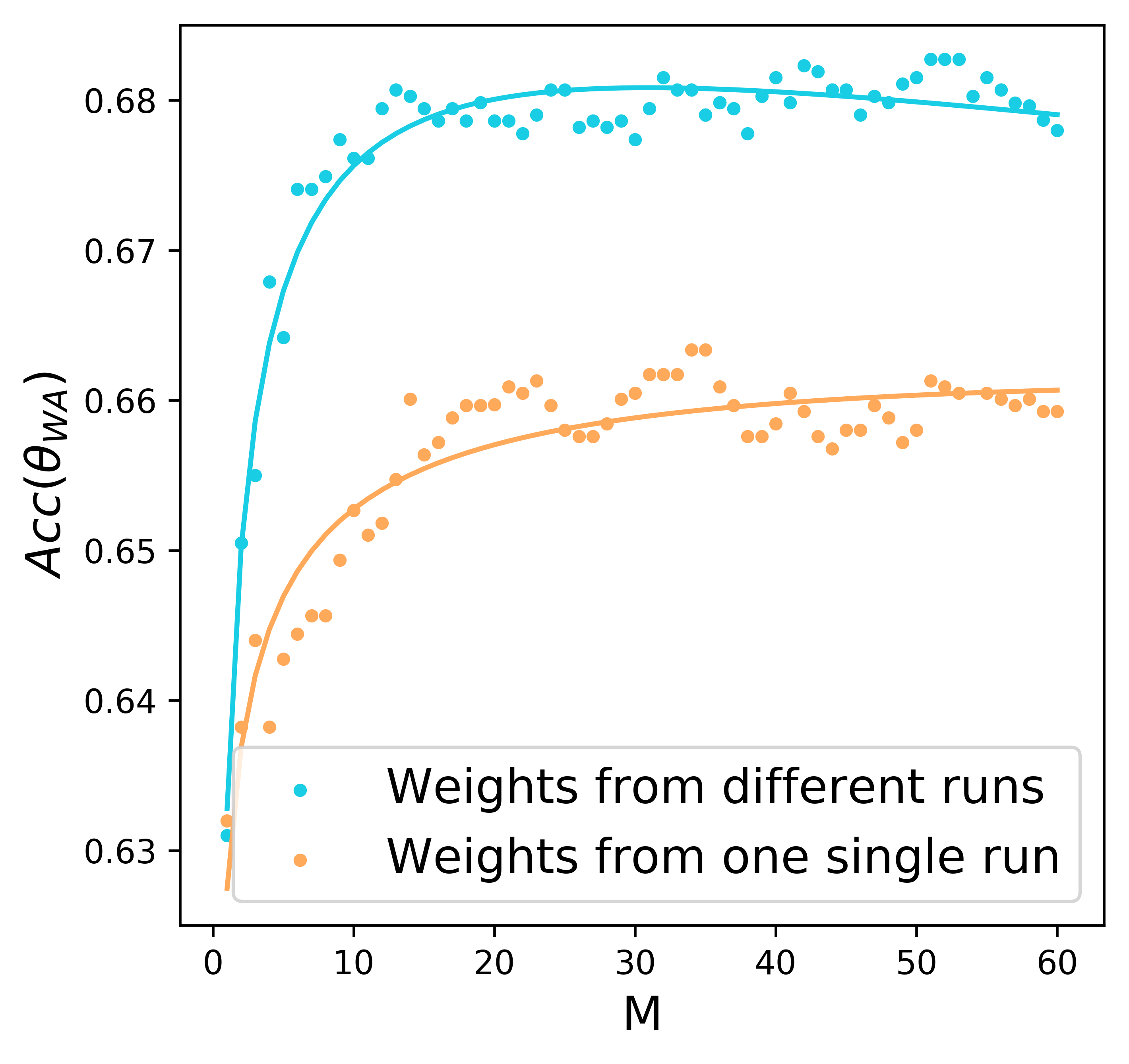}%
    \captionof{figure}{WA accuracy ($\uparrow$) as $M$ increases, when the $M$ weights are obtained along a single run or from different runs.}%
    \label{fig:home0_m_vs_acc}%
\end{minipage}%
\hskip 1.5ex
\begin{minipage}{.32\textwidth}%
    \includegraphics[width=1.0\textwidth]{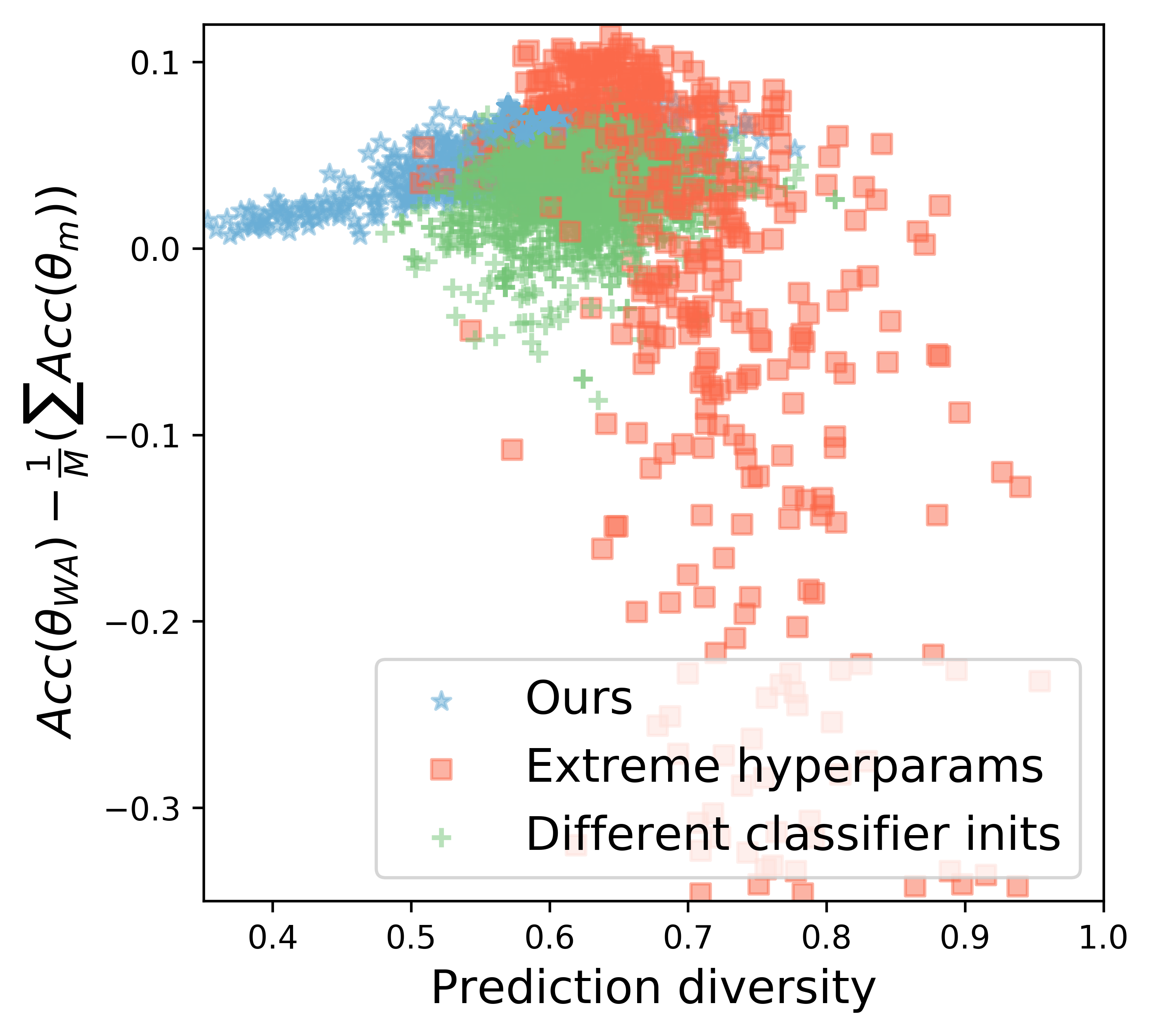}%
    \captionof{figure}{Each dot displays the accuracy ($\uparrow$) gain of WA over its members \versus prediction diversity ($\uparrow$) for $2\leq M < 10$ models.}
    \label{fig:home0_locality_requirement}%
\end{minipage}%
\vspace{-1.em}%
\end{figure}
\clearpage
\section{Experimental results on the DomainBed benchmark}%
\label{sec:domainbed}
\textbf{Datasets.}
We now present our evaluation on DomainBed \cite{gulrajani2021in}.
By imposing the code, the training procedures and the ResNet50 \cite{he51deep} architecture, DomainBed is arguably the fairest benchmark for \ood generalization.
It includes $5$ multi-domain real-world datasets: PACS \cite{li2017deeper}, VLCS \cite{fang2013unbiased}, OfficeHome \cite{venkateswara2017deep}, TerraIncognita \cite{beery2018recognition} and DomainNet \cite{peng2019moment}. \cite{ye2021odbench} showed that \textit{diversity shift dominates in these datasets}. Each domain is successively considered as the target $T$ while other domains are merged into the source $S$. The validation dataset is sampled from $S$, \ie we follow DomainBed's training-domain model selection.
The experimental setup is further described in \Cref{app:domainbed_description}.
Our code is available at \url{https://github.com/alexrame/diwa}.

\textbf{Baselines.}
ERM is the standard Empirical Risk Minimization.
Coral \cite{coral216aaai} is the best approach based on domain invariance.
SWAD (Stochastic Weight Averaging Densely) \cite{cha2021wad} and MA (Moving Average) \cite{arpit2021ensemble} average weights along one training trajectory but differ in their weight selection strategy.
SWAD \cite{cha2021wad} is the current state of the art (SoTA) thanks to it \enquote{overfit-aware} strategy, yet at the cost of three additional hyperparameters (a patient parameter, an overfitting patient parameter and a tolerance rate) tuned per dataset.
In contrast, MA \cite{arpit2021ensemble} is easy to implement as it simply combines all checkpoints uniformly starting from batch $100$ until the end of training.
\reb{Finally, we report the scores obtained in \cite{arpit2021ensemble} for the costly Deep Ensembles (DENS) \cite{Lakshminarayanan2017} (with different initializations)}: we discuss other ensembling strategies in~\Cref{app:wa_vs_ens_analysis}.

\textbf{Our runs.}
ERM and DiWA share the same training protocol in DomainBed: yet, instead of keeping only one run from the grid-search, DiWA leverages $M$ runs.
In practice, we sample $20$ configurations from the hyperparameter distributions detailed in \Cref{tab:hyperparam} and report the mean and standard deviation across $3$ data splits.
For each run, we select the weights of the epoch with the highest validation accuracy.
ERM and MA select the model with highest validation accuracy across the $20$ runs, following standard practice on DomainBed.
\reb{Ensembling (ENS) averages the predictions of all $M=20$ models (with shared initialization).}
DiWA-restricted selects $1 \leq M \leq 20 $ weights with \Cref{alg:pseudo-code} while DiWA-uniform averages all $M=20$ weights.
DiWA$^{\dagger}$ averages uniformly the $M=3\times20=60$ weights from all $3$ data splits.
DiWA$^{\dagger}$ benefits from larger $M$ (without additional inference cost) and from data diversity (see \Cref{app:more_diversity_comparison}).
However, we cannot report standard deviations for DiWA$^{\dagger}$ for computational reasons. Moreover, DiWA$^{\dagger}$ cannot leverage the restricted weight selection, as the validation is not shared across all $60$ weights that have different data splits.

\subsection{Results on DomainBed}
\label{subsec:domain_bed}
We report our \textbf{main results} in \Cref{table:db_all_training}, detailed per domain in \Cref{app:full_results}.
With a randomly initialized classifier, DiWA$^\dagger$-uniform is the best on PACS, VLCS and OfficeHome:
DiWA-uniform is the second best on PACS and OfficeHome.
On TerraIncognita and DomainNet, DiWA is penalized by some bad runs, filtered in DiWA-restricted which improves results on these datasets.
Classifier initialization with linear probing (LP) \cite{kumar2022finetuning} improves all methods on OfficeHome, TerraIncognita and DomainNet.
On these datasets, DiWA$^\dagger$ increases MA by $1.3$, $0.5$ and $1.1$ points respectively.
After averaging, DiWA$^\dagger$ with LP \textit{establishes a new SoTA of $68.0\%$}, improving SWAD by $1.1$ points.%
\begin{table}[b]%
    \vspace{-1.5em}%
    \caption{\textbf{Accuracy ($\%, \uparrow$) on DomainBed with ResNet50} (best in \textbf{bold} and second best \underline{underlined}).}%
    \centering
    \vspace{0.2em}
    \adjustbox{width=\textwidth}{
        \begin{tabular}{llll|cccccc}
            \toprule
             & \textbf{Algorithm}                                & \textbf{Weight selection} & \textbf{Init}                                  & \textbf{PACS}  & \textbf{VLCS}              & \textbf{OfficeHome}        & \textbf{TerraInc}          & \textbf{DomainNet}      & \textbf{Avg}     \\
            \midrule
             & ERM                                               & N/A                       & \multirow{5}{*}{Random}                        & 85.5 $\pm$ 0.2 & 77.5 $\pm$ 0.4             & 66.5 $\pm$ 0.3             & 46.1 $\pm$ 1.8             & 40.9 $\pm$ 0.1          & 63.3             \\
             & Coral \cite{coral216aaai}                         & N/A                       &                                                & 86.2 $\pm$ 0.3 & 78.8 $\pm$ 0.6             & 68.7 $\pm$ 0.3             & 47.6 $\pm$ 1.0             & 41.5 $\pm$ 0.1          & 64.6             \\
             & SWAD \cite{cha2021wad}                            & Overfit-aware             &                                                & 88.1 $\pm$ 0.1 & 79.1 $\pm$ 0.1             & 70.6 $\pm$ 0.2             & 50.0 $\pm$ 0.3             & 46.5 $\pm$ 0.1          & 66.9             \\
             & MA \cite{arpit2021ensemble}                       & Uniform                   &                                                & 87.5 $\pm$ 0.2 & 78.2 $\pm$ 0.2             & 70.6 $\pm$ 0.1             & 50.3 $\pm$ 0.5             & 46.0 $\pm$ 0.1          & 66.5             \\
            & \reb{DENS \cite{Lakshminarayanan2017,arpit2021ensemble}} & \reb{Uniform: $M=6$}            &                                                & \reb{87.6}           & \reb{78.5}                       & \reb{70.8}                       & \reb{49.2}                       & \reb{\textbf{47.7}}           & \reb{66.8}             \\
            \midrule
            \multirow{12}{*}{\begin{turn}{90} Our runs \end{turn}}
             & ERM                                               & N/A                       & \multirow{6}{*}{Random}                        & 85.5 $\pm$ 0.5 & 77.6 $\pm$ 0.2             & 67.4 $\pm$ 0.6             & 48.3 $\pm$ 0.8             & 44.1 $\pm$ 0.1          & 64.6             \\
             & MA \cite{arpit2021ensemble}                       & Uniform                   &                                                & 87.9 $\pm$ 0.1 & 78.4 $\pm$ 0.1             & 70.3 $\pm$ 0.1             & 49.9 $\pm$ 0.2             & 46.4 $\pm$ 0.1          & 66.6             \\
            & \reb{ENS}                                               & \reb{Uniform: $M=20$}           &                                                & \reb{88.0 $\pm$ 0.1} & \reb{78.7 $\pm$ 0.1}             & \reb{70.5 $\pm$ 0.1}             & \reb{51.0 $\pm$ 0.5}             & \reb{47.4 $\pm$ 0.2}          & \reb{67.1}             \\
             & DiWA                                              & Restricted: $M \leq 20$   &                                                & 87.9 $\pm$ 0.2 & \underline{79.2} $\pm$ 0.1 & 70.5 $\pm$ 0.1             & 50.5 $\pm$ 0.5             & 46.7 $\pm$ 0.1          & 67.0             \\
             & DiWA                                              & Uniform: $M=20$           &                                                & 88.8 $\pm$ 0.4 & 79.1 $\pm$ 0.2             & 71.0 $\pm$ 0.1             & 48.9 $\pm$ 0.5             & 46.1 $\pm$ 0.1          & 66.8             \\
             & DiWA$^{\dagger}$                                  & Uniform: $M=60$           &                                                & \textbf{89.0}  & \textbf{79.4}              & 71.6                       & 49.0                       & 46.3                    & 67.1             \\
            \cmidrule{2-10}
             & ERM                                               & N/A                       & \multirow{6}{*}{LP \cite{kumar2022finetuning}} & 85.9 $\pm$ 0.6 & 78.1 $\pm$ 0.5             & 69.4 $\pm$ 0.2             & 50.4 $\pm$ 1.8             & 44.3 $\pm$ 0.2          & 65.6             \\
             & MA \cite{arpit2021ensemble}                       & Uniform                   &                                                & 87.8 $\pm$ 0.3 & 78.5 $\pm$ 0.4             & 71.5 $\pm$ 0.3             & 51.4 $\pm$ 0.6             & 46.6 $\pm$ 0.0          & 67.1             \\
             & \reb{ENS}                                               & \reb{Uniform: $M=20$}           &                                                & \reb{88.1 $\pm$ 0.3} & \reb{78.5 $\pm$ 0.1}             & \reb{71.7 $\pm$ 0.1}             & \reb{50.8 $\pm$ 0.5}             & \reb{47.0 $\pm$ 0.2}          & \reb{67.2}             \\
             & DiWA                                              & Restricted: $M \leq 20$   &                                                & 88.0 $\pm$ 0.3 & 78.5 $\pm$ 0.1             & 71.5 $\pm$ 0.2             & \underline{51.6} $\pm$ 0.9 & \textbf{47.7} $\pm$ 0.1 & 67.5             \\
             & DiWA                                              & Uniform: $M=20$           &                                                & 88.7 $\pm$ 0.2 & 78.4 $\pm$ 0.2             & \underline{72.1} $\pm$ 0.2 & 51.4 $\pm$ 0.6             & 47.4 $\pm$ 0.2          & \underline{67.6} \\
             & DiWA$^{\dagger}$                                  & Uniform: $M=60$           &                                                & \textbf{89.0}  & 78.6                       & \textbf{72.8}              & \textbf{51.9}              & \textbf{47.7}           & \textbf{68.0}    \\
            \bottomrule%
        \end{tabular}%
    }%
    \label{table:db_all_training}%
\end{table}%
\clearpage
\begin{wraptable}[9]{hR!}{0.45\textwidth}%
  \caption{\textbf{Accuracy ($\%, \uparrow$) on OfficeHome} domain \enquote{Art} with various objectives.}%
  \centering%
  \adjustbox{width=1.0\linewidth}{%
    \begin{tabular}{lcccc}%
      \toprule
      \textbf{Algorithm} & \textbf{No WA}               & \textbf{MA} & \textbf{DiWA}       & \textbf{DiWA$^{\dagger}$} \\
      \midrule
      ERM                & 62.9 $\pm$ 1.3             & \underline{65.0} $\pm$ 0.2           & 67.3 $\pm$ 0.2             & 67.7                             \\
      Mixup              & \underline{63.1} $\pm$ 0.7 & \textbf{66.2} $\pm$ 0.3              & 67.8 $\pm$ 0.6             & 68.4                             \\
      Coral              & \textbf{64.4} $\pm$ 0.4    & 64.4 $\pm$ 0.4                       & 67.7 $\pm$ 0.2             & 68.2                             \\
      ERM/Mixup          & N/A                        & N/A                                  & 67.9 $\pm$ 0.7             & \underline{68.9}                 \\
      ERM/Coral          & N/A                        & N/A                                  & \underline{68.1} $\pm$ 0.3 & 68.7                             \\
      ERM/Mixup/Coral    & N/A                        & N/A                                  & \textbf{68.4} $\pm$ 0.4    & \textbf{69.1}                    \\
      \bottomrule%
    \end{tabular}}%
  \label{table:db_home0_training}%
\end{wraptable}
\textbf{DiWA with different objectives.}
So far we used ERM that does not leverage the domain information.
\Cref{table:db_home0_training} shows that DiWA-uniform benefits from averaging weights trained with Interdomain Mixup \cite{yan2020improve} and Coral \cite{coral216aaai}: accuracy gradually improves as we add more objectives.
Indeed, as highlighted in \Cref{app:more_diversity_comparison}, DiWA benefits from the increased diversity brought by the various objectives.
This suggests a new kind of linear connectivity across models trained with different objectives; the full analysis of this is left for future work.%
\vspace{-0.2em}%
\subsection{Limitations of DiWA}
\label{subsec:limitations}
Despite this success, DiWA has some limitations.
\textit{First}, DiWA cannot benefit from additional diversity that would break the linear connectivity between weights --- as discussed in \Cref{app:wa_vs_ens_analysis}.
\textit{Second}, DiWA (like all WA approaches) can tackle diversity shift but not correlation shift: this property is explained for the first time in \Cref{subsec:analysisbvc} and illustrated in \Cref{app:failure_corr_shift} on ColoredMNIST.%
\vspace{-0.2em}%
\section{Related work}%
\label{sec:related_work}
\textbf{Generalization and ensemble.}
To generalize under distribution shifts, invariant approaches \cite{arjovsky2019invariant,krueger2020utofdistribution,rame_fishr_2021,coral216aaai,Sagawa2020Distributionally,ganin2016domain} try to detect the causal mechanism rather than memorize correlations: yet, they do not outperform ERM on various benchmarks \cite{gulrajani2021in,ye2021odbench,pmlr-v139-koh21a}.
In contrast, ensembling of deep networks \cite{Lakshminarayanan2017,hansen1990neural,krogh1995neural} consistently increases robustness \cite{Ovadia2019} and was successfully applied to domain generalization \cite{arpit2021ensemble,thopalli2021multidomain,Mesbah2022,li2022domain,lee2022diversify,pagliardini2022agree}.
As highlighted in \cite{ueda1996generalization} (whose analysis underlies our \Cref{eq:b_var_cov}), ensembling works due to the diversity among its members.
This diversity comes primarily from the randomness of the learning procedure \cite{Lakshminarayanan2017} and can be increased with different hyperparameters \cite{wenzel2020hyperparameter}, data \cite{breiman1996bagging,nixon2020why,Yeo2021}, augmentations \cite{wen2021combining,rame2021ixmo} or with regularizations \cite{lee2022diversify,pagliardini2022agree,rame2021dice,teneydiv}.%

\textbf{Weight averaging.}
Recent works \cite{izmailov2018,Draxler2018,Guo2022,zhang2019lookahead} combine in weights (rather than in predictions) models collected along a single run. This was shown suboptimal in \iid \cite{ashukha2020pitfalls} but successful in \ood \cite{cha2021wad,arpit2021ensemble}.
Following the linear mode connectivity \cite{Frankle2020,nagarajan2019uniform} and the property that many independent models are connectable \cite{Benton2021}, a second group of works average weights with fewer constraints  \cite{Wortsman2022robust,mergefisher21,Wortsman2022ModelSA,Gupta2020Stochastic,choshen2022fusing,wortsman2021learning}.
\reb{To induce greater diversity, \cite{maddox2019simple} used a high constant learning rate; \cite{Benton2021} explicitly encouraged the weights to encompass more volume in the weight space; \cite{wortsman2021learning} minimized cosine similarity between weights; \cite{izmailov_subspace_2019} used a tempered posterior.}
\reb{From a loss landscape perspective \cite{Fort2019DeepEA}, these methods aimed at \enquote{explor[ing] the set of possible solutions instead of simply converging to a single point}, as stated in \cite{maddox2019simple}}.
The recent \enquote{Model soups} introduced by Wortsman \textit{et al.} \cite{Wortsman2022ModelSA} is a WA algorithm similar to \Cref{alg:pseudo-code}; yet, the theoretical analysis and the goals of these two works are different.
Theoretically, we explain why WA succeeds under diversity shift: the bias/correlation shift, variance/diversity shift and diversity-based findings are novel and are confirmed empirically.
Regarding the motivation, our work aims at combining more diverse weights: it may be analyzed as a general framework to average weights obtained in various ways. In contrast, \cite{Wortsman2022ModelSA} challenges the standard model selection after a grid search.
Regarding the task, \cite{Wortsman2022ModelSA} and our work complement each other: while \cite{Wortsman2022ModelSA} demonstrate robustness on several ImageNet variants with distribution shift, we improve the SoTA on the multi-domain DomainBed benchmark against other established \ood methods after a thorough and fair comparison. Thus, DiWA and \cite{Wortsman2022ModelSA} are theoretically complementary with different motivations and applied successfully for different tasks.
\section{Conclusion}
In this paper, we propose a new explanation for the success of WA in \ood by leveraging its ensembling nature.
Our analysis is based on a new bias-variance-covariance-locality decomposition for WA, where we theoretically relate bias to correlation shift and variance to diversity shift. It also shows that diversity is key to improve generalization.
This motivates our DiWA approach that averages in weights models trained independently.
DiWA improves the state of the art on DomainBed, the reference benchmark for \ood generalization.
Critically, DiWA has no additional inference cost --- removing a key limitation of standard ensembling.
Our work may encourage the community to further create diverse learning procedures and objectives --- whose models may be averaged in weights.

\subsection*{Acknowledgements}
We would like to thank Jean-Yves Franceschi for his helpful comments and discussions on our paper.
This work was granted access to the HPC resources of IDRIS under the allocation AD011011953 made by GENCI. We acknowledge the financial support by the French National Research Agency (ANR) in the chair VISA-DEEP (project number ANR-20-CHIA-0022-01) and the ANR projects DL4CLIM ANR-19-CHIA-0018-01, RAIMO ANR-20-CHIA-0021-01, OATMIL ANR-17-CE23-0012 and LEAUDS ANR-18-CE23-0020.

\bibliographystyle{unsrt}
\bibliography{ref}
\clearpage
\section*{Checklist}
\begin{enumerate}
\item For all authors...
\begin{enumerate}
  \item Do the main claims made in the abstract and introduction accurately reflect the paper's contributions and scope?
    \answerYes{}
  \item Did you describe the limitations of your work?
    \answerYes{In \Cref{subsec:limitations}.}
  \item Did you discuss any potential negative societal impacts of your work?
    \answerYes{In \Cref{app:broader_impact}}
  \item Have you read the ethics review guidelines and ensured that your paper conforms to them?
    \answerYes{}
\end{enumerate}

\item If you are including theoretical results...
\begin{enumerate}
    \item Did you state the full set of assumptions of all theoretical results?
    \answerYes{\Cref{ass:no_bias_iid} discussed in \Cref{app:bias_correlation_ass_noiidbias} and \Cref{ass:infinite_width,ass:inter_sample} discussed in \Cref{app:discussion_inter_sample}.}
    \item Did you include complete proofs of all theoretical results?
    \answerYes{In \Cref{app:proof}}
\end{enumerate}

\item If you ran experiments...
\begin{enumerate}
    \item Did you include the code, data, and instructions needed to reproduce the main experimental results (either in the supplemental material or as a URL)? \answerYes{Our code is available at \url{https://github.com/alexrame/diwa}.}
    \item Did you specify all the training details (e.g., data splits, hyperparameters, how they were chosen)? \answerYes{See \Cref{sec:domainbed} and \Cref{app:domainbed_description}}
    \item Did you report error bars (e.g., with respect to the random seed after running experiments multiple times)?
    \answerYes{Defined by different data splits when possible.}
    \item Did you include the total amount of compute and the type of resources used (e.g., type of GPUs, internal cluster, or cloud provider)?
    \answerYes{Approximately $20000$ hours of GPUs (Nvidia V100) on an internal cluster, mostly for the $2640$ runs needed in \Cref{table:db_all_training}.}
\end{enumerate}

\item If you are using existing assets (e.g., code, data, models) or curating/releasing new assets...
\begin{enumerate}
  \item If your work uses existing assets, did you cite the creators?
    \answerYes{DomainBed benchmark \cite{gulrajani2021in} and its datasets.}
  \item Did you mention the license of the assets?
    \answerYes{DomainBed is under \enquote{The MIT License}.}
  \item Did you include any new assets either in the supplemental material or as a URL?
    \answerNo{}
  \item Did you discuss whether and how consent was obtained from people whose data you're using/curating?
    \answerNA{}
  \item Did you discuss whether the data you are using/curating contains personally identifiable information or offensive content?
    \answerNA{}
\end{enumerate}

\item If you used crowdsourcing or conducted research with human subjects...
\begin{enumerate}
  \item Did you include the full text of instructions given to participants and screenshots, if applicable?
    \answerNA{}
  \item Did you describe any potential participant risks, with links to Institutional Review Board (IRB) approvals, if applicable?
    \answerNA{}
  \item Did you include the estimated hourly wage paid to participants and the total amount spent on participant compensation?
    \answerNA{}
\end{enumerate}

\end{enumerate}

\newpage
\begin{appendices}
    \appendix

This supplementary material complements the main paper.
It is organized as follows:
\begin{enumerate}
    \item \Cref{app:broader_impact} describes the broader impact of our work.
    \item \Cref{app:limit_proof_swad} points out the limitations of existing flatness-based analysis of WA and shows how our analysis solves these limitations.
    \item \Cref{app:proof} details all the proofs of the propositions and lemmas found in our work.
        \begin{itemize}
            \item \Cref{app:wa_loss,app:proof_bvc} derive the bias-variance-covariance-locality decomposition for WA (\Cref{prop:b_var_cov}).
            \item \Cref{app:bias_correlation} establishes the link between bias and correlation shift (\Cref{prop:bias}).
            \item \Cref{app:var_diversity} establishes the link between variance and diversity shift (\Cref{prop:var}).
            \item \Cref{app:proof_wa_ind} compares WA with one of its member (\Cref{lemma:wa_ind}).
        \end{itemize}
    \item \Cref{app:wa_vs_ens_analysis} empirically compares WA to functional ensembling ENS.
    \item \Cref{app:additional_diversity_analysis} presents some additional diversity results on OfficeHome and PACS.
    \item \reb{\Cref{app:valueofm} ablates the importance of the number of training runs}.
    \item \Cref{app:domainbeddetails} describes our experiments on DomainBed and our per-domain results.
    \item \Cref{app:failure_corr_shift} empirically confirms a limitation of WA approaches expected from our theoretical analysis: they do not tackle correlation shift on ColoredMNIST.
    \item \reb{\Cref{app:llr} suggests DiWA's potential when some target data is available for training \cite{kirichenko2022last}}.
\end{enumerate}

\section{Broader impact statement}
\label{app:broader_impact}

We believe our paper can have several positive impacts.
\textit{First}, our theoretical analysis enables practitioners to know when averaging strategies succeed (under diversity shift, where variance dominates) or break down (under correlation shift, where bias dominates).
This is key to understand when several models can be combined into a production system, or if the focus should be put on the training objective and/or the data.
\textit{Second}, it sets a new state of the art for \ood generalization under diversity shift without relying on a specific objective, architecture or task prior. It could be useful in medicine \cite{Zech2018,degrave2021ai} or to tackle fairness issues related to under-representation \cite{Sagawa2020Distributionally,blodgett2016demographic,barocas2016big}.
\textit{Finally}, DIWA has no additional inference cost; in contrast, functional ensembling needs one forward per member.
Thus, DiWA removes the carbon footprint overhead of ensembling strategies at test-time.

Yet, our paper may also have some negative impacts.
\textit{First}, it requires independent training of several models. It may motivate practitioners to learn even more networks and average them afterwards.
Note that in \Cref{sec:domainbed}, we restricted ourselves to combining only the runs obtained from the standard ERM grid search from DomainBed \cite{gulrajani2021in}.
\textit{Second}, our model is fully deep learning based with the corresponding risks, \eg adversarial attacks and lack of interpretability.
\textit{Finally}, we do not control its possible use to surveillance or weapon systems.

\section{Limitations of the flatness-based analysis in \ood}%
\label{app:limit_proof_swad}%
\begin{theorem}[Equation 21 from \cite{cha2021wad}, simplified version of their Theorem 1]
    Consider a set of $N$ covers $\left\{\Theta_{k}\right\}_{k=1}^{N}$ \sut the parameter space $\Theta \subset \cup_{k}^{N} \Theta_{k}$ where $\operatorname{diam}(\Theta)\triangleq\sup _{\theta, \theta^{\prime} \in \Theta}\left\|\theta-\theta^{\prime}\right\|_{2}, N\triangleq\left\lceil(\operatorname{diam}(\Theta) / \gamma)^{d} \right\rceil$ and $d$ is the dimension of $\Theta$. Then, $\forall\theta \in \Theta$ with probability at least $1-\delta$:
    \begin{equation} \label{eq:bound_swad}
        \begin{aligned}
            \mathcal{E}_{T}(\theta) & \leq \frac{1}{2} \operatorname{Div}(p_S, p_T) + \mathcal{E}_{S}(\theta)                                                                                                                    \\
                                   & \leq \frac{1}{2} \operatorname{Div}(p_S, p_T) + \mathcal{E}_{d_S}^{\gamma}(\theta)+\max _{k} \sqrt{\frac{\left(v_{k}\left[\ln \left(n_S / v_{k}\right)+1\right]+\ln (N / \delta)\right)}{2 n_S}},%
        \end{aligned}%
    \end{equation}%
    where:
    \begin{itemize}%
        \vspace{-0.5em}%
        \item $\mathcal{E}_{T}(\theta) \triangleq \mathbb{E}_{(x,y)\sim p_T(X,Y)}[\ell\left(f_\theta(x);y\right)]$ is the expected risk on the target domain,
        \item $\operatorname{Div}(p_S, p_T) \triangleq 2 \sup _{A}\left|p_{S}(A)-p_{T}(A)\right|$ is a divergence between the source and target marginal distributions $p_S$ and $p_T$: it measures diversity shift.%
        \item $\mathcal{E}_{S}(\theta) \triangleq \mathbb{E}_{(x,y)\sim p_S(X,Y)}[\ell\left(f_\theta(x);y\right)]$ is the expected risk on the source domain,
        \item $\mathcal{E}_{d_S}^{\gamma}(\theta) \triangleq \max _{\|\Delta\| \leq \gamma} \mathcal{E}_{d_S}(\theta+\Delta)$
        (where $\mathcal{E}_{d_S}(\theta+\Delta) \triangleq \mathbb{E}_{(x,y)\in d_S}[\ell\left(f_{\theta+\Delta}(x);y\right)]$)
        is the robust empirical loss on source training dataset $d_S$ from $S$ of size $n_S$,
        \item $v_{k}$ is a VC dimension of each $\Theta_{k}$.%
    \end{itemize}%
    \label{theorem:swad}%
    \vspace{-0.5em}%
\end{theorem}%
Previous understanding of WA's success in \ood relied on this upper-bound, where $\mathcal{E}_{d_S}^{\gamma}(\theta)$ involves the solution's flatness.
This is usually empirically analyzed by the trace of the Hessian \cite{pmlr-v70-dinh17b,petzka2021relative,yao2020pyhessian}: indeed, with a second-order Taylor approximation around the local minima $\theta$ and $h$ the Hessian's maximum eigenvalue, $\mathcal{E}_{d_S}^{\gamma}(\theta) \approx \mathcal{E}_{d_S}(\theta) + h \times \gamma^2$.

In the following subsections, we show that this inequality does not fully explain the exceptional performance of WA on DomainBed \cite{gulrajani2021in}. Moreover, we illustrate that our bias-variance-covariance-locality addresses these limitations.
\subsection{Flatness does not act on distribution shifts}
\label{app:flatness_distribution_shifts}
The flatness-based analysis is not specific to \ood.
Indeed, the upper-bound in \Cref{eq:bound_swad} sums up two noninteracting terms: a domain divergence $\mathrm{Div}(p_S,p_T)$ that grows in \ood and $\mathcal{E}_{d_S}^{\gamma}(\theta)$ that measures the \iid flatness.
The flatness term can indeed be reduced empirically with WA: yet, it does not tackle the domain gap.
In fact, \Cref{eq:bound_swad} states that additional flatness reduces the upper bound of the error similarly no matter the strength of the distribution shift, thus as well \ood than \iid.
In contrast, our analysis shows that variance (which grows with diversity shift, see \Cref{subsec:expression_ood_var}) is tackled for large $M$: our error is controlled even under large diversity shift. This is consistent with our experiments in \Cref{table:db_all_training}.
Our analysis also explains why WA cannot tackle correlation shift (where bias dominates, see \Cref{app:failure_corr_shift}), a limitation \cite{cha2021wad} does not illustrate.

\subsection{SAM leads to flatter minimas but worse \ood performance}
\label{app:mav_better_than_sam}
The flatness-based analysis does not explain why WA outperforms other flatness-based methods in \ood.
We consider Sharpness-Aware Minimizer (SAM) \cite{foret2021}, another popular method to find flat minima based on minimax optimization: it minimizes the maximum loss around a neighborhood of the current weights $\theta$.
In \Cref{fig:soup:boxplot_masam}, we compare the flatness (\ie the Hessian trace computed with the package in \cite{yao2020pyhessian}) and accuracy of ERM, MA \cite{arpit2021ensemble} (a WA strategy) and SAM \cite{foret2021} when trained on the \enquote{Clipart}, \enquote{Product} and \enquote{Photo} domains from OfficeHome \cite{venkateswara2017deep}: they are tested \ood on the fourth domain \enquote{Art}.
Analyzing the second and the third rows of \Cref{fig:home0:boxplot_masam_train_hess,fig:home0:boxplot_masam_ood_hess}, we observe that SAM indeed finds flat minimas (at least comparable to MA), both in training (\iid) and test (\ood). However, this is not reflected in the  \ood accuracies in \Cref{fig:soup:boxplot_masam_ood_acc}, where MA outperforms SAM.
As reported in \Cref{table:db_swad_all}, similar experiments across more datasets lead to the same conclusions in \cite{cha2021wad}.
In conclusion, flatness is not sufficient to explain why WA works so well in \ood, because SAM has similar flatness but worse \ood results.
In contrast, we highlight in this paper that WA succeeds in \ood by reducing the impact of variance thanks to its similarity with prediction ensembling \cite{Lakshminarayanan2017} (see \Cref{lemma:wa_ensembling}), a privileged link that SAM does not benefit from.%
\begin{table}[h]%
    \caption{\textbf{Accuracy ($\uparrow$) on DomainBed for SWAD}, taken from Table 4 in \cite{cha2021wad}}%
    \centering
    \adjustbox{width=0.9\textwidth}{
        \begin{tabular}{lccccc|c}
            \toprule
                       & \textbf{PACS}  & \textbf{VLCS}  & \textbf{OfficeHome} & \textbf{TerraInc} & \textbf{DomainNet} & \textbf{Avg}. ( $\Delta)$ \\
            \midrule
            ERM        & 85.5 $\pm$ 0.2 & 77.5 $\pm$ 0.4 & 66.5 $\pm$ 0.3      & 46.1 $\pm$ 1.8    & 40.9 $\pm$ 0.1     & 63.3                      \\
            SWAD \cite{cha2021wad} + ERM & \textbf{88.1} $\pm$ 0.1 & 79.1 $\pm$ 0.1 & \textbf{70.6} $\pm$ 0.2      & \textbf{50.0} $\pm$ 0.3    & \textbf{46.5} $\pm$ 0.1     & \textbf{66.9}(+3.6)                \\
            \midrule
            SAM \cite{foret2021}        & 85.8 $\pm$ 0.2 & \textbf{79.4} $\pm$ 0.1 & 69.6 $\pm$ 0.1      & 43.3 $\pm$ 0.7    & 44.3 $\pm$ 0.0     & 64.5                      \\
            SWAD \cite{cha2021wad} + SAM \cite{foret2021} & 87.1 $\pm$ 0.2 & 78.5 $\pm$ 0.2 & 69.9 $\pm$ 0.1      & 45.3 $\pm$ 0.9    & \textbf{46.5} $\pm$ 0.1     & 65.5(+1.0)                \\
            \bottomrule
        \end{tabular}
    }
    \label{table:db_swad_all}%
\end{table}%
\FloatBarrier
\clearpage
\begin{figure}[t]%
    \centering%
    \begin{subfigure}{.33\textwidth}%
        \centering%
        \includegraphics[width=\linewidth]{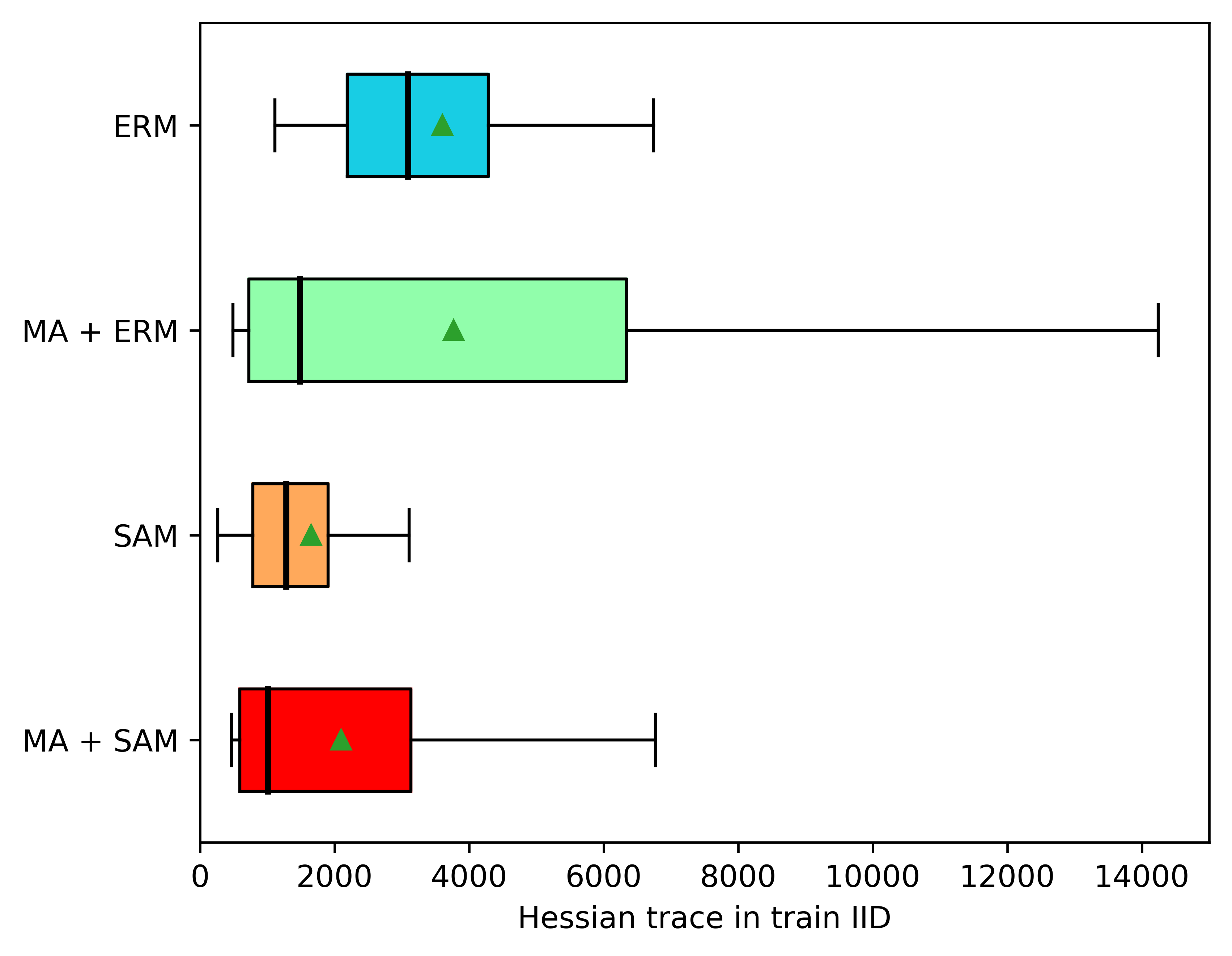}%
        \caption{Hessian trace ($\downarrow$) in train.}%
        \label{fig:home0:boxplot_masam_train_hess}%
    \end{subfigure}%
    \begin{subfigure}{.33\textwidth}%
        \centering%
        \includegraphics[width=\linewidth]{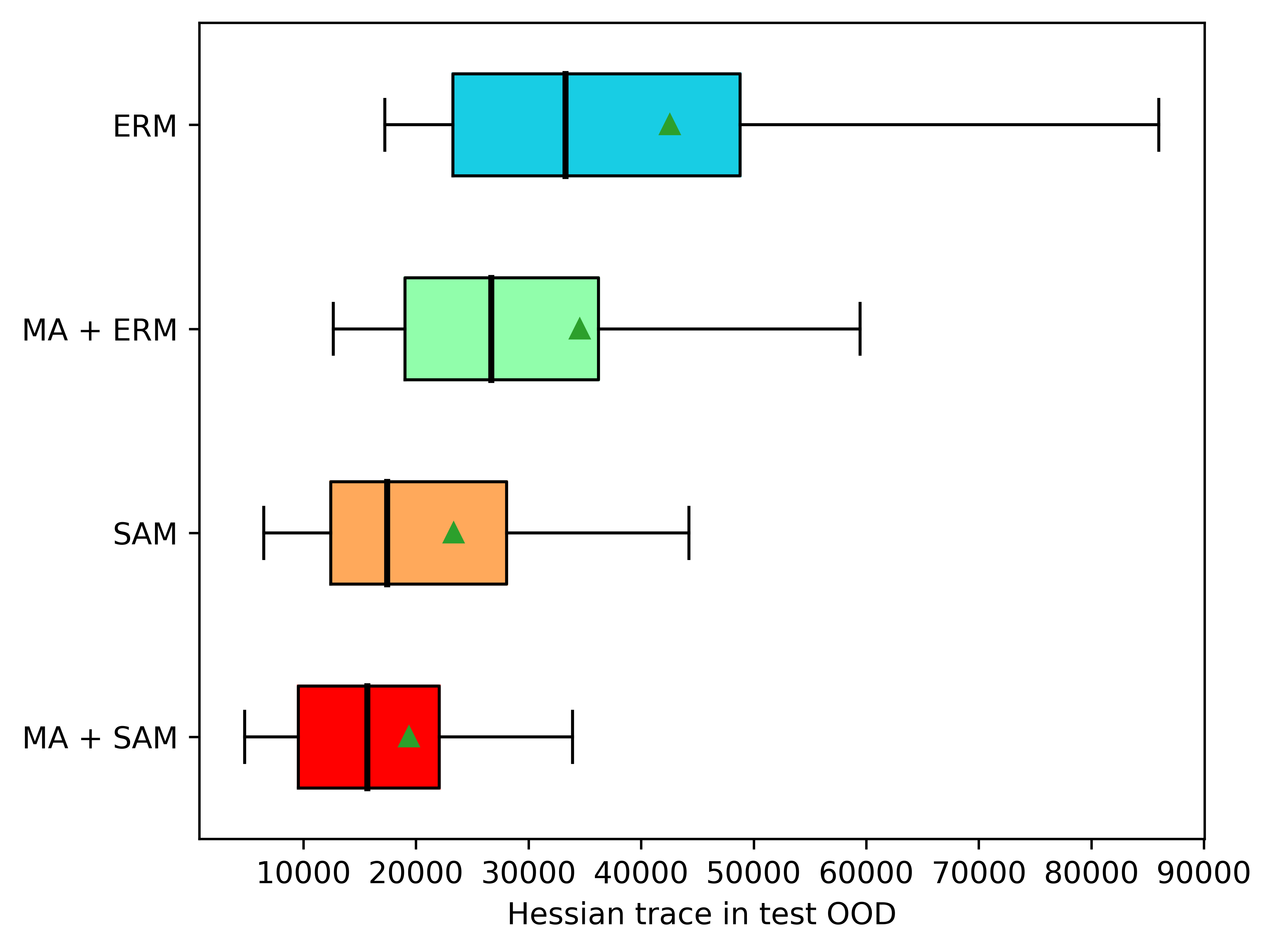}%
        \caption{Hessian trace ($\downarrow$) in test \ood.}%
        \label{fig:home0:boxplot_masam_ood_hess}%
    \end{subfigure}%
    \begin{subfigure}{.33\textwidth}%
        \centering%
        \includegraphics[width=\linewidth]{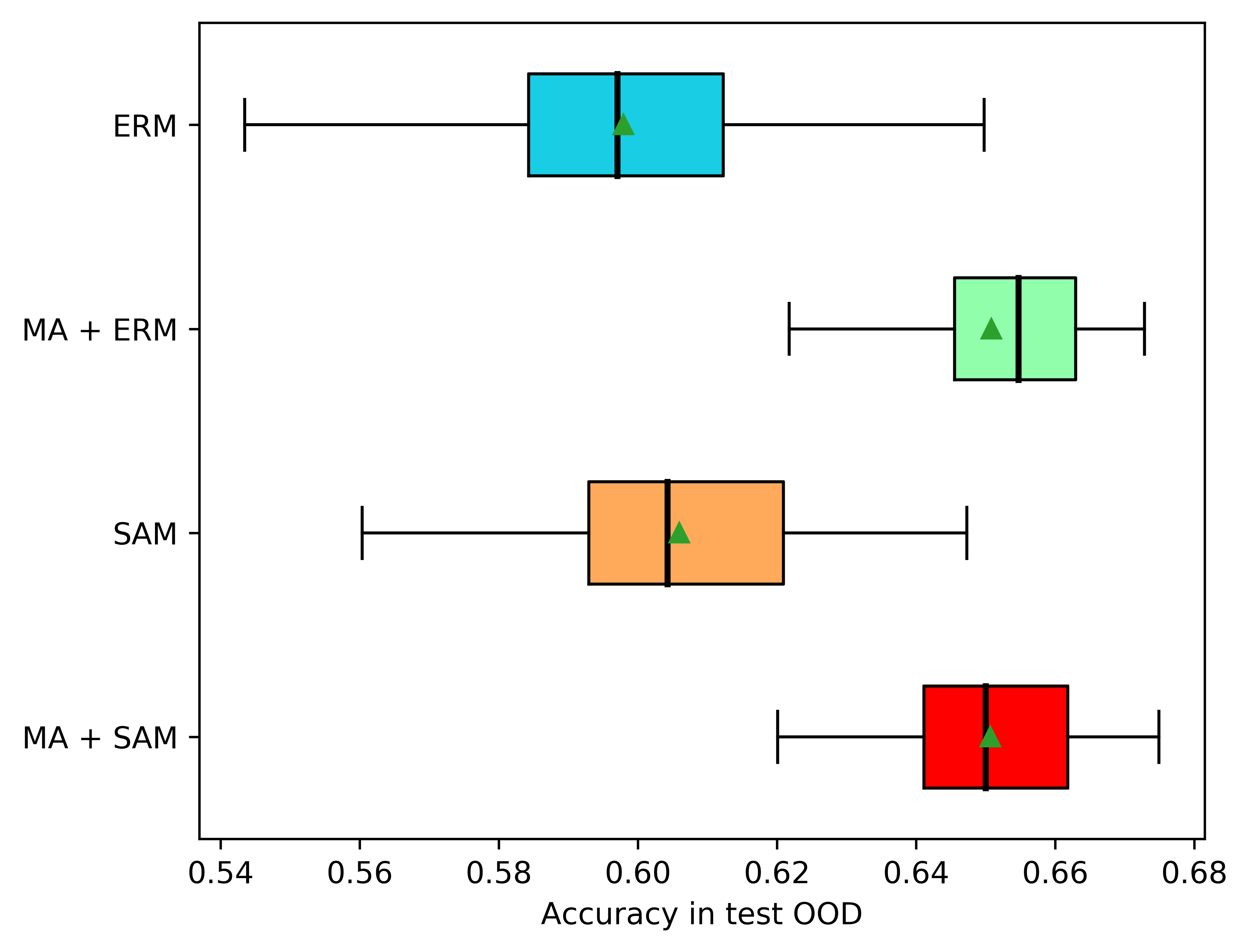}%
        \caption{Accuracy ($\uparrow$) in test \ood.}
        \label{fig:soup:boxplot_masam_ood_acc}%
    \end{subfigure}%
    \caption{MA \cite{arpit2021ensemble} (a WA strategy) and SAM \cite{foret2021} similarly improve flatness. When combined, they further improve flatness. Yet, MA outperforms SAM and beats MA + SAM in \ood accuracy on domain \enquote{Art} from OfficeHome.}%
    \label{fig:soup:boxplot_masam}%
\end{figure}%
\begin{table}[h]%
    \caption{\reb{\textbf{Accuracy ($\uparrow$) impact of including SAM} on domain \enquote{Art} from OfficeHome.}}
    \centering
    \adjustbox{width=0.7\textwidth}{
        \reb{\begin{tabular}{llll}
            \toprule
            \textbf{Algorithm}          & \textbf{Weight selection} & ERM            & SAM \cite{foret2021} \\
            \midrule
            No DiWA                         & N/A                       & 62.9 $\pm$ 1.3 & 63.5 $\pm$ 0.5       \\
            DiWA                        & Restricted: $M \leq 20$   & 66.7 $\pm$ 0.1 & 65.4 $\pm$ 0.1       \\
            DiWA                        & Uniform: $M = 20$         & 67.3 $\pm$ 0.3 & 66.7 $\pm$ 0.2       \\
            DiWA$^{\dagger}$            & Uniform: $M = 60$         & \textbf{67.7}           & \textbf{67.4}                 \\
            \bottomrule
        \end{tabular}}}
    \label{table:db_home0_sam}%
\end{table}

\subsection{WA and SAM are not complementary in \ood when variance dominates}%
\begin{wrapfigure}[20]{hR!}{0.34\textwidth}
    \vspace{-0.5em}%
    \centering%
    \includegraphics[width=0.34\textwidth]{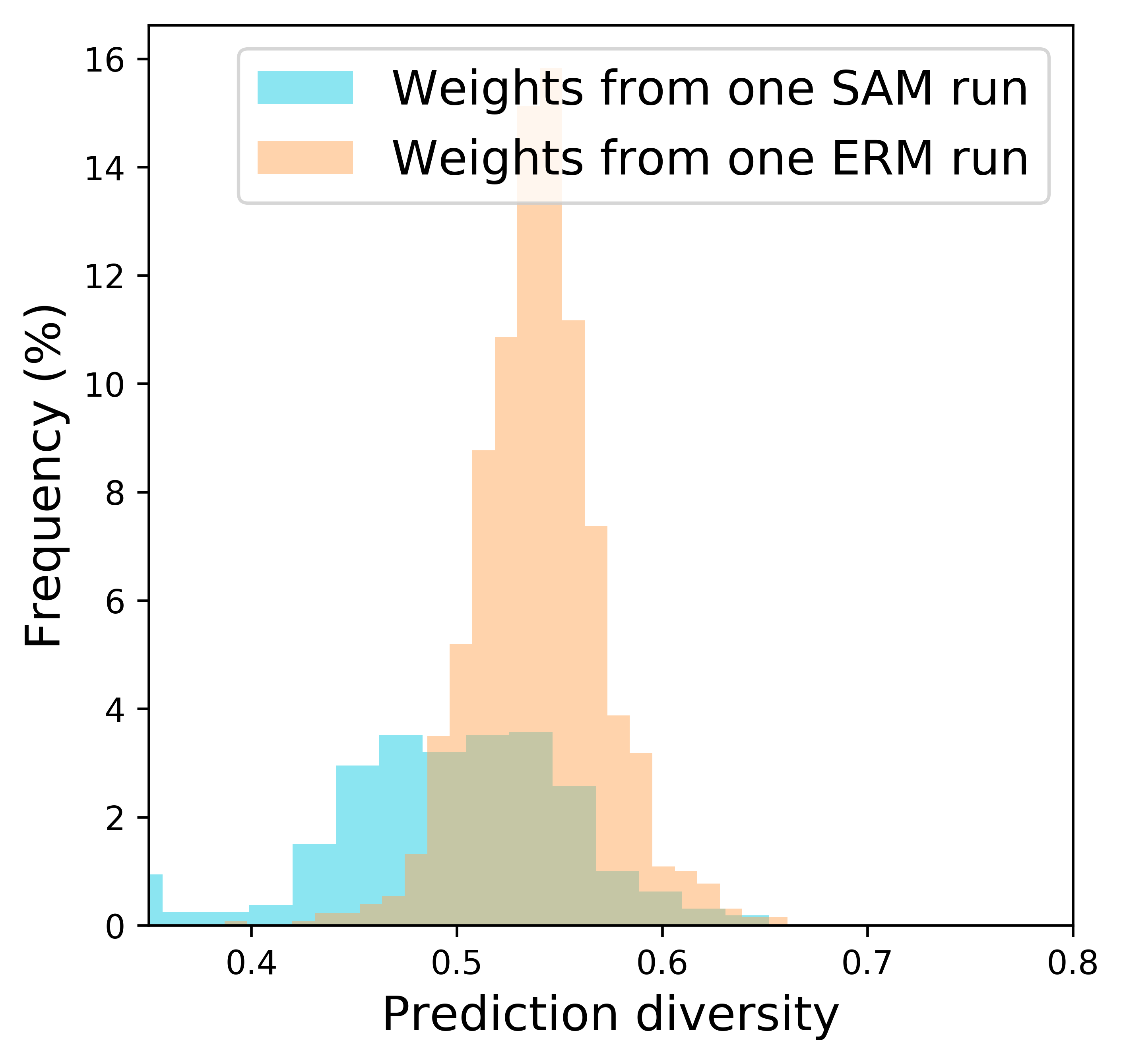}%
    \caption{Prediction diversity in ratio-error \cite{aksela2003comparison} ($\uparrow$) on domain \enquote{Art} from OfficeHome. Checkpoints along a SAM run are less diverse than along an ERM run.}%
    \label{fig:home0_hist_ermsam_dr}%
\end{wrapfigure}%
\label{app:mav_sam_failure}
We investigate a similar inconsistency when combining these two flatness-based methods.
As argued in \cite{questionsflatminima22}, we confirm in \Cref{fig:home0:boxplot_masam_train_hess,fig:home0:boxplot_masam_ood_hess} that MA + SAM leads to flatter minimas than MA alone (\ie with ERM) or SAM alone.
Yet, MA does not benefit from SAM in \Cref{fig:soup:boxplot_masam_ood_acc}.
\cite{cha2021wad} showed an even stronger result in \Cref{table:db_swad_all}: SWAD + ERM performs better than SWAD + SAM.
\reb{We recover similar findings in \Cref{table:db_home0_sam}: DiWA performs worse when SAM is applied in each training run}.

This behavior is not explained by \Cref{theorem:swad}, which states that more flatness should improve \ood generalization. Yet it is explained by our diversity-based analysis.
Indeed, we observe in \Cref{fig:home0_hist_ermsam_dr} that the diversity across two checkpoints along a SAM trajectory is much lower than along a standard ERM trajectory (with SGD).
We speculate that this is related to the recent empirical observation made in \cite{dallev2}: \enquote{the rank of the CLIP representation space is drastically reduced when training CLIP with SAM}.
Under diversity shift, variance dominates (see \Cref{eq:int_var}): in this setup, the gain in accuracy of models trained with SAM cannot compensate the decrease in diversity. This explains why WA and SAM are not complementary under diversity shift: in this case, variance is large.
\FloatBarrier
\clearpage

\section{Proofs}
\label{app:proof}

\subsection{WA loss derivation}
\label{app:wa_loss}
\begin{lemma*}[\ref{lemma:wa_ensembling}]
    Given $\{\theta_m\}_{m=1}^M$ with learning procedures $L_S^M\triangleq\{l_S^{(m)}\}_{m=1}^M$. Denoting $\Delta_{L_S^M}=\max_{m=1}^{M}\left\|\theta_m-\theta_{\text{WA}}\right\|_2$, $\forall (x,y) \in \X \times \Y$:%
    \begin{equation*}
        f_{\text{WA}}(x) = f_{\text{ENS}}(x) + O(\Delta^2_{L_S^M}) \text{~and~} \ell\left(f_{\text{WA}}(x), y\right) = \ell\left(f_{\text{ENS}}(x) , y\right) + O(\Delta_{L_S^M}^2).%
    \end{equation*}%
\end{lemma*}
\begin{proof}
This proof has two components:
\begin{itemize}
    \item to establish the functional approximation, as \cite{izmailov2018}, it performs Taylor expansion of the models' predictions at the first order.
    \item to establish the loss approximation, as \cite{Wortsman2022ModelSA}, it performs Taylor expansion of the loss at the first order.
\end{itemize}
\paragraph{Functional approximation}
With a Taylor expansion at the first order of the models' predictions \wrt parameters $\theta$:
\begin{align*}
    f_{\theta_m}&=f_{\text{WA}}+\nabla f_{\text{WA}}^\intercal \Delta_m+ O\left(\left\|\Delta_m\right\|_2^2\right)\\
    f_{\text{ENS}}-f_{\text{WA}}&=\frac{1}{M} \sum_{m=1}^{M}\left(\nabla f_{\text{WA}}^\intercal \Delta_m+O\left(\left\|\Delta_m\right\|_2^{2}\right)\right)
\end{align*}
Therefore, because $\sum_{m=1}^{M} \Delta_m=0$,
\begin{equation} \label{eq:diff_fens_fwa}
    f_{\text{ENS}}-f_{\text{WA}} = O\left(\Delta^2\right) \text{~where~}\Delta=\max_{m=1}^{M}\left\|\Delta_m\right\|_2.
\end{equation}
\paragraph{Loss approximation}
With a Taylor expansion at the zeroth order of the loss \wrt its first input and injecting \Cref{eq:diff_fens_fwa}:
\begin{align*}
    \ell\left(f_{\text{ENS}}(x) ; y\right)&=\ell\left(f_{\text{WA}}(x); y\right)+O\left(\left\|f_{\text{ENS}}(x)-f_{\text{WA}}(x)\right\|_2\right) \\
     \ell\left(f_{\text{ENS}}(x) ; y\right) &= \ell\left(f_{\text{WA}}(x); y\right) + O\left(\Delta^2\right).
\end{align*}
\end{proof}
\subsection{Bias-variance-covariance-locality decomposition}

\begin{remark} \label{remark:identical_assumption}
    Our result in \Cref{prop:b_var_cov} is simplified by leveraging the fact that the learning procedures $L_S^M=\{l_S^{(m)}\}_{m=1}^M$ are identically distributed (i.d.). This assumption naturally holds for DiWA which selects weights from different runs with i.i.d. hyperparameters.
    It may be less obvious why it applies to MA \cite{arpit2021ensemble} and SWAD \cite{cha2021wad}.
    It is even false if the weights $\{\theta(l_S^{(m)})\}_{m=1}^M$ are defined as being taken sequentially along a training trajectory, \ie when $0\leq i < j \leq M$ implies that $l_S^{(i)}$ has fewer training steps than $l_S^{(j)}$.
    We propose an alternative indexing strategy to respect the i.d. assumption.
    Given $M$ weights selected by the weight selection procedure, we draw without replacement the $M$ weights, \ie $\theta(l_S^{(i)})$ refers to the $i^{th}$ sampled weights.
    With this procedure, all weights are i.d. as they are uniformly sampled.
    Critically, their WA are unchanged for the two definitions.%
\end{remark}
\label{app:proof_bvc}
\begin{proposition*}[\ref{prop:b_var_cov}]
    Denoting $\bar{f}_S\left(x\right) = \mathbb{E}_{l_S} \left[f\left(x,\theta\left(l_S\right)\right)\right]$, under identically distributed learning procedures $L_S^M\triangleq\{l_S^{(m)}\}_{m=1}^M$, the expected generalization error on domain $T$ of $\theta_{\text{WA}}(L_S^M)\triangleq\frac{1}{M} \sum\nolimits_{m=1}^{M} \theta_m$ over the joint distribution of $L_S^M$ is:%
    \begin{equation*}\tag{\ref{eq:b_var_cov}}
        \begin{aligned}
             \mathbb{E}_{L_S^M}\mathcal{E}_T(\theta_{\text{WA}}(L_S^M)) &= \mathbb{E}_{(x,y)\sim p_T}\Big[\biasb^2(x, y)+\frac{1}{M} \varb(x)+\frac{M-1}{M} \covb(x)\Big] + O(\Deltab^2), \\
             \text{where~}\biasb(x,y)&=y-\bar{f}_S\left(x\right), \\
             \text{and~}\varb(x)&=\mathbb{E}_{l_S}\left[\left(f(x, \theta(l_S)) - \bar{f}_S\left(x\right)\right)^{2}\right], \\
             \text{and~}\covb(x)&= \mathbb{E}_{l_S,l_S'}\left[\left(f(x,\theta(l_S))-\bar{f}_S\left(x\right)\right)\left(f(x,\theta(l_S')))-\bar{f}_S\left(x\right)\right)\right], \\
             \text{and~}\Deltab^2&=\mathbb{E}_{L_S^M}\Delta_{L_S^M}^2 \text{~with~} \Delta_{L_S^M}=\max_{m=1}^{M}\left\|\theta_m-\theta_{\text{WA}}\right\|_2.%
        \end{aligned}%
    \end{equation*}%
    $\covb$ is the prediction covariance between two member models whose weights are averaged.
    The locality term $\Deltab^2$ is the expected squared maximum distance between weights and their average.
\end{proposition*}
\begin{proof}
    This proof has two components:
    \begin{itemize}
        \item it follows the bias-variance-covariance decomposition from \cite{ueda1996generalization,brown2005between} for functional ensembling. It is tailored to WA by assuming that learning procedures are identically distributed.
        \item it injects the obtained equation into \Cref{lemma:wa_ensembling} to obtain the \Cref{prop:b_var_cov} for WA.
    \end{itemize}
    \paragraph{BVC for ensembling with identically distributed learning procedures}
    With $\bar{f}_S\left(x\right) = \mathbb{E}_{l_S} \left[f\left(x,\theta\left(l_S\right)\right)\right]$, we recall the bias-variance decomposition \cite{kohavi1996bias} (\Cref{eq:b_v}):
    \begin{align*}
        \mathbb{E}_{l_S}\mathcal{E}_T(\theta(l_S)) & = \mathbb{E}_{(x,y)\sim p_T}\Big[\biasb(x,y)^{2}+\varb(x)\Big],\\
        \text{where~}\biasb(x,y) & =  \operatorname{Bias}\{f|(x,y)\} = y-\bar{f}_S\left(x\right), \\
        \text{and~}\varb(x) &= \operatorname{Var}\{f|x\} = \mathbb{E}_{l_S}\left[\left(f(x, \theta(l_S)) - \bar{f}_S\left(x\right)\right)^{2}\right]. \\
    \end{align*}
    Using $f_{\text{ENS}}\triangleq f_{\text{ENS}}(\cdot, \{\theta(l_S^{(m)})\}_{m=1}^M) \triangleq \frac{1}{M} \sum_{m=1}^{M} f(\cdot, \theta(l_S^{(m)}))$ in this decomposition yields,
    \begin{equation} \label{eq:b_v_ens}
        \mathbb{E}_{L_S^M}\mathcal{E}_T(\{\theta(l_S^{(m)})\}_{m=1}^M)=\mathbb{E}_{x\sim p_T}\left[\operatorname{Bias}\left\{f_{\text{ENS}} \mid (x,y)\right\}^{2}+\operatorname{Var}\left\{f_{\text{ENS}} \mid x\right\}\right].
    \end{equation}
    As $f_{\text{ENS}}$ depends on $L_S^M$, we extend the bias into:
    \begin{align*}
        \operatorname{Bias}\left\{f_{\text{ENS}} \mid (x,y)\right\}&=y-\mathbb{E}_{L_S^M}\left[\frac{1}{M} \sum_{m=1}^{M} f(x, \theta(l_S^{(m)}))\right]=y-\frac{1}{M} \sum_{m=1}^{M} \mathbb{E}_{l_S^{(m)}}\left[f(x, \theta(l_S^{(m)}))\right]
    \end{align*}
    Under identically distributed $L_S^M\triangleq\{l_S^{(m)}\}_{m=1}^M$,
    $$
        \frac{1}{M} \sum_{m=1}^{M} \mathbb{E}_{l_S^{(m)}}\left[y-f(x, \theta(l_S^{(m)}))\right]=\mathbb{E}_{l_S}\left[y-f(x, \theta(l_S))\right]=\operatorname{Bias}\{f|(x,y)\}.
    $$
    Thus the bias of ENS is the same as for a single member of the WA.
    
    Regarding the variance:
    \begin{align*}
        &\operatorname{Var}\left\{f_{\text{ENS}} \mid x\right\}=\mathbb{E}_{L_S^M}\left[\left(\frac{1}{M} \sum_{m=1}^{M} f(x, \theta(l_S^{(m)}))-\mathbb{E}_{L_S^M}\left[\frac{1}{M} \sum_{m=1}^{M} f(x, \theta(l_S^{(m)}))\right]\right)^{2}\right].\\
    \end{align*}
    Under identically distributed $L_S^M\triangleq\{l_S^{(m)}\}_{m=1}^M$,
    \begin{align*}
        &\operatorname{Var}\left\{f_{\text{ENS}} \mid x\right\} = \frac{1}{M^{2}} \sum_{m=1}^{M} \mathbb{E}_{l_S}\left[\left(f(x, \theta(l_S))-\mathbb{E}_{l_S}\left[f(x, \theta(l_S))\right]\right)^{2}\right]+ \\
        &\quad\quad\quad\quad \frac{1}{M^{2}} \sum_{m} \sum_{m^{\prime} \neq m} \mathbb{E}_{l_S, l_S^\prime}\left[\left(f(x, \theta(l_S))-\mathbb{E}_{l_S}\left[f(x, \theta(l_S))\right]\right)\left(f(x, \theta(l_S^\prime))-\mathbb{E}_{l_S^\prime}\left[f(x, \theta(l_S^\prime))\right]\right)\right] \\
        &= \frac{1}{M} \mathbb{E}_{l_S}\left[\left(f(x, \theta(l_S))-\mathbb{E}_{l_S}\left[f(x, \theta(l_S))\right]\right)^{2}\right] + \\
        &\quad\quad\quad\quad \frac{M-1}{M}\mathbb{E}_{l_S, l_S^\prime}\left[\left(f(x, \theta(l_S))-\mathbb{E}_{l_S}\left[f(x, \theta(l_S))\right]\right)\left(f(x, \theta(l_S^\prime))-\mathbb{E}_{l_S^\prime}\left[f(x, \theta(l_S^\prime))\right]\right)\right] \\
        &=\frac{1}{M} \varb\left(x\right)+\left(1-\frac{1}{M}\right) \covb\left(x\right).
    \end{align*}
    The variance is split into the variance of a single member (divided by $M$) and a covariance term.
    \paragraph{Combination with \Cref{lemma:wa_ensembling}}
    We recall that per \Cref{lemma:wa_ensembling},
    \begin{align*}
        \ell\left(f_{\text{WA}}(x), y\right) = \ell\left(f_{\text{ENS}}(x) , y\right) + O(\Delta_{L_S^M}^2).
    \end{align*}
    Then we have:
    \begin{align*}
        \mathcal{E}_T(\theta_{\text{WA}}(L_S^M))&=\mathbb{E}_{(x,y) \sim p_T}[\ell(f_{\text{WA}}(x),y)]\\
        &=\mathbb{E}_{(x,y) \sim p_T}[\ell(f_{\text{ENS}}(x),y)] + O(\Delta_{L_S^M}^2) = \mathcal{E}_T(\{\theta(l_S^{(m)})\}_{m=1}^M) + O(\Delta_{L_S^M}^2), \\
        \mathbb{E}_{L_S^M} \mathcal{E}_T(\theta_{\text{WA}}(L_S^M))&=\mathbb{E}_{L_S^M}\mathcal{E}_T(\{\theta(l_S^{(m)})\}_{m=1}^M) + O(\mathbb{E}_{L_S^M}[\Delta_{L_S^M}^2]).
    \end{align*}
    We eventually obtain the result:
    \begin{align*}
        \mathbb{E}_{L_S^M}\mathcal{E}_T(\theta_{\text{WA}}(L_S^M)) = \mathbb{E}_{(x,y)\sim p_T}\Big[\biasb(x,y)^{2}+\frac{1}{M} \varb(x)+\frac{M-1}{M} \covb(x)\Big] + O(\Deltab^2).
    \end{align*}
\end{proof}%
\subsection{Bias, correlation shift and support mismatch}
\label{app:bias_correlation}
We first present in \Cref{app:app_bias_full} a decomposition of the \ood bias without any assumptions.
We then justify in \Cref{app:bias_correlation_ass_noiidbias} the simplifying \Cref{ass:no_bias_iid} from \Cref{subsec:expression_ood_bias}.
\subsubsection{\ood bias}
\label{app:app_bias_full}
\begin{proposition}[\ood bias]
    \label{prop:app_bias_full}
    Denoting $\bar{f}_{S}\left(x\right) = \mathbb{E}_{l_S}[f\left(x,\theta\left(l_S\right)\right)]$, the bias is:
    \begin{align*}
        &\mathbb{E}_{(x,y)\sim p_T} [\biasb^2(x, y)] = \int_{\X_T \cap \X_S} \left(f_T\left(x\right)-f_S\left(x\right)\right)^{2} p_T(x) dx                                          & (\text{Correlation shift})                                               \\
                         & + \int_{\X_T \cap \X_S} \left(f_{S}(x)-\bar{f}_{S}\left(x\right)\right)^{2} p_T(x) dx                                                            & (\text{Weighted \iid bias})                                               \\
                         & + \int_{\X_T \cap \X_S} 2 \left(f_T\left(x\right) - f_S\left(x\right)\right) \left(f_S\left(x\right) - \bar{f}_{S}\left(x\right)\right) p_T(x)dx & (\text{Interaction \iid bias and corr. shift})                     \\
                         & + \int_{\X_T \setminus \X_S} \left(f_T\left(x\right)-\bar{f}_{S}\left(x\right)\right)^{2} p_T(x) dx.                                   & (\text{Support mismatch}) %
    \end{align*}%
\end{proposition}%

\begin{proof}
    This proof is original and based on splitting the \ood bias in and out of $\X_S$:
    \begin{align*}
        &\mathbb{E}_{(x,y)\sim p_T}[\biasb^2(x, y)] = \mathbb{E}_{(x,y)\sim p_T} \left(y -\bar{f}_{S}\left(x\right)\right)^{2}                          \\
                                         & = \int_{\X_T} \left(f_T\left(x\right)-\bar{f}_{S}\left(x\right)\right)^{2} p_T(x) dx                \\
                                         & =\int_{\X_T \cap \X_S} \left(f_T\left(x\right)-\bar{f}_{S}\left(x\right)\right)^{2} p_T(x) dx +\int_{\X_T \setminus \X_S} \left(f_T\left(x\right)-\bar{f}_{S}\left(x\right)\right)^{2} p_T(x) dx.
    \end{align*}
    To decompose the first term, we write $\forall x \in \X_S$, $ - \bar{f}_{S}\left(x\right) = - f_S\left(x\right) + \left(f_{S}(x) - \bar{f}_{S}\left(x\right)\right)$.
    \begin{align*}
        &\int_{\X_T \cap \X_S} \left(f_T\left(x\right)-\bar{f}_{S}\left(x\right)\right)^{2} p_T(x) dx=\int_{\X_T \cap \X_S} \left((f_T\left(x\right) - f_S\left(x\right)) + \left(f_{S}(x) - \bar{f}_{S}\left(x\right)\right) \right)^{2} p_T(x) dx \\
        & =\int_{\X_T \cap \X_S} \left(f_T\left(x\right) - f_S\left(x\right)\right)^{2} p_T(x)dx + \int_{\X_T \cap \X_S} \left(f_{S}(x) - \bar{f}_{S}\left(x\right)\right)^{2} p_T(x)dx \\
        & + \int_{\X_T \cap \X_S} 2 \left(f_T\left(x\right) - f_S\left(x\right)\right) \left(f_S\left(x\right) - \bar{f}_{S}\left(x\right)\right) p_T(x)dx.
    \end{align*}%
\end{proof}
The four terms can be qualitatively analyzed:
\begin{itemize}
    \item The first term measures differences between train and test labelling function. By rewriting $\forall x \in \X_T \cap \X_S$, $f_T(x) \triangleq \mathbb{E}_{p_T}[Y|X=x]$ and $f_S(x) \triangleq \mathbb{E}_{p_S}[Y|X=x]$, this term measures whether conditional distributions differ.
    This recovers a similar expression to the correlation shift formula from \cite{ye2021odbench}.
    \item The second term is exactly the \iid bias, but weighted by the marginal distribution $p_T(X)$.
    \item The third term  $\int_{\X_T \cap \X_S} 2 \left(f_T\left(x\right) - f_S\left(x\right)\right) \left(f_S\left(x\right) - \bar{f}_{S}\left(x\right)\right) p_T(x)dx$ measures to what extent the \iid bias compensates the correlation shift. It can be negative if (by chance) the \iid bias goes in opposite direction to the correlation shift.
    \item The last term measures support mismatch between test and train marginal distributions. It lead to the \enquote{No free lunch for learning representations for DG} in \cite{ruan2022optimal}. The error is irreducible because \enquote{outside of the source domain, the label distribution is unconstrained}: \enquote{for any domain which gives some probability mass on an example that has not been seen during training, then all \textelp{} labels for that example} are possible.
\end{itemize}

\subsubsection{Discussion of the small \iid bias \Cref{ass:no_bias_iid}}
\label{app:bias_correlation_ass_noiidbias}
\Cref{ass:no_bias_iid} states that $\exists \epsilon > 0 \text{~small~s.t.~} \forall x\in \X_S, |f_{S}\left(x\right)-\bar{f}_{S}\left(x\right)|\leq \epsilon$ where $\bar{f}_{S}\left(x\right) = \mathbb{E}_{l_S} \left[f\left(x,\theta\left(l_S\right)\right)\right]$.
$\bar{f}_{S}$ is the expectation over the possible learning procedures $l_S=\{d_S, c\}$. Thus \Cref{ass:no_bias_iid} involves:
\begin{itemize}
    \item the network architecture $f$ which should be able to fit a given dataset $d_S$. This is realistic when the network is sufficiently parameterized, \ie when the number of weights $|\theta|$ is large.
    \item the expected datasets $d_S$ which should be representative enough of the underlying domain $S$; in particular the dataset size $n_S$ should be large.
    \item the sampled configurations $c$ which should be well chosen: the network should be trained for enough steps, with an adequate learning rate ...
\end{itemize}
For DiWA, this is realistic as it selects the weights with the highest training validation accuracy from each run. For SWAD \cite{cha2021wad}, this is also realistic thanks to their overfit-aware weight selection strategy. In contrast, this assumption may not perfectlty hold for MA \cite{arpit2021ensemble}, which averages weights starting from batch $100$ until the end of training: indeed, $100$ batches are not enough to fit the training dataset.%
\subsubsection{\ood bias when small \iid bias}
\label{app:simplified_bias}
We now develop our equality under \Cref{ass:no_bias_iid}.
\begin{proposition*}[\ref{prop:bias}. \ood bias when small \iid bias]
    With a bounded difference between the labeling functions $f_T-f_S$ on $\X_T \cap \X_S$, under \Cref{ass:no_bias_iid}, the bias on domain $T$ is:
    \begin{equation*} \tag{\ref{eq:bias}}
        \begin{aligned}
            &\mathbb{E}_{(x,y)\sim p_T}[\biasb^2(x, y)] = \text{Correlation shift} + \text{Support mismatch} + O(\epsilon), \\
            &\text{where~} \text{Correlation shift} = \int_{\X_T \cap \X_S} \left(f_T(x)-f_S(x)\right)^{2} p_T(x) dx, \\
            &\text{and~} \text{Support mismatch} = \int_{\X_T \setminus \X_S} \left(f_T(x)-\bar{f}_{S}\left(x\right)\right)^{2} p_T(x) dx.%
        \end{aligned}%
    \end{equation*}%
\end{proposition*}%
\begin{proof}
    We simplify the second and third terms from \Cref{prop:app_bias_full} under \Cref{ass:no_bias_iid}.

    \textbf{The second term} is $\int_{\X_T \cap \X_S} \left(f_{S}\left(x\right)-\bar{f}_{S}\left(x\right)\right)^{2} p_T(x) dx$. Under \Cref{ass:no_bias_iid}, $|f_{S}\left(x\right)-\bar{f}_{S}\left(x\right)|\leq \epsilon$. Thus the second term is $O(\epsilon^2)$.
    
    \textbf{The third term} is $\int_{\X_T \cap \X_S} 2 \left(f_T\left(x\right) - f_S\left(x\right)\right) \left(f_S\left(x\right) - \bar{f}_{S}\left(x\right)\right) p_T(x)dx$.
    As $f_T - f_S$ is bounded on $\X_S \cap \X_T$, $\exists K\geq0$ such that $\forall x \in \X_S$,
    $$
        |\left(f_T(x) - f_S(x)\right)\left(f_{S}(x) - \bar{f}_{S}\left(x\right)\right) p_T(x)| \leq K \left|f_{S}(x) - \bar{f}_{S}\left(x\right)\right|p_T(x) = O(\epsilon) p_T(x).
    $$
    Thus the third term is $O(\epsilon)$.
    
    Finally, note that we cannot say anything about $\bar{f}_{S}\left(x\right)$ when $x\in \X_T \setminus \X_S$.
\end{proof}%

To prove the previous equality, we needed a bounded difference between labeling functions $f_T-f_S$ on $\X_T \cap \X_S$.
We relax this bounded assumption to obtain an inequality in the following \Cref{prop:oodbias_smalliid_nobound}.
\begin{proposition}[\ood bias when small \iid bias without bounded difference between labeling functions]
\label{prop:oodbias_smalliid_nobound}
    Under \Cref{ass:no_bias_iid},
    \begin{equation}
        \begin{aligned}
            &\mathbb{E}_{(x,y)\sim p_T}[\biasb^2(x, y)] \leq 2 \times \text{Correlation shift} + \text{Support mismatch} + O(\epsilon^2)
        \end{aligned}%
    \end{equation}%
\end{proposition}%
\begin{proof}
    We follow the same proof as in \Cref{prop:app_bias_full}, except that we now use: $(a+b)^2\leq 2 (a^2 + b^2)$. Then,
    \begin{align*}
    &\int_{\X_T \cap \X_S} \left(f_T\left(x\right)-\bar{f}_{S}\left(x\right)\right)^{2} p_T(x) dx=\int_{\X_T \cap \X_S} \left((f_T\left(x\right) - f_S\left(x\right)) + \left(f_{S}(x) - \bar{f}_{S}\left(x\right)\right) \right)^{2} p_T(x) dx \\
    & \leq 2 \times \int_{\X_T \cap \X_S} \left(f_T\left(x\right) - f_S\left(x\right)\right)^{2} + \left(f_{S}(x) - \bar{f}_{S}\left(x\right)\right)^{2} p_T(x) dx \\
    & \leq 2 \times \int_{\X_T \cap \X_S} \left(f_T\left(x\right) - f_S\left(x\right)\right)^{2} p_T(x) dx + 2 \times \int_{\X_T \cap \X_S} \epsilon^2 p_T(x) dx \\
    & \leq 2 \times \int_{\X_T \cap \X_S} \left(f_T\left(x\right) - f_S\left(x\right)\right)^{2} p_T(x) dx + O(\epsilon^2)
    \end{align*}%
\end{proof}%

\subsection{Variance and diversity shift}
\label{app:var_diversity}
We prove the link between variance and diversity shift.
Our proof builds upon the similarity between NNs and GPs in the kernel regime, detailed in \Cref{app:nns_as_gps}.
We discuss our simplifying \Cref{ass:inter_sample} in \Cref{app:discussion_inter_sample}.
We present our final proof in \Cref{app:proof_ood_var}. We discuss the relation between variance and initialization in \Cref{remark:mmdk}.

\subsubsection{Neural networks as Gaussian processes}
\label{app:nns_as_gps}
We fix $d_S, d_T$ and denote $X_{d_S}=\{x_S\}_{(x_S,y_S)\in d_S}$, $X_{d_T}=\{x_T\}_{(x_T,y_T)\in d_T}$ their respective input supports.
We fix the initialization of the network.
$l_S$ encapsulates all other sources of randomness.
\begin{lemma}[Inspired from \cite{RasmussenW06}] \label{lemma:var_gp}
    Given a NN $f(\cdot, \theta(l_S))$ under \Cref{ass:infinite_width}, we denote $K$ its neural tangent kernel and $K(X_{d_S}, X_{d_S})\triangleq(K\left(x_S, x_S^\prime\right))_{x_S, x_S^\prime \in X_{d_S}^2} \in \mathbb{R}^{n_S\times n_S}$.
    Given $x\in\X$, we denote $K(x, X_{d_S})\triangleq[K\left(x, x_S\right)]_{x_S \in X_{d_S}} \in \mathbb{R}^{n_S}$. Then:
    \begin{equation}
        \begin{aligned}
            &\varb(x) = K(x, x)-K(x, X_{d_S})K(X_{d_S}, X_{d_S})^{-1} K(x, X_{d_S})^\intercal.
        \end{aligned}\label{eq:var}
    \end{equation}
\end{lemma}
\begin{proof}
    Under \Cref{ass:infinite_width}, NNs are equivalent to GPs.
    $\varb(x)$ is the formula of the variance of the GP posterior given by Eq. (2.26) in \cite{RasmussenW06}, when conditioned on $d_S$.
    This formula thus also applies to the variance $f(\cdot,\theta(l_S))$ when $l_S$ varies (at fixed $d_S$ and initialization).
\end{proof}
\begin{figure}[H]
    \vspace{-1.em}
    \centering
    \includegraphics[width=0.4\textwidth]{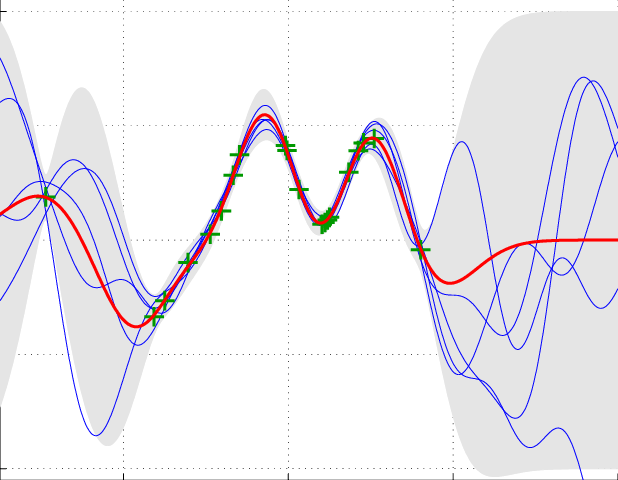}%
    \caption{\textbf{Mean and variance of a Gaussian process's prediction}. Image from \cite{perez2013gaussian}. Intuitively, variance grows when samples are distant from training samples.}
    \label{fig:gp}%
\end{figure}%

\subsubsection{Discussion of the same norm and low similarity \Cref{ass:inter_sample} on source dataset}
\label{app:discussion_inter_sample}
\Cref{lemma:var_gp} shows that the variance only depends on the input distributions $p(X)$ without involving the label distributions $p(Y|X)$.
This formula highlights that the variance is related to shifts in input similarities (measured by $K$) between $X_{d_S}$ and $X_{d_T}$.
Yet, a more refined analysis of the variance requires additional assumptions, in particular to obtain a closed-form expression of $K(X_{d_S}, X_{d_S})^{-1}$.
\Cref{ass:inter_sample} is useful because then $K(X_{d_S}, X_{d_S})$ is diagonally dominant and can be approximately inverted (see \Cref{app:proof_ood_var}).

The first part of \Cref{ass:inter_sample} assumes that $\exists \lambda_S$ such that all training inputs $x_S\in X_{d_S}$ verify $K(x_S, x_S)= \lambda_S$.
Note that this equality is standard in some kernel machine algorithms \cite{ah2010normalized,ghojogh2021reproducing,rennie2005} and is usually achieved by replacing $K(x,x^\prime)$ by $\lambda_S \frac{K(x,x^\prime)}{\sqrt{K(x,x)}\sqrt{K(x^\prime,x^\prime)}}, \forall(x,x^\prime)\in {(X_{d_S}\cup X_{d_T})}^2$. 
In the NTK literature, this equality is achieved without changing the kernel by normalizing the samples of $X_{d_S}$ such that they lie on the hypersphere; this input preprocessing was used in \cite{Lee2018DeepNN}.
This is theoretically based: for example, the NTK $K(x, x^\prime)$ for an architecture with an initial fully connected layer only depends on $\|x\|,\|x^\prime\|,\langle x, x^\prime\rangle$ \cite{yang2019}. Thus in the case where all samples from $X_{d_S}$ are preprocessed to have the same norm, the value of $K(x_S,x_S)$ does not depend on $x_S\in X_{d_S}$; we denote $\lambda_S$ the corresponding value.

The second part of \Cref{ass:inter_sample} states that $\exists 0 \leq \epsilon \ll \lambda_S, \text{s.t.~} \forall x_S, x_S^\prime \in X_{d_S}^2, x_S\neq x_S^\prime \Rightarrow |K(x_S,x_S^\prime)|\leq\epsilon$, \ie that training samples are dissimilar and do not interact.
This diagonal structure of the NTK \cite{Jacot2018}, with diagonal values larger than non-diagonal ones, is consistent with empirical observations from \cite{seleznova2022neural} at initialization.
Theoretically, this is reasonable if $K$ is close to the RBF kernel $K_{h}\left(x, x^{\prime}\right)=\exp (-\left\|x-x^{\prime}\right\|_{2}^{2}/h)$ where $h$ would be the bandwidth: in this case, \Cref{ass:inter_sample} is satisfied when training inputs are distant in pixel space.

We now provide an analysis of the variance where the diagonal assumption is relaxed.
Specifically, we provide the sketch for proving an upper-bound of the variance when the NTK has a block-diagonal structure.
This is indeed closer to the empirical observations in \cite{seleznova2022neural} at the end of training, consistently with the local elasticity property of NNs \cite{He2020The}.
We then consider the dataset $d_{S^\prime} \subset d_S$ made of one sample per block, to which \Cref{ass:inter_sample} applies.
As decreasing the size of a training dataset empirically reduces variance \cite{brain1999effect}, the variance of $f$ trained on $d_S$ is upper-bounded by the variance of $f$ trained on $d_{S^\prime}$; the latter is given by applying \Cref{prop:var} to $d_{S^\prime}$.
We believe that the proper formulation of this idea is beyond the scope of this article and best left for future theoretical work.

\subsubsection{Expression of \ood variance}
\label{app:proof_ood_var}

\begin{proposition*}[\ref{prop:var}]
    Given $f$ trained on source dataset $d_S$ (of size $n_S$) with NTK $K$, under \Cref{ass:infinite_width,ass:inter_sample}, the variance on dataset $d_T$ is:%
    \begin{equation*} \tag{\ref{eq:int_var}}
        \mathbb{E}_{x_T\in X_{d_T}}[\varb(x_T)] = \frac{n_S}{2\lambda_S}\text{MMD}^{2}(X_{d_S}, X_{d_T}) + \lambda_T - \frac{n_S}{2\lambda_S}\beta_T + O(\epsilon),%
    \end{equation*}
    with $\text{MMD}$ the empirical Maximum Mean Discrepancy in the RKHS of $K^2(x,y)=(K(x,y))^2$;$\lambda_T \triangleq \mathbb{E}_{x_T \in X_{d_T}} K\left(x_T, x_T\right)$ and $\beta_T \triangleq \mathbb{E}_{(x_T,x_T^\prime)\in X_{d_T}^2, x_T \neq x_T^\prime} K^2\left(x_T, x_T^{\prime}\right)$ the empirical mean similarities resp. measured between identical (\wrt $K$) and different (\wrt $K^2$) samples averaged over $X_{d_T}$.%
\end{proposition*}
\begin{proof}
Our proof is original and is based on the posterior form of GPs in \Cref{lemma:var_gp}.
Given $d_S$, we recall \Cref{eq:var} that states $\forall x\in\X$:
\begin{align*}
    \varb(x) =K(x, x)-K(x, X_{d_S})K(X_{d_S}, X_{d_S})^{-1} K(x, X_{d_S})^\intercal.
\end{align*}
Denoting $B=K(X_{d_S},X_{d_S})^{-1}$ with symmetric coefficients $b_{i,j}=b_{j,i}$, then
\begin{equation} \label{eq:inversion}
    \varb(x) = K(x, x)-\sum_{\substack{1 \leq i \leq n_S \\ 1 \leq j \leq n_S}} b_{i, j} K(x, x_S^i) K(x, x_S^j).
\end{equation}

\Cref{ass:inter_sample} states that
$K(X_{d_S}, X_{d_S})=A+H$ where $A=\lambda_S\mathbb{I}_{n_S}$ and $H=(h_{ij})_{\substack{1 \leq i \leq n_S \\ 1 \leq j \leq n_S}}$ with $h_{i,i}=0$ and $\max_{i,j}|h_{i,j}|\leq\epsilon$.

We fix $x_T \in X_{d_T}$ and determine the form of $B^{-1}$ in two cases: $\epsilon=0$ and $\epsilon\neq 0$.

\paragraph{Case when $\epsilon=0$}
We first derive a simplified result, when $\epsilon=0$.

Then, $b_{i,i}=\frac{1}{\lambda_S}$ and $b_{i,j}=0$ \sut
$$
    \varb(x_T)=K(x_T, x_T)-\sum_{x_S\in X_{d_S}} \frac{K(x_T, x_S)^2}{\lambda_S}=K(x,x) - \frac{n_S}{\lambda_S} \mathbb{E}_{x_S\in X_{d_S}}[K^2(x, x_S)]
$$
We can then write:
\begin{align*}
    &\mathbb{E}_{x_T\in X_{d_T}} [\varb(x_T)] = \mathbb{E}_{x_T\in X_{d_T}}[K(x_T, x_T)] - \frac{n_S}{\lambda_S} \mathbb{E}_{x_T\in X_{d_T}}[\mathbb{E}_{x_S\in X_{d_S}}[K^2(x_T, x_S)]] \\
    &\mathbb{E}_{x_T\in X_{d_T}} [\varb(x_T)] = \lambda_T - \frac{n_S}{\lambda_S} \mathbb{E}_{x_S\in X_{d_S}, x_T\in X_{d_T}}[K^2(x_T, x_S)]. \\
\end{align*}
We now relate the second term on the r.h.s. to a MMD distance.
As $K$ is a kernel, $K^2$ is a kernel and its MMD between $X_{d_S}$ and $X_{d_T}$ is per \cite{Gretton2012}:
\begin{align*}
    \text{MMD}^{2}(X_{d_S}, X_{d_T})= & \mathbb{E}_{x_S \neq x_S^\prime \in X_{d_S}^2}[K^2(x_S, x_S^\prime)]+\mathbb{E}_{x_T \neq x_T^\prime \in X_{d_T}^2}[K^2(x_T, x_T^\prime)]\\
    & \quad -2\mathbb{E}_{x_S\in X_{d_S}, x_T\in X_{d_T}}[K^2(x_T, x_S)].
\end{align*}
Finally, because $\epsilon=0$, $\mathbb{E}_{x_S \neq x_S^\prime \in X_{d_S}^2} K^2\left(x_S, x_S^{\prime}\right)=0$ \sut
\begin{align*}
    \mathbb{E}_{x_T\in X_{d_T}}[\varb(x_T)] &=\frac{n_S}{2\lambda_S}\text{MMD}^{2}(X_{d_S}, X_{d_T}) + \lambda_T\\
    & \quad\quad\quad -\frac{n_S}{2\lambda_S}\Big(\mathbb{E}_{x_T \neq x_T^\prime \in X_{d_T}^2} K^2\left(x_T, x_T^{\prime}\right)+\mathbb{E}_{x_S \neq x_S^\prime \in X_{d_S}^2} K^2\left(x_S, x_S^{\prime}\right)\Big) \\
    &=\frac{n_S}{2\lambda_S}\text{MMD}^{2}(X_{d_S}, X_{d_T}) +\lambda_T-\frac{n_S}{2\lambda_S}\mathbb{E}_{x_T \neq x_T^\prime \in X_{d_T}^2} K^2\left(x_T, x_T^{\prime}\right)\\
    &= \frac{n_S}{2\lambda_S}\text{MMD}^{2}(X_{d_S}, X_{d_T}) + \lambda_T - \frac{n_S}{2\lambda_S}\beta_T.
\end{align*}

We recover the same expression with a $O(\epsilon)$ in the general setting where $\epsilon\neq 0$.

\paragraph{Case when $\epsilon\neq 0$}
We denote
$I:\left\{\begin{array}{cl}\mathrm{GL}_{n_S}(\mathbb{R}) & \rightarrow \mathrm{GL}_{n_S}(\mathbb{R}) \\ A & \mapsto A^{-1}\end{array}\right.$
the inversion function defined on $\mathrm{GL}_{n_S}(\mathbb{R})$, the set of invertible matrices of $\mathcal{M}_{n_S}(\mathbb{R})$.

The function $I$ is differentiable \cite{magnus2019matrix} in all $A\in \mathrm{GL}_{n_S}(\mathbb{R})$ with its differentiate given by the linear application
$dI_A:\left\{\begin{array}{cl}\mathcal{M}_{n_S}(\mathbb{R})&\rightarrow \mathcal{M}_{n_S}(\mathbb{R}) \\ H &\mapsto - A^{-1}HA^{-1}\end{array}\right.$. Therefore, we can perform a Taylor expansion of $I$ at the first order at $A$:
\begin{align*}
    I(A+H)&=I(A)+dI_A(H)+o(\|H\|), \\
    (A+H)^{-1}&=A^{-1}-A^{-1}HA^{-1}+o(\|H\|).
\end{align*}
where $\|H\|\leq n_S\epsilon=O(\epsilon)$. Thus,
\begin{align*}
    (\lambda_S\mathbb{I}_{n_S}+H)^{-1}&=(\lambda_S\mathbb{I}_{n_S})^{-1}-(\lambda_S\mathbb{I}_{n_S})^{-1}H(\lambda_S\mathbb{I}_{n_S})^{-1}+O(\epsilon) =\frac{1}{\lambda_S}\mathbb{I}_{n_S}-\frac{1}{\lambda_S^2}H+O(\epsilon), \\
    \forall i \in \llbracket1, n_S\rrbracket, b_{ii}&=\frac{1}{\lambda_S}-\frac{1}{\lambda_S^2}h_{i,i}+o(\epsilon)=\frac{1}{\lambda_S}+O(\epsilon), \\
    \forall i \neq j \in \llbracket1, n_S\rrbracket, b_{ij}&=-\frac{1}{\lambda_S^2}h_{i,j}+o(\epsilon)=O(\epsilon).
\end{align*}

Therefore, when $\epsilon$ is small, \Cref{eq:inversion} can be developed into:
\begin{align*}
    \varb(x_T)&=K(x_T, x_T)-\sum_{x_S\in X_{d_S}} (\frac{1}{\lambda_S}+O(\epsilon)) K(x_T, x_S)^2 + O(\epsilon) \\
    &= K(x_T,x_T)- \frac{n_S}{\lambda_S} \mathbb{E}_{x_S\in X_{d_S}}[K(x_T, x_S)^2] + O(\epsilon)
\end{align*}
Following the derivation for the case $\epsilon=0$, and remarking that under \Cref{ass:inter_sample} we have $\mathbb{E}_{x_S \neq x_S^\prime \in X_{d_S}^2} K^2\left(x_S, x_S^{\prime}\right)=O(\epsilon^2)$, yields:
\begin{equation*}
    \mathbb{E}_{x_T\in X_{d_T}}[\varb(x_T)] = \frac{n_S}{2\lambda_S}\text{MMD}^{2}(X_{d_S}, X_{d_T}) + \lambda_T - \frac{n_S}{2\lambda_S}\beta_T+ O(\epsilon).
\end{equation*}


\end{proof}

\subsubsection{Variance and initialization}
\label{remark:mmdk}

The MMD depends on the kernel $K$, \ie only on the initialization of $f$ in the kernel regime per \cite{Jacot2018}. Thus, to reduce variance, we could act on the initialization to match $p_S(X)$ and $p_T(X)$ in the RKHS of $K^2$.
This is consistent with \Cref{subsec:expression_ood_bias} that motivated matching the train and test in features.
In our paper, we used the standard pretraining from ImageNet \cite{krizhevsky2012imagenet}, as commonly done on DomainBed \cite{gulrajani2021in}.
The Linear Probing \cite{kumar2022finetuning} initialization of the classifier was shown in \cite{kumar2022finetuning} to prevent the distortion of the features along the training.
This could be improved by pretraining the encoder on a task with fewer domain-specific information, \eg CLIP \cite{radford2021learning} image-to-text translation as in \cite{ruan2022optimal}.
\subsection{WA \versus its members}
\label{app:proof_wa_ind}
We validate that WA's expected error is smaller than its members' error under the locality constraint.
\begin{lemma}[WA \versus its members.]\label{lemma:wa_ind}
    \begin{equation}
        \mathbb{E}_{L_S^M}\mathcal{E}_T(\theta_{\text{WA}}(L_S^M)) - \mathbb{E}_{l_S} \mathcal{E}_T \left(\theta \left( l_S \right) \right)= \frac{M-1}{M} \mathbb{E}_{x\sim p_T}[\covb(x)-\varb(x)] + O(\Deltab^2) \leq O(\bar{\Delta}^2).%
    \end{equation}%
\end{lemma}%
\begin{proof}
    The proof builds upon \Cref{eq:b_var_cov}:
    \begin{equation*}
        \mathbb{E}_{L_S^M}\mathcal{E}_T(\theta_{\text{WA}}) = \mathbb{E}_{(x,y)\sim p_T}\Big[\biasb(x,y)^{2}+\frac{1}{M} \varb(x)+\frac{M-1}{M} \covb(x)\Big] + O(\Deltab^2),
    \end{equation*}
    and the expression of the standard bias-variance decomposition in \Cref{eq:b_v} from \cite{kohavi1996bias},
    \begin{equation*}
        \mathbb{E}_{l_S}\mathcal{E}_T(\theta) = \mathbb{E}_{(x,y)\sim p_T}\Big[\biasb(x,y)^{2}+\varb(x)\Big].
    \end{equation*}
    The difference between the two provides:
    \begin{equation*}
        \mathbb{E}_{L_S^M}\mathcal{E}_T(\theta_{\text{WA}}) - \mathbb{E}_{l_S}\mathcal{E}_T(\theta) = \frac{M-1}{M} \mathbb{E}_{(x,y)\sim p_T} \Big[\covb(x)-\varb(x)\Big] + O(\Deltab^2).
    \end{equation*}
    Cauchy Schwartz inequality states $|\mathrm{cov}(Y, Y')| \leq \sqrt{\mathrm{var}(Y) \mathrm{var}(Y')}$, thus $\covb(x)\leq\varb(x)$. Then:
    \begin{equation*}
        \mathbb{E}_{L_S^M}\mathcal{E}_T(\theta_{\text{WA}}) - \mathbb{E}_{l_S}\mathcal{E}_T(\theta) \leq O(\Deltab^2).
    \end{equation*}
\end{proof}

\clearpage
\section{Weight averaging versus functional ensembling}%
\label{app:wa_vs_ens_analysis}%
\begin{wrapfigure}[12]{hR!}{0.34\textwidth}
    \vspace{-3em}%
    \centering%
    \includegraphics[width=0.34\textwidth]{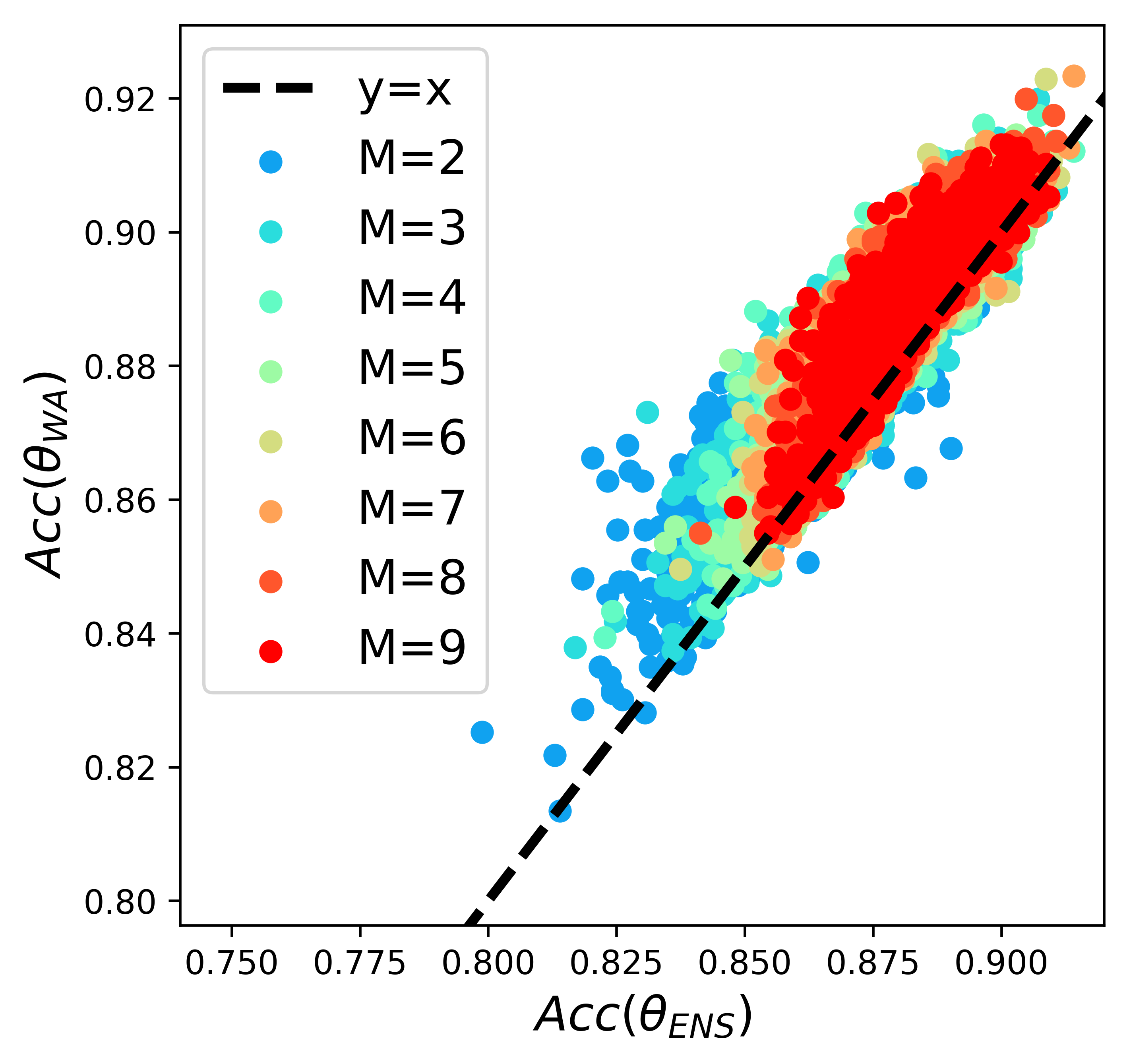}%
    \caption{$f_{\text{WA}}$ performs similarly or better than $f_{\text{ENS}}$ on domain \enquote{Art} on PACS.}%
    \label{fig:pacs_samediffruns_ens210_soup}%
\end{wrapfigure}%
We further compare the following two methods to combine $M$ weights $\{\theta(l_S^{(m)})\}_{m=1}^M$: $f_{\text{WA}}$ that averages the weights and $f_{\text{ENS}}$ \cite{Lakshminarayanan2017} that averages the predictions.
We showed in \Cref{lemma:wa_ensembling} that $f_{\text{WA}} \approx f_{\text{ENS}}$ when $\max_{m=1}^{M}\| \theta(l_S^{(m)}) - \theta_{\text{WA}}\|_2$ is small.

In particular, when $\{l_S^{(m)}\}_{m=1}^M$ share the same initialization and the hyperparameters are sampled from mild ranges, we empirically validate our approximation on OfficeHome in \Cref{fig:home0_samediffruns_net_soup}. This is confirmed on PACS dataset in \Cref{fig:pacs_samediffruns_ens210_soup}.
For both datasets, we even observe that $f_{\text{WA}}$ performs slightly but consistently better than $f_{\text{ENS}}$. The observed improvement is non-trivial; we refer to Equation 1 in \cite{Wortsman2022ModelSA} for some initial explanations based on the value of \ood Hessian and the confidence of $f_{\text{WA}}$. The complete analysis of this second-order difference is left for future work.

Yet, we do not claim that $f_{\text{WA}}$ is systematically better than $f_{\text{ENS}}$.
In \Cref{table:db_home0_ensvsswa}, we show that this is no longer the case when we relax our two constraints, consistently with \Cref{fig:home0_locality_requirement}.
\textit{First}, when the classifiers' initializations vary, ENS improves thanks to this additional diversity; in contrast, DiWA degrades because weights are no longer averageable.
\textit{Second}, when the hyperparameters are sampled from extreme ranges (defined in \Cref{tab:hyperparam}), performance drops significantly for DiWA, but much less for ENS.
As a side note, the downward trend in this second setup (even for ENS) is due to inadequate hyperparameters that degrade the expected individual performances.

This highlights a limitation of DiWA, which requires weights that satisfy the locality requirement or are at least linearly connectable.
In contrast, Deep Ensembles \cite{Lakshminarayanan2017} are computationally expensive (and even impractical for large $M$), but can leverage additional sources of diversity.
An interesting extension of DiWA for future work would be to consider the functional ensembling of several DiWAs trained from different initializations or even with different network architectures \cite{singh2016swapout}.
Thus the Ensemble of Averages (EoA) strategy introduced in  \cite{arpit2021ensemble} is complementary to DiWA and could be extended into an Ensemble of Diverse Averages.

\begin{table}[h]%
    \caption{\textbf{DiWA's vs. ENS's accuracy ($\%, \uparrow$)} on domain \enquote{Art} from OfficeHome when varying initialization and hyperparameter ranges. Best on each setting is in \textbf{bold}.}
    \centering
    \adjustbox{width=1.0\textwidth}{
        \centering
        \begin{tabular}{ccc ccccc}
            \toprule
            \multicolumn{3}{c}{Configuration} & \multicolumn{2}{c}{\textbf{$M=20$}} && \multicolumn{2}{c}{\textbf{$M=60$}} \\
            \cmidrule(lr){0-2} \cmidrule(lr){4-5} \cmidrule(lr){7-8} Shared classifier init && Mild hyperparameter ranges & DiWA & ENS && DiWA & ENS \\
            \midrule
              \cmark          && \cmark    & \textbf{67.3} $\pm$ 0.2  & 66.1 $\pm$ 0.1 & & \textbf{67.7}    &   66.5      \\
            \xmark          && \cmark    & 65.0 $\pm$ 0.5  & \textbf{67.5} $\pm$ 0.3 & & 65.9     &  \textbf{68.5}      \\
            \cmark          && \xmark    & 56.6 $\pm$ 0.9  & \textbf{64.3} $\pm$ 0.4 & & 59.5     &  \textbf{64.7}      \\   

            \bottomrule
        \end{tabular}}
    \label{table:db_home0_ensvsswa}%
\end{table}

\section{Additional diversity analysis}%

\label{app:additional_diversity_analysis}
\subsection{On OfficeHome}

\subsubsection{Feature diversity}
\label{app:analysis_office_features}
In \Cref{sect:analysis}, our diversity-based theoretical findings were empirically validated using the ratio-error \cite{aksela2003comparison}, a common diversity measure notably used in \cite{rame2021dice,rame2021ixmo}.
In \Cref{fig:home0_df}, we recover similar conclusions with another diversity measure: the Centered Kernel Alignment Complement (CKAC) \cite{DBLP:journals/corr/abs-1905-00414}, also used in \cite{Neyshabur2020,Wortsman2022robust}.
CKAC operates in the feature space and measures to what extent the pairwise similarity matrices (computed on domain $T$) are aligned --- where similarity is the dot product between penultimate representations extracted from two different networks.
\begin{figure}[h]%
    \centering%
    \begin{subfigure}{.33\textwidth}%
        \centering%
        \includegraphics[width=\linewidth]{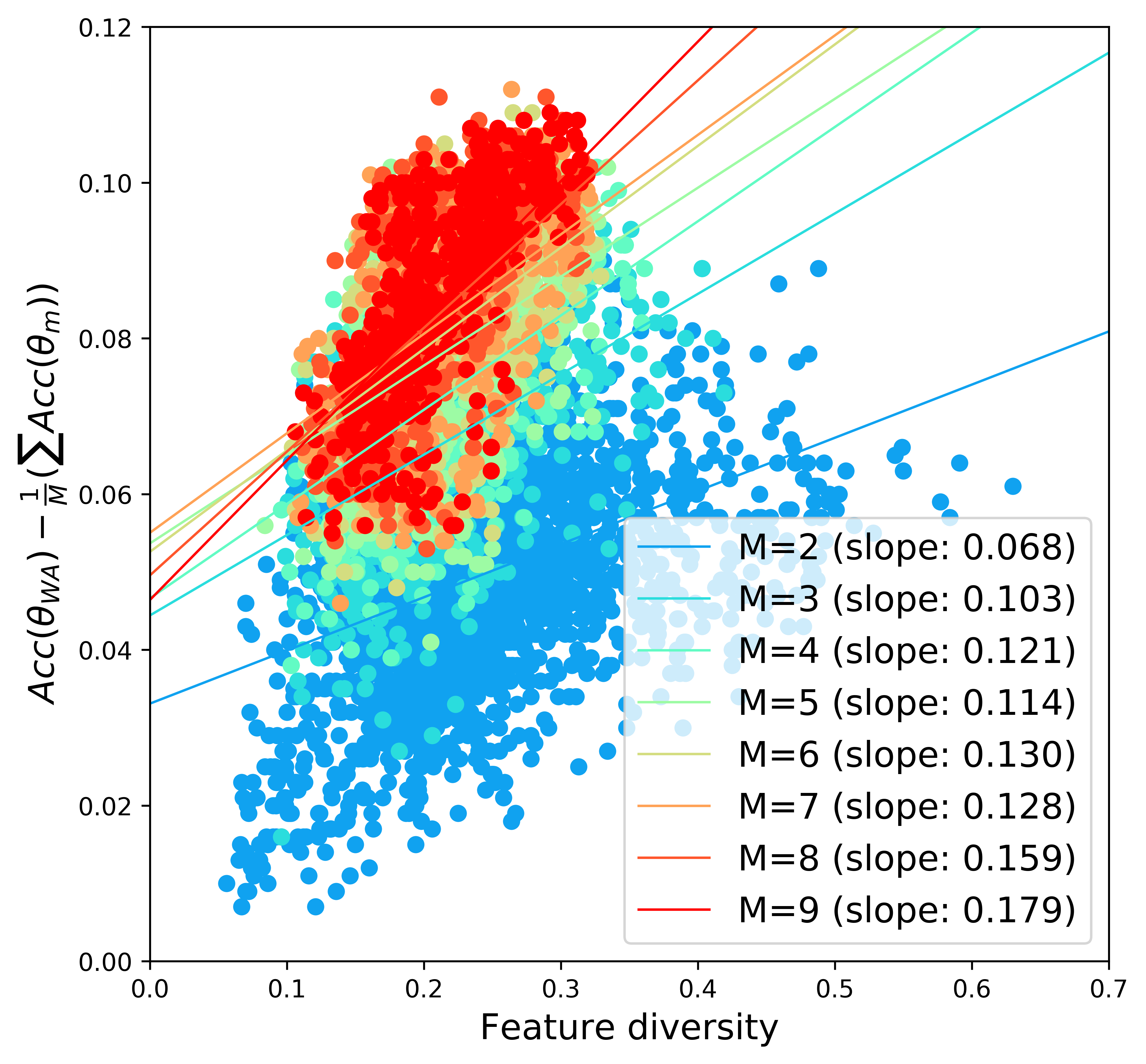}%
        \caption{Same as \Cref{fig:home0_dr_soup-netm}.}%
        \label{fig:home0_df_soup-netm}%
    \end{subfigure}%
    \begin{subfigure}{.33\textwidth}%
        \centering%
        \includegraphics[width=\textwidth]{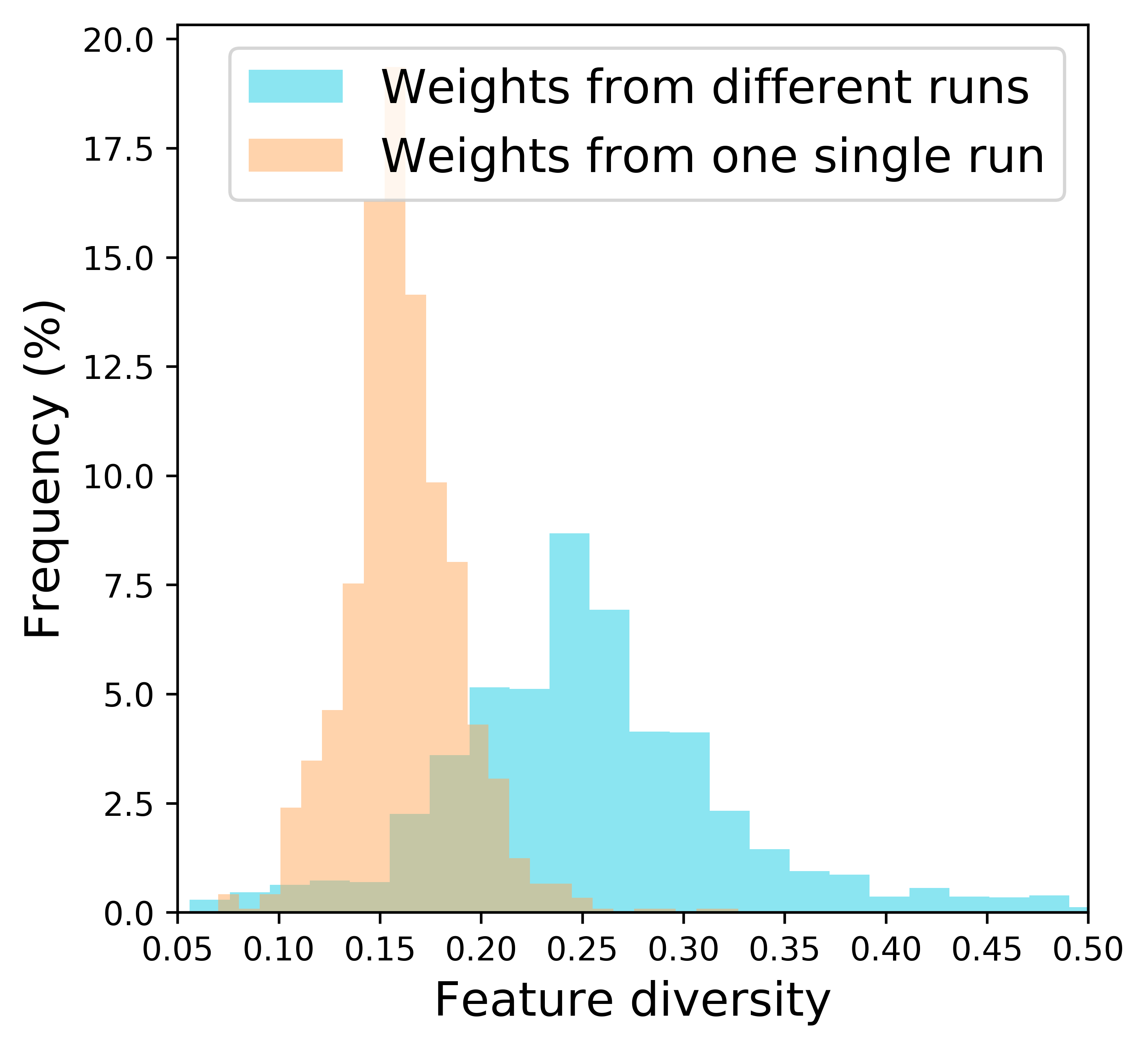}%
        \caption{Same as \Cref{fig:home0_dr_frequency}.}%
        \label{fig:home0_df_frequency}%
    \end{subfigure}%
    \begin{subfigure}{.33\textwidth}%
        \centering%
        \includegraphics[width=\linewidth]{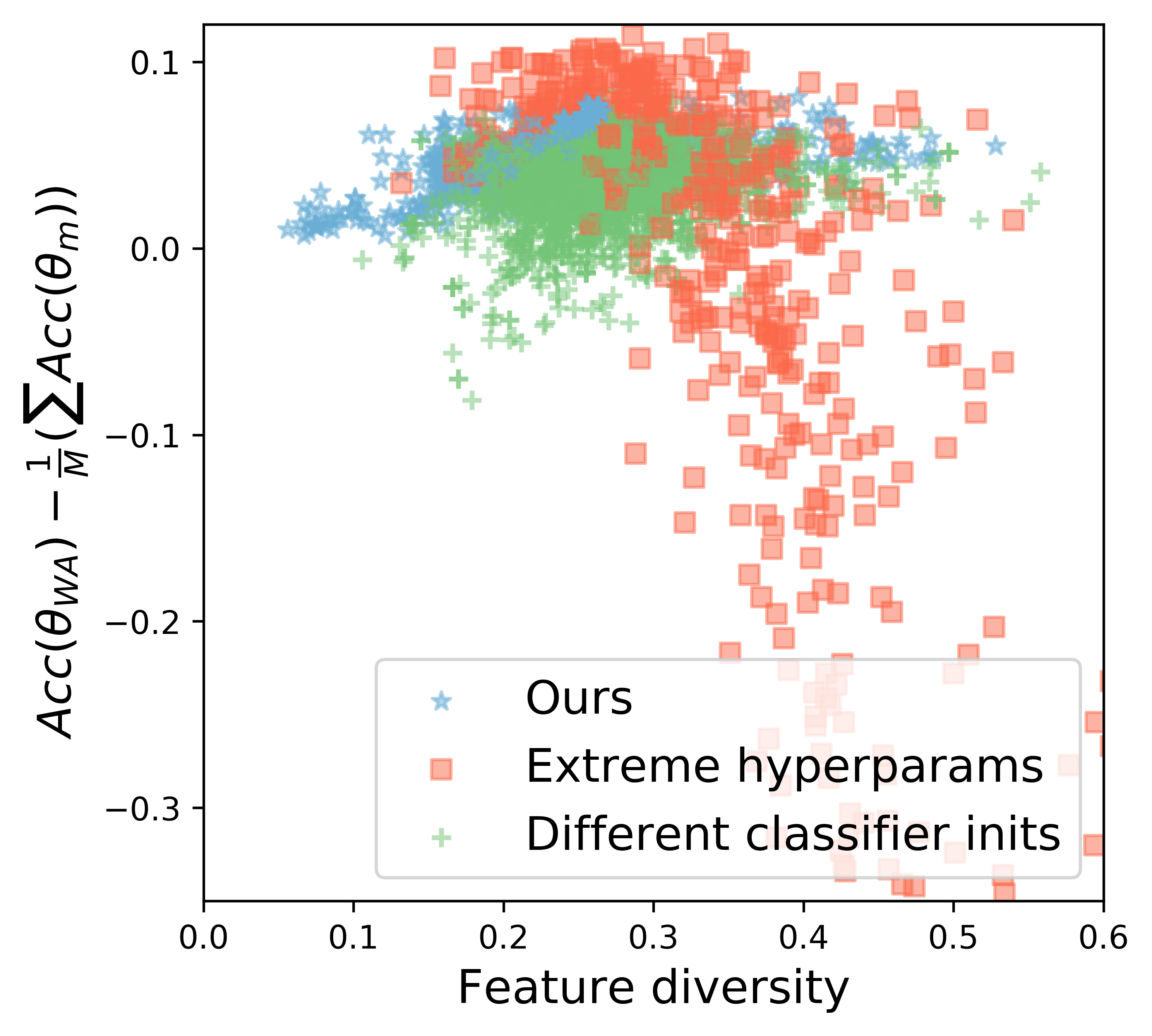}%
        \caption{Same as \Cref{fig:home0_locality_requirement}.}%
        \label{fig:home0_locality_requirement_df}%
    \end{subfigure}%
    \caption{Same analysis as \Cref{sect:analysis}, where diversity is measured with CKAC \cite{DBLP:journals/corr/abs-1905-00414} in features rather than with ratio-error \cite{aksela2003comparison} in predictions.}%
    \label{fig:home0_df}%
\end{figure}
\subsubsection{Accuracy gain per unit of diversity}
\label{app:soup:m_slope}
In \Cref{fig:home0_dr_soup-netm,fig:home0_df_soup-netm}, we indicated the slope of the linear regressions relating diversity to accuracy gain at fixed $M$ (between $2$ and $9$).
For example, when $M=9$ weights are averaged, the accuracy gain increases by $0.297$ per unit of additional diversity in prediction  \cite{aksela2003comparison} (see \Cref{fig:home0_dr_soup-netm}) and by $0.179$ per unit of additional diversity in features \cite{DBLP:journals/corr/abs-1905-00414} (see \Cref{fig:home0_df_soup-netm}).
Most importantly, we note that the slope increases with $M$.
To make this more visible, we plot slopes \wrt $M$ in \Cref{fig:soup:m_slope}.
Our observations are consistent with the $(M-1)/M$ factor in front of $\covb\left(x\right)$ in \Cref{eq:b_var_cov}.
This shows that diversity becomes more important for large $M$.
Yet, large $M$ is computationally impractical in standard functional ensembling, as one forward step is required per model.
In contrast, WA has a fixed inference time which allows it to consider larger $M$. Increasing $M$ from $20$ to $60$ is the main reason why DiWA$^{\dagger}$ improves DiWA.
\begin{figure}[H]
    \centering%
    \includegraphics[width=0.5\textwidth]{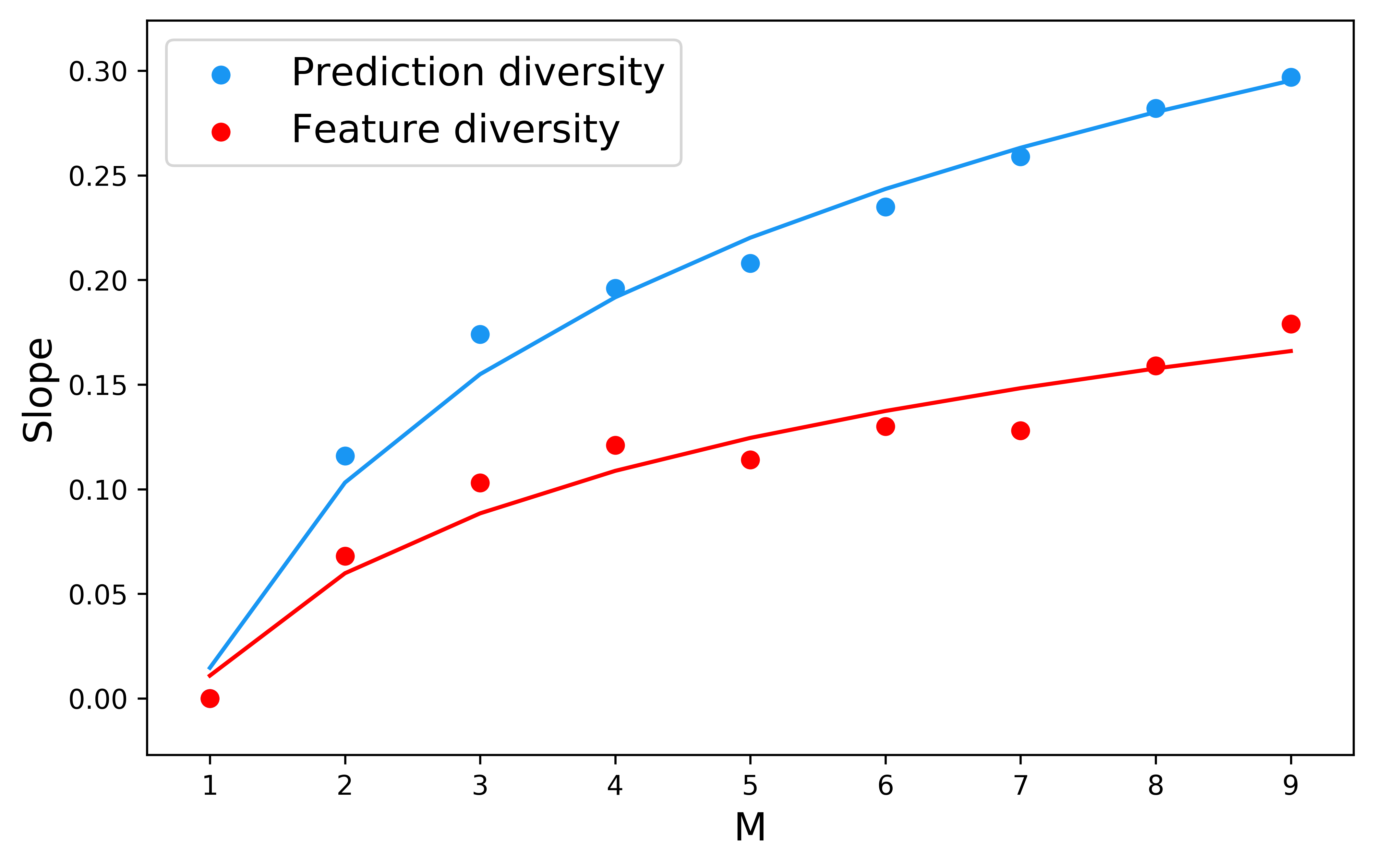}%
    \caption{The slopes of linear regression --- relating diversity to accuracy gain in \Cref{fig:home0_dr_soup-netm} and \Cref{fig:home0_df_soup-netm} --- increases with $M$.}
    \label{fig:soup:m_slope}%
\end{figure}%

\subsubsection{Diversity comparison across a wide range of methods}
\label{app:more_diversity_comparison}
Inspired by \cite{gontijolopes2022no}, we further analyze in \Cref{fig:home0_boxplot} the diversity between two weights obtained from different (more or less correlated) learning procedures.
\begin{itemize}
    \item In the upper part, weights are obtained from a single run. They share the same initialization/hyperparameters/data/noise in the optimization procedure and only differ by the number of training steps (which we choose to be a multiple of $50$). They are less diverse than the weights in the middle part of \Cref{fig:home0_boxplot}, that are sampled from two ERM runs.
    \item  When sampled from different runs, the weights become even more diverse when they have more extreme hyperparameter ranges, they do not share the same classifier initialization or they are trained on different data. The first two are impractical for WA, as it breaks the locality requirement (see \Cref{fig:home0_locality_requirement,fig:home0_locality_requirement_df}).
          Luckily, the third setting \enquote{data diversity} is more convenient and is another reason for the success of DiWA$^{\dagger}$; its $60$ weights were trained on $3$ different data splits. Data diversity has provable benefits \cite{efron1992bootstrap}, \eg in bagging \cite{breiman1996bagging}.
    \item Finally, we observe that diversity is increased (notably in features) when two runs have different objectives, for example, Interdomain Mixup \cite{yan2020improve} and Coral \cite{coral216aaai}. Thus incorporating weights trained with different invariance-based objectives have two benefits that explain the strong results in \Cref{table:db_home0_training}: (1) they learn invariant features by leveraging the domain information and (2) they enrich the diversity of solutions by extracting different features. These solutions can bring their own particularity to WA.
\end{itemize}
In conclusion, our analysis confirms that \enquote{model pairs that diverge more in training methodology display categorically different generalization behavior, producing increasingly uncorrelated errors}, as stated in \cite{gontijolopes2022no}.
\begin{figure}[h]%
    \centering%
    \begin{subfigure}{.5\textwidth}%
        \centering%
        \includegraphics[width=1.0\textwidth]{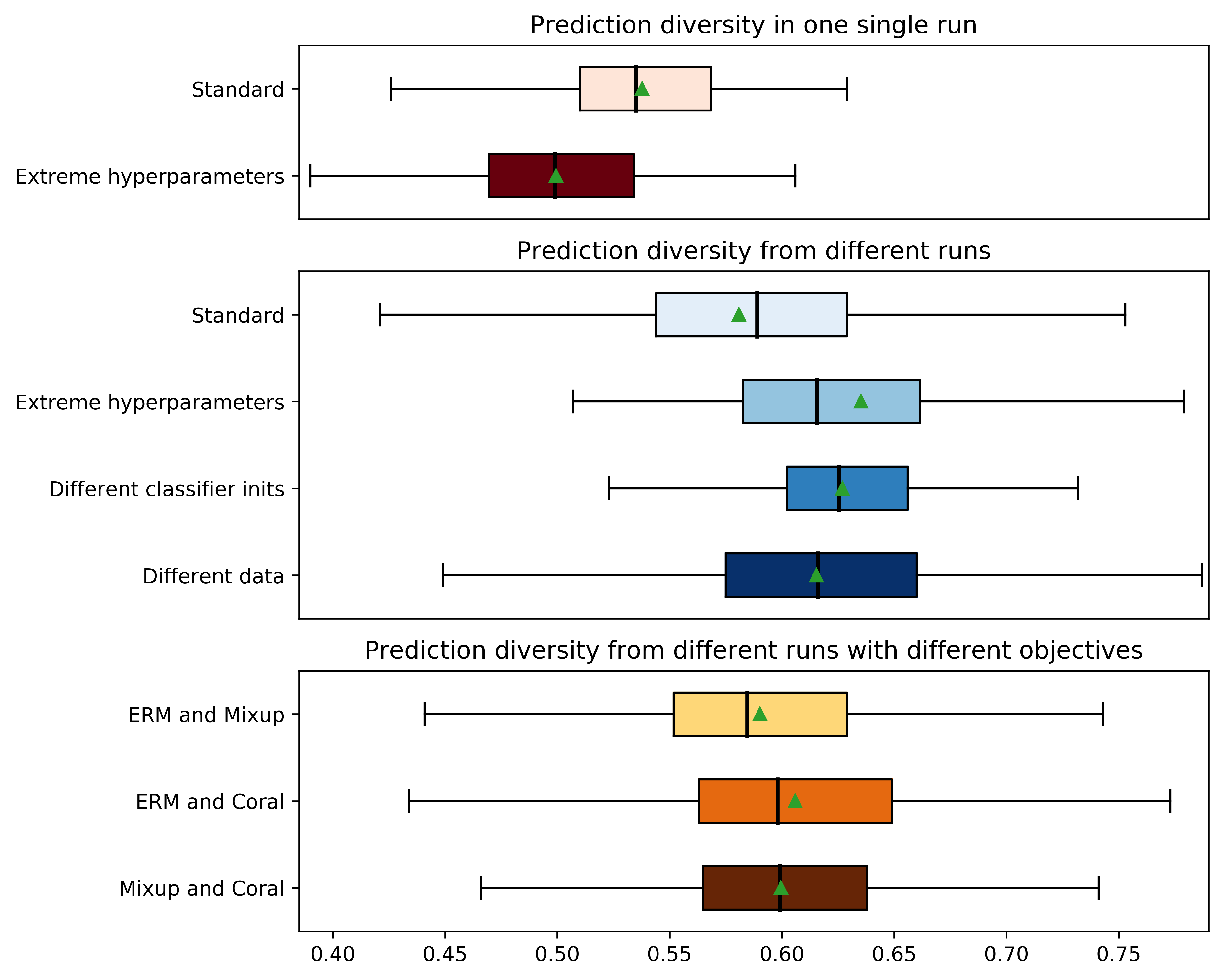}%
        \caption{Prediction diversity \cite{aksela2003comparison}.}
        \label{fig:home0_dr_boxplot}%
    \end{subfigure}%
    \begin{subfigure}{.5\textwidth}%
        \centering%
        \includegraphics[width=1.0\textwidth]{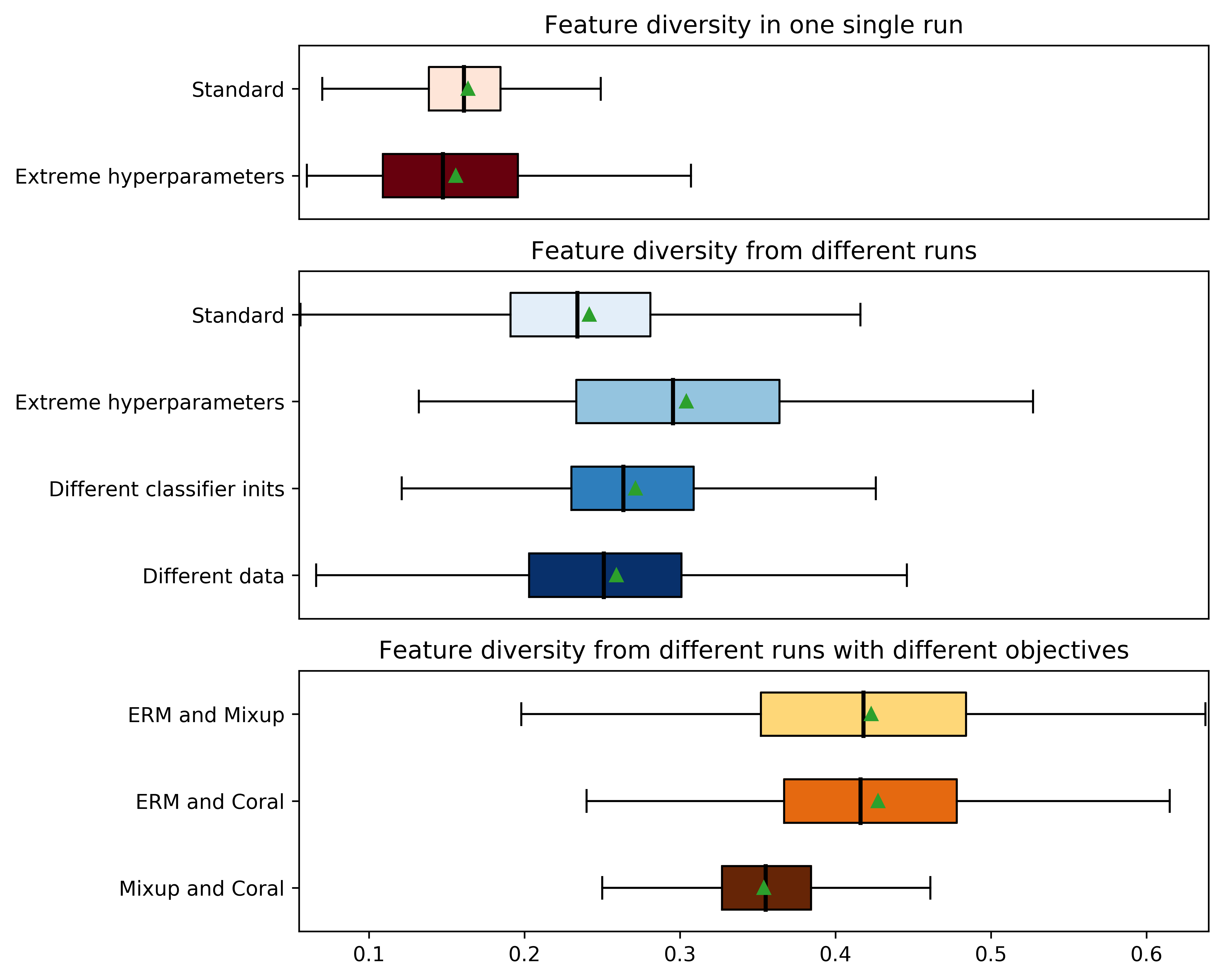}%
        \caption{Feature diversity \cite{DBLP:journals/corr/abs-1905-00414}.}
        \label{fig:home0_df_boxplot}%
    \end{subfigure}%
    \caption{\textbf{Diversity analysis} across weights, which are per default trained with ERM, with a mild hyperparameter range (see \Cref{tab:hyperparam}), with a shared random classifier initialization, on a given data split.
    \textit{First}, it confirms \Cref{fig:home0_dr_frequency,fig:home0_df_frequency}: weights obtained from two different runs are more different than those sampled from a single run (even with extreme hyperparameters).
    \textit{Second}, this shows that weights from two runs are more diverse when the two runs have different hyperparameters/data/classifier initializations/training objectives.
    Domain \enquote{Art} on OfficeHome.}
    \label{fig:home0_boxplot}%
\end{figure}%

\subsubsection{\reb{Trade-off between diversity and averageability}}

\reb{We argue in \Cref{subsec:expression_loc_lir} that our weights should ideally be diverse functionally while being averageable (despite the nonlinearities in the network).
    We know from \cite{Neyshabur2020} that models fine-tuned from a shared initialization with shared hyperparameters can be connected along a linear path where error remains low; thus, they are averageable as their WA also has a low loss.
    In \Cref{fig:home0_locality_requirement}, we confirmed that averaging models from different initializations performs poorly.
    Regarding the hyperparameters, \Cref{fig:home0_locality_requirement} shows that hyperparameters can be selected slightly different but not too distant.
    That is why we chose mild hyperparameter ranges (defined in \Cref{tab:hyperparam}) in our main experiments.}
    
\reb{A complete analysis of when the averageability holds when varying the different hyperparameters is a promising lead for future work. 
    Still, \Cref{fig:home0_lrs} is a preliminary investigation of the impact of different learning rates  (between learning procedures of each weight). First, we validate that more distant learning rates lead to more functional diversity in \Cref{fig:home0_lrs_dr}. Yet, we observe in \Cref{fig:home0_lrs_acc} that if learning rates are too different, weight averaging no longer approximates functional ensembling because the $O(\Delta_{L_S^M}^2)$ term in \Cref{lemma:wa_ensembling} can be large.}

\begin{figure}[h]%
    \centering%
    \begin{subfigure}{.5\textwidth}%
        \centering%
        \includegraphics[width=1.0\textwidth]{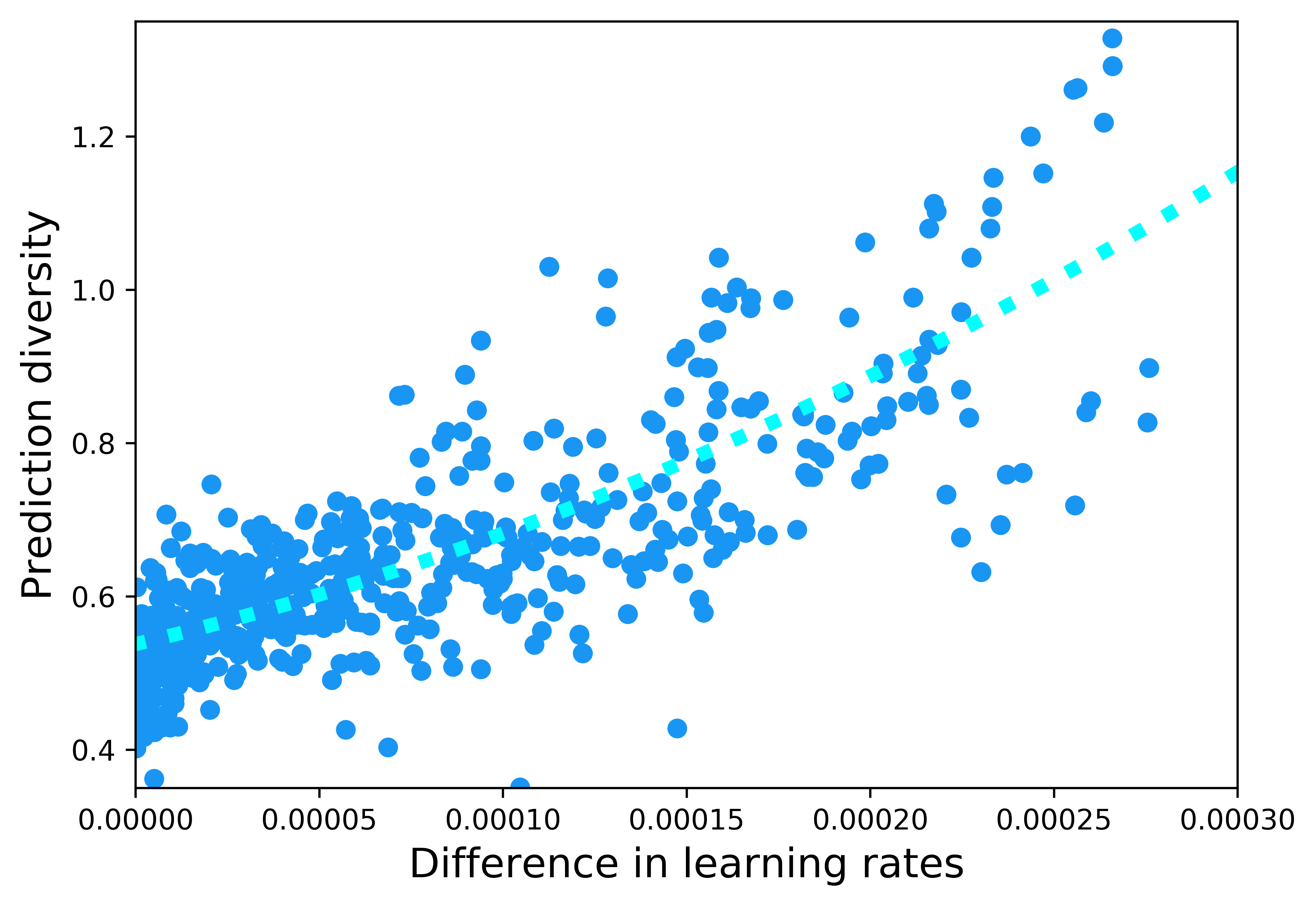}%
        \caption{\reb{Prediction diversity ($\uparrow$) \cite{aksela2003comparison} between models.}}
        \label{fig:home0_lrs_dr}%
    \end{subfigure}%
    \begin{subfigure}{.5\textwidth}%
        \centering%
        \includegraphics[width=1.0\textwidth]{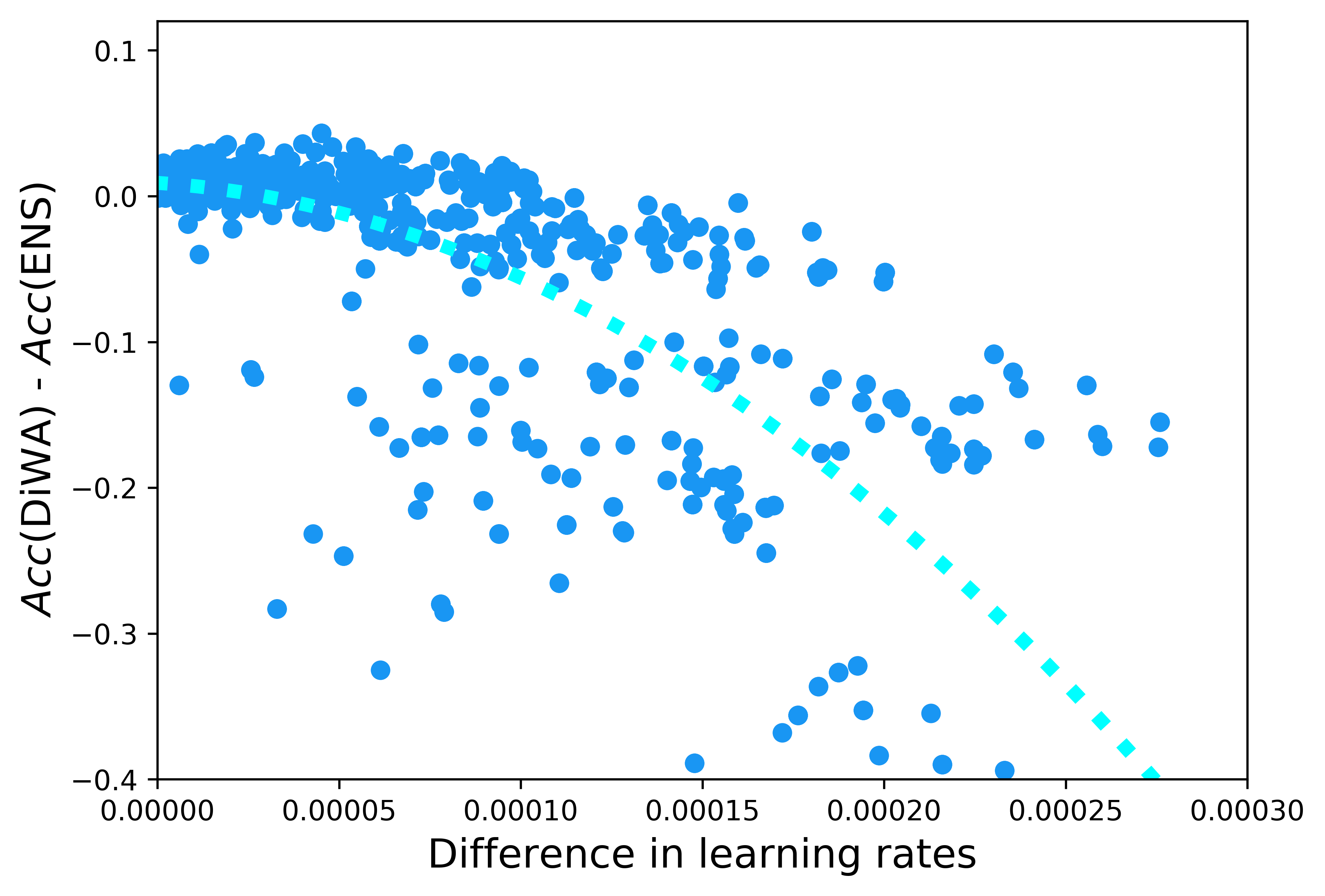}%
        \caption{\reb{Accuracy ($\uparrow$) difference between DiWA and ENS.}}
        \label{fig:home0_lrs_acc}%
    \end{subfigure}%
    \caption{\reb{\textbf{Trade-off between diversity and averageability for various differences in learning rates}. Considering $M=2$ weights obtained from two learning procedures with learning rates $\text{lr}_1$ and $\text{lr}_2$ (sampled from the extreme distribution in \Cref{tab:hyperparam}), we plot in \Cref{fig:home0_lrs_dr} the prediction diversity for these $M=2$ models \versus $|\text{lr}_1-\text{lr}_2|$. Then, in \Cref{fig:home0_lrs_acc}, we plot the accuracy differences $\operatorname{Acc}(\text{DiWA})-\operatorname{Acc}(\text{ENS})$ \versus $|\text{lr}_1-\text{lr}_2|$.}}
    \label{fig:home0_lrs}%
\end{figure}%

\FloatBarrier
\subsection{On PACS}
\label{app:analysis_pacs}
We perform in \Cref{fig:soup:pacs_div_soup-netm} on domain \enquote{Art} from PACS the same core diversity-based experiments than on OfficeHome in \Cref{sect:analysis}.
We recover the same conclusions.
\begin{figure}[h]%
    \centering%
    \begin{subfigure}{.33\textwidth}%
        \centering%
        \includegraphics[width=\linewidth]{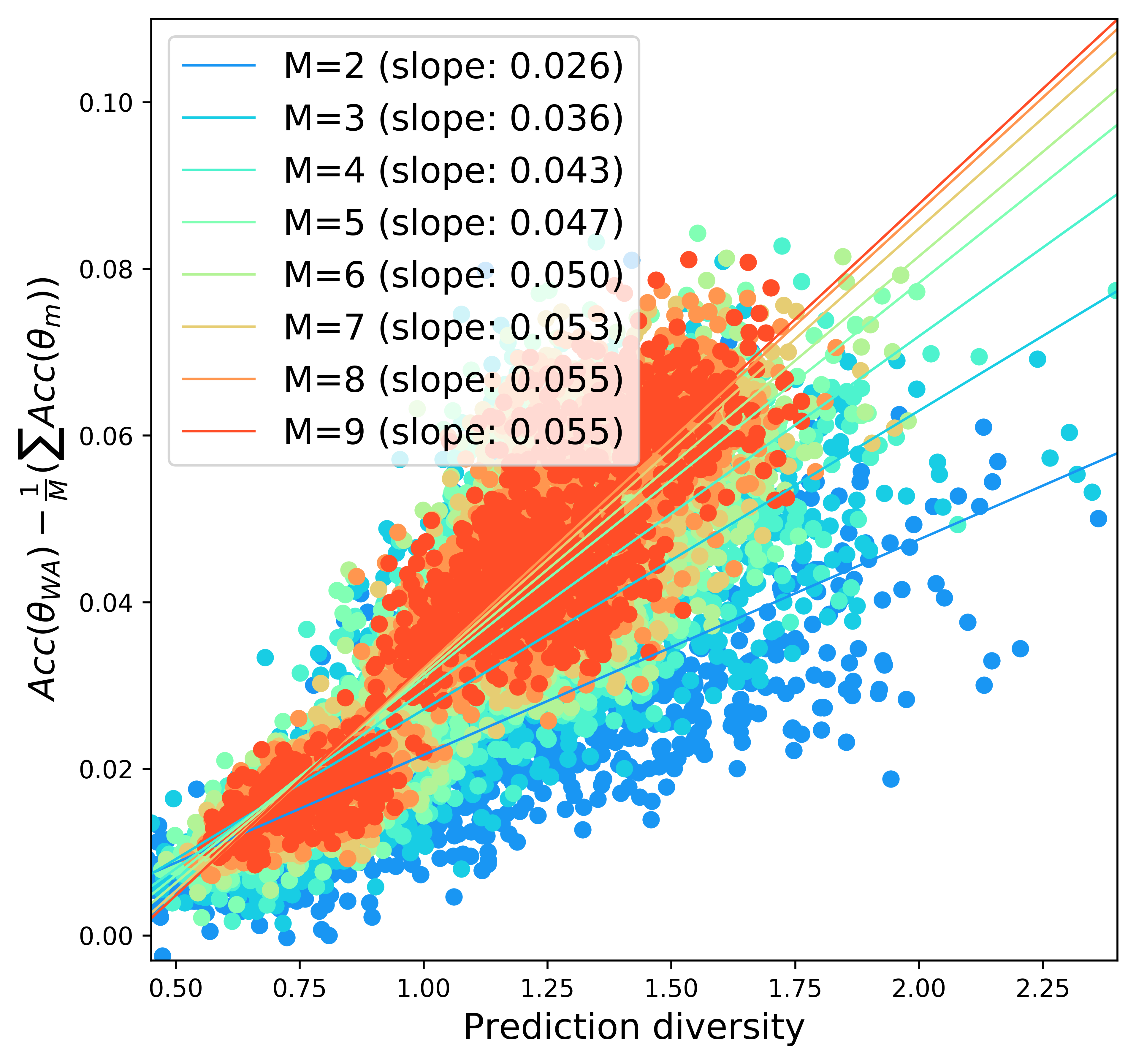}%
        \caption{Same as \Cref{fig:home0_dr_soup-netm}.}%
        \label{fig:pacs_dr_soup-netm}%
    \end{subfigure}%
    \begin{subfigure}{.33\textwidth}%
        \centering%
        \includegraphics[width=\linewidth]{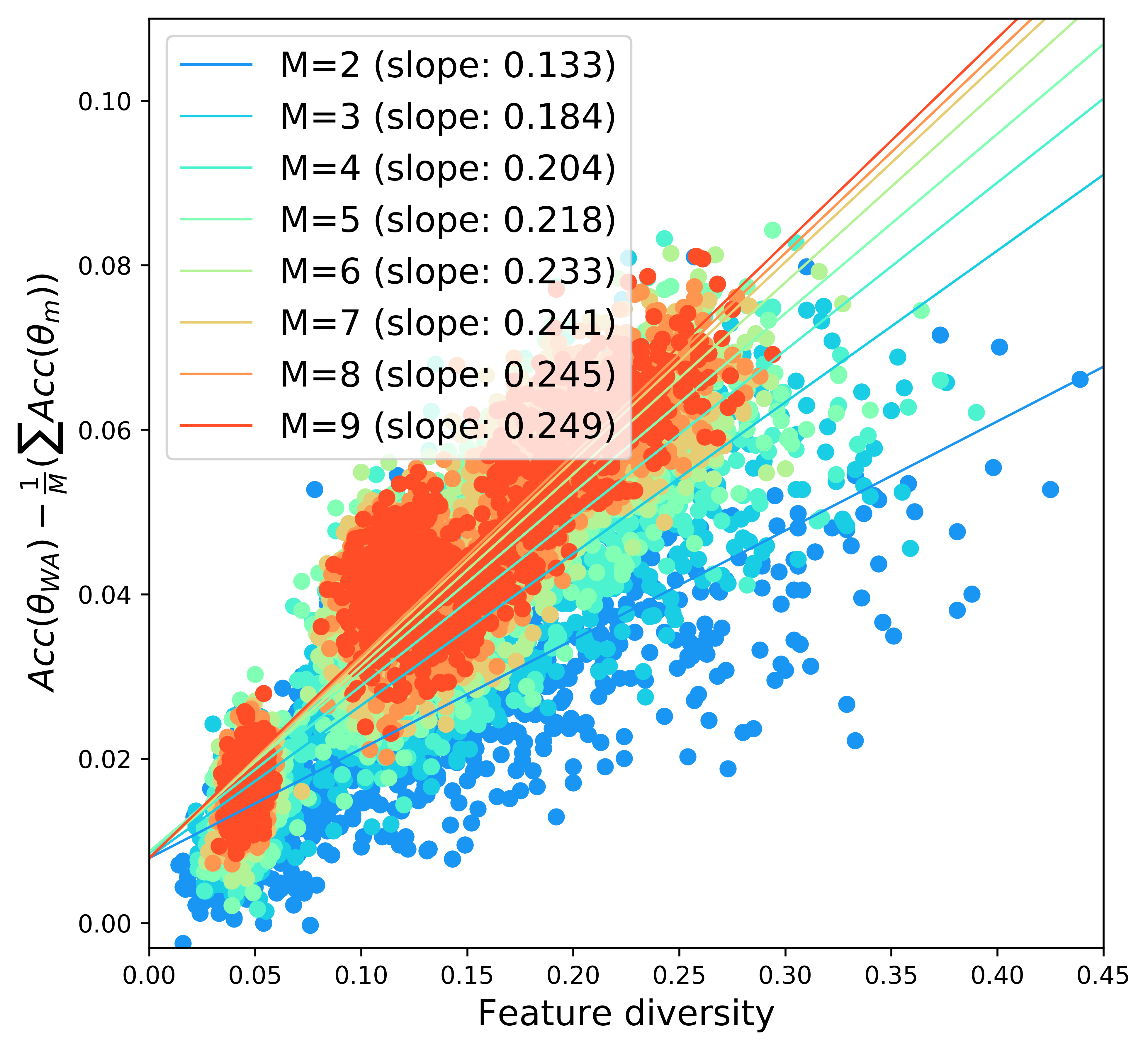}%
        \caption{Same as \Cref{fig:home0_df_soup-netm}.}%
        \label{fig:pacs_df_soup-netm}%
    \end{subfigure}%
    \begin{subfigure}{.33\textwidth}%
        \centering%
        \includegraphics[width=\textwidth]{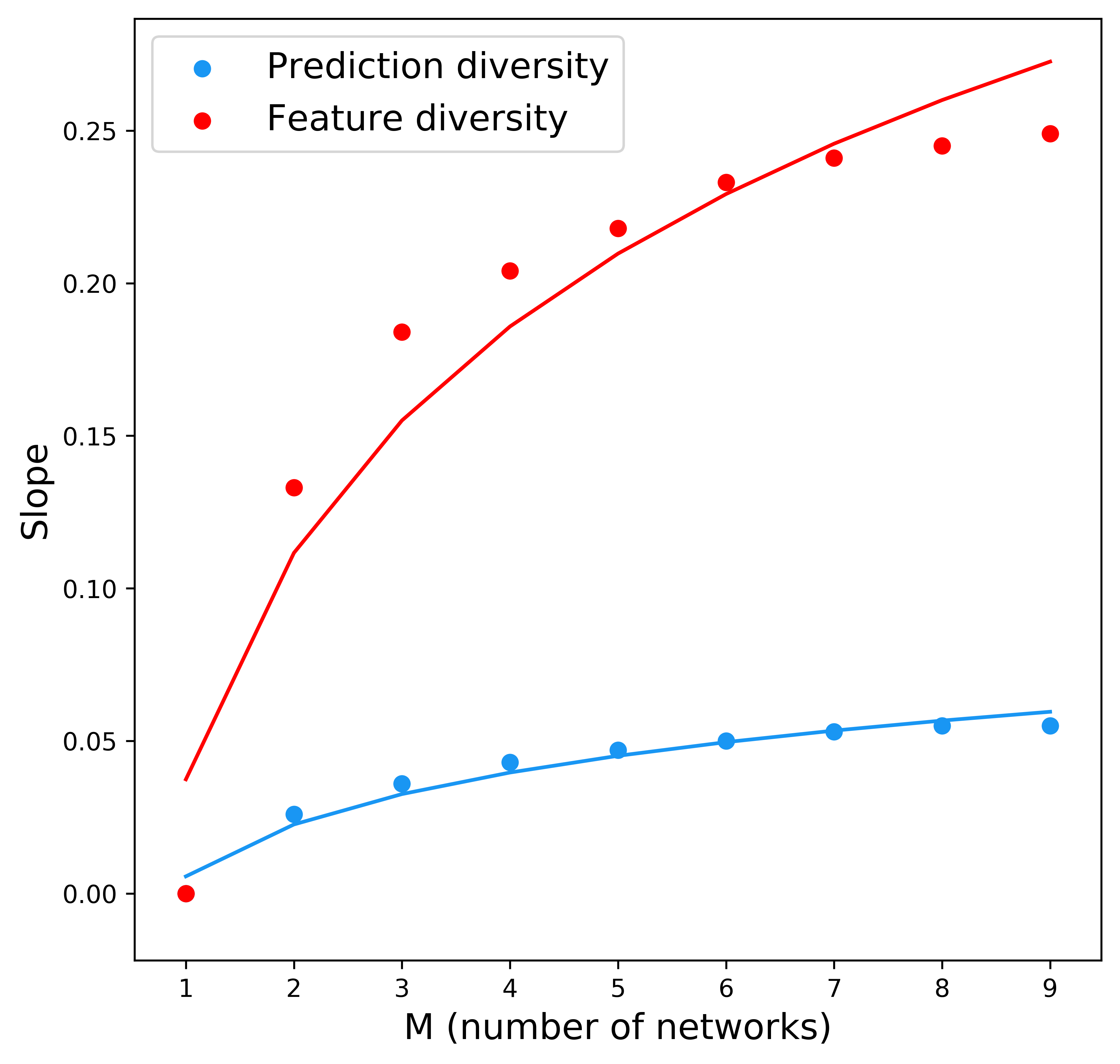}%
        \caption{Same as \Cref{fig:soup:m_slope}.}
        \label{fig:soup:pacs_m_slope}%
    \end{subfigure}%
    \caption{Same analysis on PACS as previously done on OfficeHome.}%
    \label{fig:soup:pacs_div_soup-netm}%
\end{figure}%

\begin{figure}[H]
    \vspace{-1.em}
    \centering
    \begin{subfigure}{.33\textwidth}%
        \centering%
        \includegraphics[width=\textwidth]{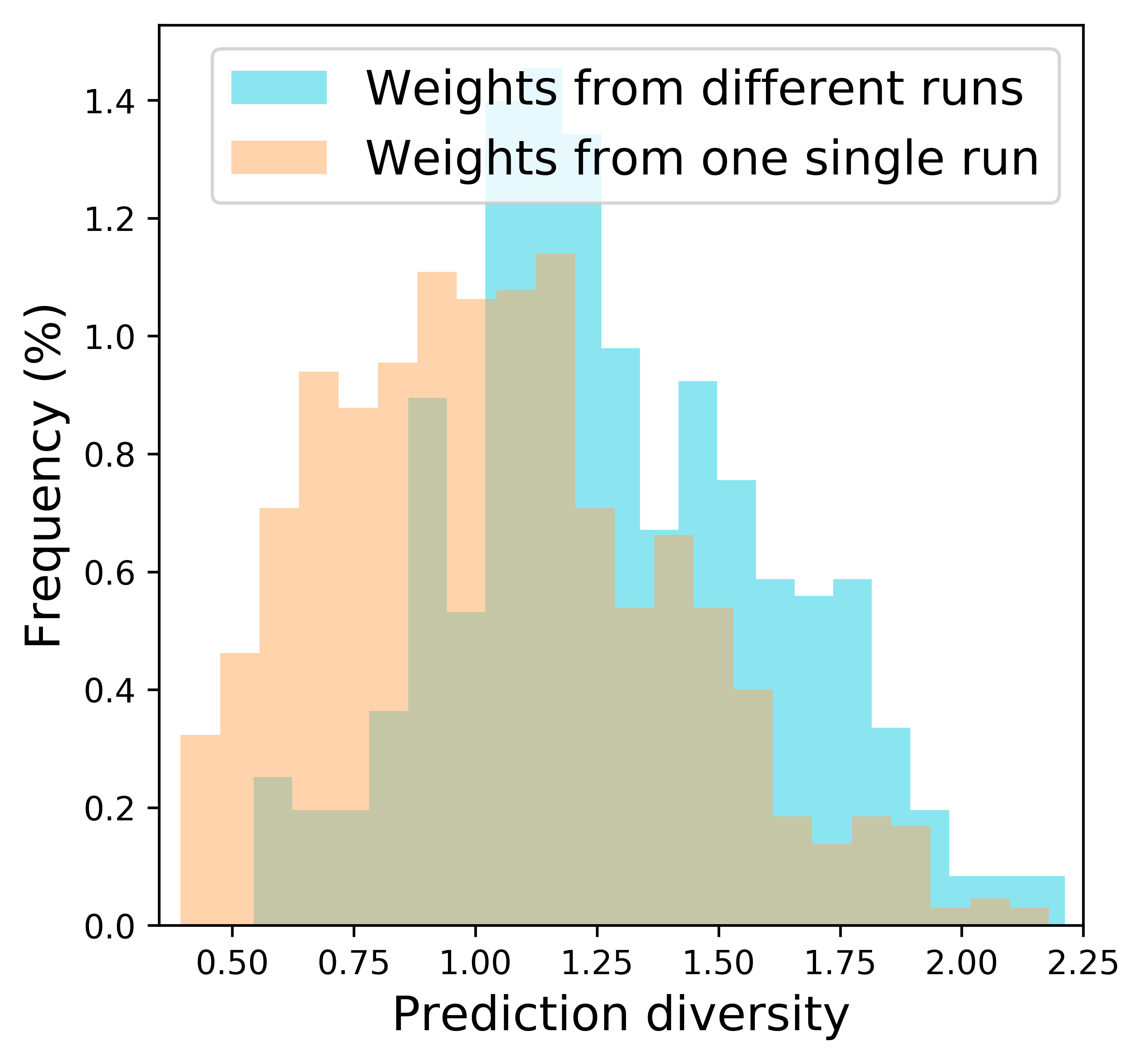}%
        \caption{Same as \Cref{fig:home0_dr_frequency}.}
    \end{subfigure}%
    \begin{subfigure}{.33\textwidth}%
        \centering%
        \includegraphics[width=\textwidth]{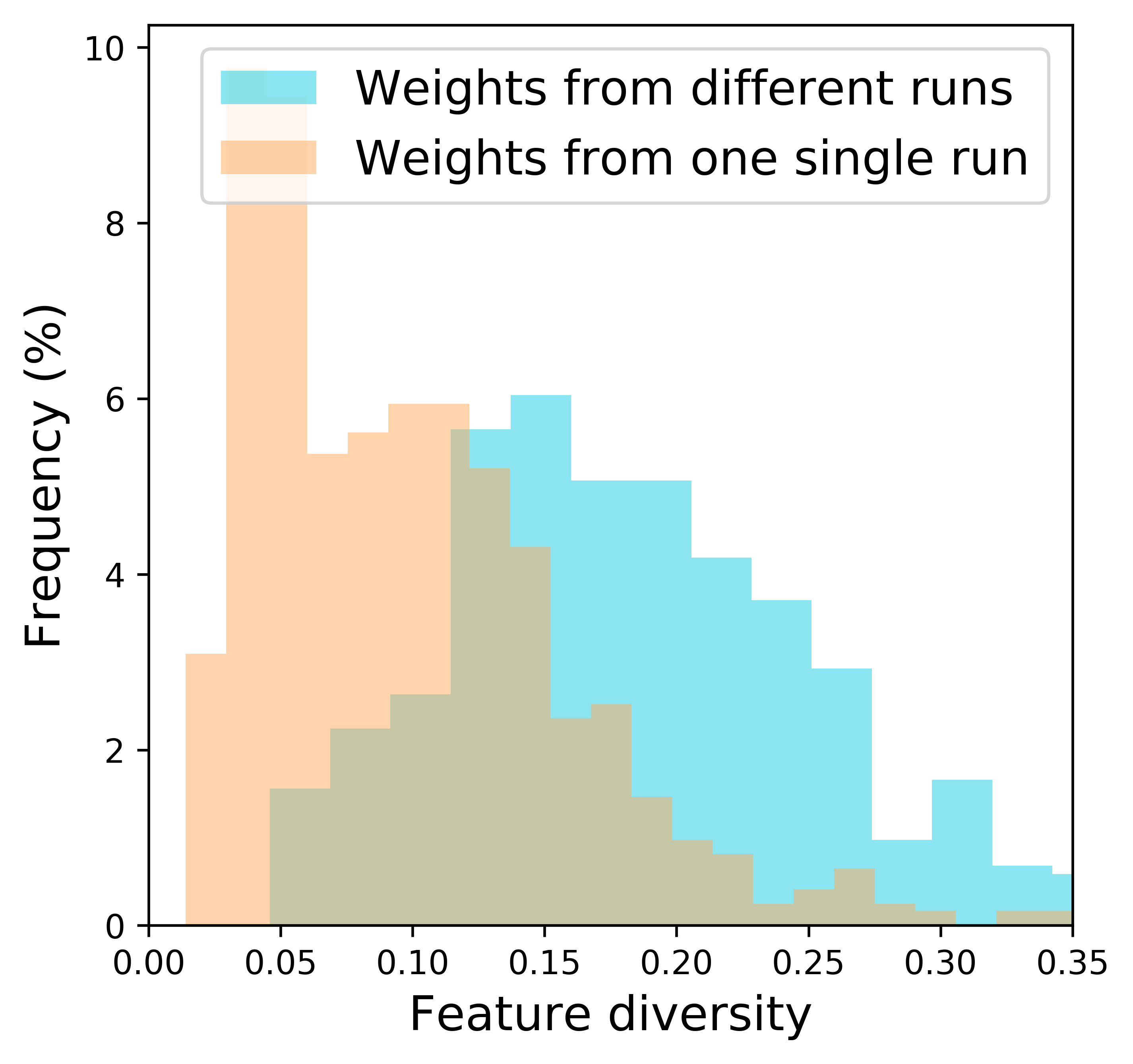}%
        \caption{Same as \Cref{fig:home0_df_frequency}.}
    \end{subfigure}%
    \begin{subfigure}{.33\textwidth}%
        \centering%
        \includegraphics[width=0.9\textwidth]{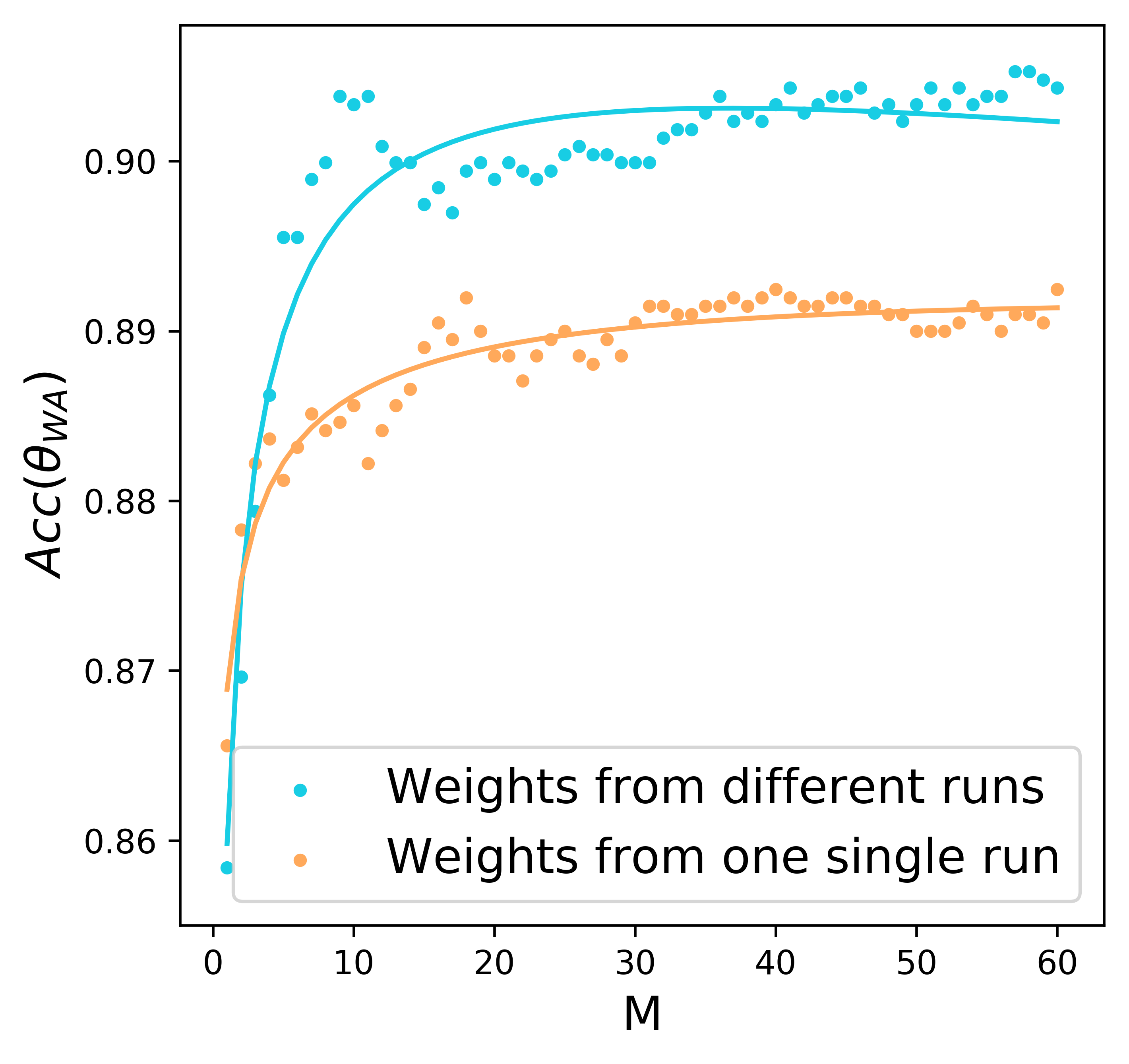}%
        \caption{Same as \Cref{fig:home0_m_vs_acc}.}
    \end{subfigure}
    \label{fig:same_analysis_pacs}%
    \caption{Same analysis on PACS as previously done on OfficeHome.}%
\end{figure}%

\section{\reb{Number of training runs}}

\label{app:valueofm}

\reb{In our experiments, we train $20$ independent training runs per data split. We selected this value as $20$ is the standard number of hyperparameter trials in DomainBed \cite{gulrajani2021in}. In \Cref{fig:home0_valueofm} we ablate this choice on the OOD domain \enquote{Art} of OfficeHome. We observe that a larger number of runs leads to improved performance and reduced standard deviation. These results are consistent with our theoretical analysis, as the variance is divided per $M$ in \Cref{prop:b_var_cov}. If reducing the training time is critical, one could benefit from significant gains over ERM even with a smaller number of runs: for example, $10$ runs seem sufficient in this case. This analysis complements \Cref{fig:home0_m_vs_acc} --- where $60$ runs were launched then sorted in increasing validation accuracy.}

\begin{figure}[H]
    \centering
    \includegraphics[width=0.5\textwidth]{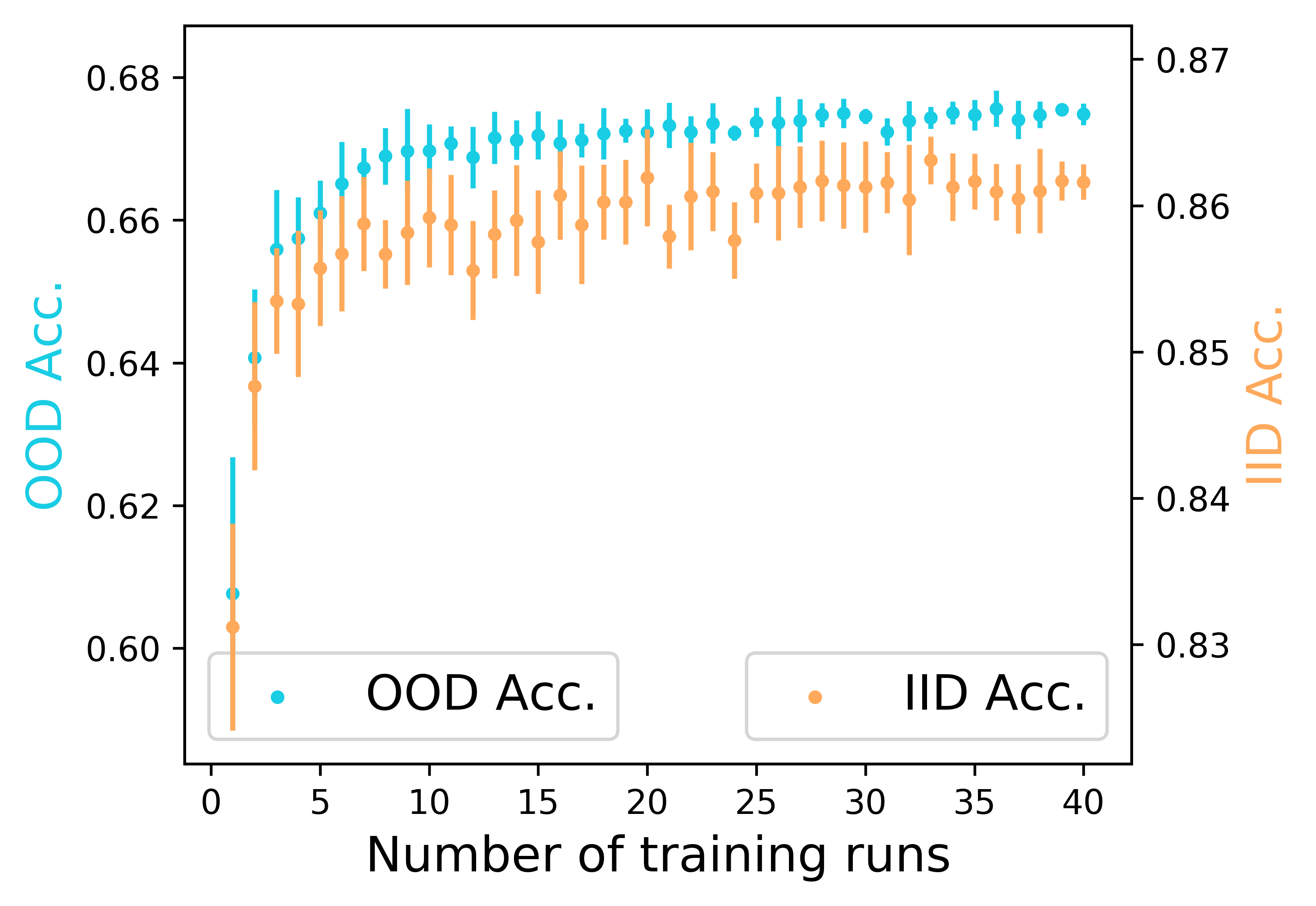}%
    \caption{\reb{\textbf{Mean and standard deviation of DiWA-uniform's accuracy ($\uparrow$)} on OfficeHome when increasing the number of training runs and uniformly averaging all weights. \ood accuracy is computed on domain \enquote{Art}, while \iid accuracy is computed on validation data from the \enquote{Clipart}+\enquote{Product}+\enquote{Photo} domains.}}
    \label{fig:home0_valueofm}%
\end{figure}

\reb{Moreover, in \Cref{tab:m5} we report DiWA's results when considering only $5$ runs, with uniform weight selection. Interestingly, it shows that $M=5$ is enough to be competitive against SWAD \cite{cha2021wad}, the previous state of the art.}
\begin{table}[h]%
    \caption{\reb{\textbf{Accuracy ($\%, \uparrow$) on DomainBed}. DiWA-uniform and LP initialization \cite{kumar2022finetuning}.}}%
    \centering
    \adjustbox{width=1.0\textwidth}{
        \reb{\begin{tabular}{l|cccccc}
            \toprule
            \textbf{Algorithm}               & \textbf{PACS}  & \textbf{VLCS}  & \textbf{OfficeHome} & \textbf{TerraInc} & \textbf{DomainNet} & \textbf{Avg} \\
            \midrule
            SWAD \cite{cha2021wad}           & 88.1 $\pm$ 0.1 & \textbf{79.1} $\pm$ 0.1 & 70.6 $\pm$ 0.2      & 50.0 $\pm$ 0.3    & 46.5 $\pm$ 0.1     & 66.9         \\
            DiWA: $M=5$              & 87.9 $\pm$ 0.2 & 78.3 $\pm$ 0.3 & 71.5 $\pm$ 0.2      & 51.0 $\pm$ 0.7    & 46.9 $\pm$ 0.3     & 67.1         \\
            DiWA: $M=20$             & 88.7 $\pm$ 0.2 & 78.4 $\pm$ 0.2 & 72.1 $\pm$ 0.2      & 51.4 $\pm$ 0.6    & 47.4 $\pm$ 0.2     & 67.6         \\
            DiWA$^{\dagger}$: $M=60$ & \textbf{89.0}           & 78.6           & \textbf{72.8}                & \textbf{51.9}              & \textbf{47.7}               & \textbf{68.0}         \\
            \bottomrule%
        \end{tabular}%
    }}%
    \label{tab:m5}%
\end{table}%

\section{DomainBed}
\label{app:domainbeddetails}

\subsection{Description of the DomainBed benchmark}
\label{app:domainbed_description}

We now further detail our experiments on the DomainBed benchmark \cite{gulrajani2021in}.

\paragraph{Data.} DomainBed includes several computer vision classification datasets divided into multiple domains.
Each domain is successively considered as the test domain while other domains are used in training.
In practice, the data from each domain is split into 80\% (used as training and testing) and 20\% (used as validation for hyperparameter selection) splits. This random process is repeated with 3 different seeds: the reported numbers are the means and the standard errors over these 3 seeds.

\paragraph{Training protocol.}
We follow the training protocol from \url{https://github.com/facebookresearch/DomainBed}.
For each dataset, domain and seed, we perform a random search of $20$ trials on the hyperparameter distributions described in Table \ref{tab:hyperparam}.
Our mild distribution is taken directly from \cite{cha2021wad}, yet could be adapted by dataset for better results.
Even though these distributions are more restricted than the extreme distributions introduced \cite{gulrajani2021in}, our ERM runs perform better.
It leads to a total amount of $2640$ runs only for \Cref{table:db_all_training}.
In \Cref{app:limit_proof_swad}, the $\rho$ hyperparameter for SAM is sampled from $[0.001, 0.002, 0.005, 0.01, 0.02, 0.05]$.
In \Cref{table:db_home0_training}, hyperparameters specific to Interdomain Mixup \cite{yan2020improve} (\enquote{mixup\_alpha}) and Coral \cite{coral216aaai} (\enquote{mmd\_gamma}) are sampled from the distributions defined in \cite{gulrajani2021in}.
We use a ResNet50 \cite{he51deep} pretrained on ImageNet, with a dropout layer before the newly added dense layer and fine-tuned with frozen batch normalization layers. The optimizer is Adam \cite{kingma2014adam}. Our classifier is either initialized randomly or with Linear Probing \cite{kumar2022finetuning}; in the latter case, we first learn only the classifier (with the encoder frozen) with the default hyperparameters defined in \Cref{tab:hyperparam}; the classifier's weights are then used to initialize all subsequent runs.
All runs are trained for 5k steps, except on DomainNet with 15k steps as done in concurrent works \cite{cha2021wad,arpit2021ensemble}. As in \cite{cha2021wad}, validation accuracy is calculated every $50$ steps for VLCS, $500$ steps for DomainNet and $100$ steps for others.

\begin{table}[H]%
    \centering%
    \caption{Hyperparameters, their default values and distributions for random search.}%
    \begin{tabular}{cccc}%
        \toprule
        \multirow{3}{*}{\textbf{Hyperparameter}} & \multirow{3}{*}{\textbf{Default value}} & \multicolumn{2}{c}{\textbf{Random distribution}} \\
                                                 &                                         & Extreme                  & Mild  \\
                                                 &                                          & (DomainBed \cite{gulrajani2021in}) & (DiWA as \cite{cha2021wad}) \\
        \midrule
        Learning rate                            & $5\cdot 10^{-5}$                        & $10^{\U(-5,-3.5)}$                  & $[1,3,5]\cdot 10^{-5}$         \\
        Batch size                               & $32$                                      & $2^{\U(3,5.5)}$                                  & $32$                             \\
        ResNet dropout                           & $0$                                       & $[0,0.1,0.5]$                                    & $[0, 0.1, 0.5]$                \\
        Weight decay                             & $0$                                       & $10^{\U(-6,-2)}$                                 & $[10^{-6}, 10^{-4}]$           \\
        \bottomrule
    \end{tabular}\label{tab:hyperparam}
\end{table}%

\paragraph{Model selection and scores.}
We consider the training-domain validation set protocol.
From each run, we thus take the weights of the epoch with maximum accuracy on the validation dataset --- which follows the training distribution. Our restricted weight selection is also based on this training-domain validation set. This strategy is not possible for DiWA$^\dagger$ as it averages $M=20\times 3$ weights trained with different data splits: they do not share a common validation dataset.
The scores for ERM and Coral are taken from DomainBed \cite{gulrajani2021in}.
Scores for SWAD \cite{cha2021wad} and MA \cite{arpit2021ensemble} are taken from their respective papers.
Note that MA and SWAD perform similarly even though SWAD introduced three additional hyperparameters tuned per dataset: \enquote{an optimum patient parameter, an overfitting patient parameter, and the tolerance rate for searching the start iteration and the end iteration}. Thus we reproduced MA \cite{arpit2021ensemble} which was much easier to implement, and closer to our uniform weight selection.%

\subsection{DomainBed results detailed per domain for each real-world dataset}
\label{app:full_results}

Tables below detail results per domain for the $5$ multi-domain real-world datasets from DomainBed: PACS \cite{li2017deeper}, VLCS \cite{fang2013unbiased}, OfficeHome \cite{venkateswara2017deep}, TerraIncognita \cite{beery2018recognition} and DomainNet \cite{peng2019moment}.
Critically, \cite{ye2021odbench} showed that diversity shift dominates in these datasets.

\begin{table}[h]%
    \caption{\textbf{Accuracy ($\%,\uparrow$) on PACS with ResNet50} (best in \textbf{bold} and second best \underline{underlined}).}%
    \centering
    \adjustbox{max width=\textwidth}{%
        \begin{tabular}{llll|ccccc}
            \toprule
             & \textbf{Algorithm}                                 & \textbf{Weight selection} & \textbf{Init}                                  & \textbf{A}       & \textbf{C}       & \textbf{P}                 & \textbf{S}     & \textbf{Avg}   \\
            \midrule
             & ERM                                                & N/A                       & \multirow{4}{*}{Random}                        & 84.7 $\pm$ 0.4   & 80.8 $\pm$ 0.6   & 97.2 $\pm$ 0.3             & 79.3 $\pm$ 1.0 & 85.5 $\pm$ 0.2 \\
             & Coral\cite{coral216aaai}                           & N/A                       &                                                & 88.3 $\pm$ 0.2   & 80.0 $\pm$ 0.5   & 97.5 $\pm$ 0.3             & 78.8 $\pm$ 1.3 & 86.2 $\pm$ 0.3 \\
             & SWAD \cite{cha2021wad}                             & Overfit-aware             &                                                & 89.3 $\pm$ 0.5   & 83.4 $\pm$ 0.6   & 97.3 $\pm$ 0.3             & 82.5 $\pm$ 0.8 & 88.1 $\pm$ 0.1 \\
             & MA \cite{arpit2021ensemble}                        & Uniform                   &                                                & 89.1 $\pm$ 0.1   & 82.6 $\pm$ 0.2   & 97.6 $\pm$ 0.0             & 80.5 $\pm$ 0.9 & 87.5 $\pm$ 0.2 \\
             & DENS \cite{Lakshminarayanan2017,arpit2021ensemble} & Uniform: $M=6$            &                                                & 88.3             & \underline{83.6} & 96.5                       & 81.9           & 87.6           \\
            \midrule
            \multirow{13}{*}{\begin{turn}{90}Our runs\end{turn}}
             & ERM                                                & N/A                       & \multirow{6}{*}{Random}                        & 87.6 $\pm$ 0.4   & 80.1 $\pm$ 1.5   & 97.7 $\pm$ 0.3             & 76.7 $\pm$ 1.2 & 85.5 $\pm$ 0.5 \\
             & MA \cite{arpit2021ensemble}                        & Uniform                   &                                                & 89.9 $\pm$ 0.1   & 83.3 $\pm$ 0.4   & 97.8 $\pm$ 0.2             & 80.6 $\pm$ 0.3 & 87.9 $\pm$ 0.1 \\
             & ENS                                                & Uniform: $M=20$           &                                                & 88.9 $\pm$ 0.4   & 82.3 $\pm$ 0.5   & 97.4 $\pm$ 0.3             & 83.2 $\pm$ 0.3 & 88.0 $\pm$ 0.1 \\
             & DiWA                                               & Restricted: $M \leq 20$   &                                                & 90.0 $\pm$ 0.3   & 82.0 $\pm$ 0.5   & 97.5 $\pm$ 0.1             & 82.0 $\pm$ 0.6 & 87.9 $\pm$ 0.2 \\
             & DiWA                                               & Uniform: $M=20$           &                                                & 90.1 $\pm$ 0.6   & 83.3 $\pm$ 0.6   & \underline{98.2} $\pm$ 0.1 & 83.4 $\pm$ 0.4 & 88.8 $\pm$ 0.4 \\
             & DiWA$^{\dagger}$                                   & Uniform: $M=60$           &                                                & \underline{90.5} & \textbf{83.7}    & \underline{98.2}           & \textbf{83.8}  & \textbf{89.0}  \\
            \cmidrule{2-9}
             & ERM                                                & N/A                       & \multirow{7}{*}{LP \cite{kumar2022finetuning}} & 86.8 $\pm$ 0.8   & 80.6 $\pm$ 1.0   & 97.4 $\pm$ 0.4             & 78.7 $\pm$ 2.0 & 85.9 $\pm$ 0.6 \\
             & MA \cite{arpit2021ensemble}                        & Uniform                   &                                                & 89.5 $\pm$ 0.1   & 82.8 $\pm$ 0.2   & 97.8 $\pm$ 0.1             & 80.9 $\pm$ 1.3 & 87.8 $\pm$ 0.3 \\
             & ENS                                                & Uniform: $M=20$           &                                                & 89.6 $\pm$ 0.2   & 81.6 $\pm$ 0.3   & 97.8 $\pm$ 0.2             & 83.5 $\pm$ 0.5 & 88.1 $\pm$ 0.3 \\
             & DiWA                                               & Restricted: $M \leq 20$   &                                                & 89.3 $\pm$ 0.2   & 82.8 $\pm$ 0.2   & 98.0 $\pm$ 0.1             & 82.0 $\pm$ 0.9 & 88.0 $\pm$ 0.3 \\
             & DiWA                                               & Uniform: $M=5$            &                                                & 89.9 $\pm$ 0.5   & 82.3 $\pm$ 0.3   & 97.7 $\pm$ 0.4             & 81.7 $\pm$ 0.8 & 87.9 $\pm$ 0.2 \\
             & DiWA                                               & Uniform: $M=20$           &                                                & 90.1 $\pm$ 0.2   & 82.8 $\pm$ 0.6   & \textbf{98.3} $\pm$ 0.1    & 83.3 $\pm$ 0.4 & 88.7 $\pm$ 0.2 \\
             & DiWA$^{\dagger}$                                   & Uniform: $M=60$           &                                                & \textbf{90.6}    & \underline{83.4} & \underline{98.2}           & \textbf{83.8}  & \textbf{89.0}  \\
            \bottomrule
        \end{tabular}}
\end{table}

\begin{table}[h]%
    \caption{\textbf{Accuracy ($\%,\uparrow$) on VLCS with ResNet50} (best in \textbf{bold} and second best \underline{underlined}).}%
    \centering
    \adjustbox{max width=\textwidth}{%
        \begin{tabular}{llll|ccccc}
            \toprule
             & \textbf{Algorithm}                                 & \textbf{Weight selection} & \textbf{Init}                                  & \textbf{C}                 & \textbf{L}              & \textbf{S}                 & \textbf{V}                 & \textbf{Avg}               \\
            \midrule
             & ERM                                                & N/A                       & \multirow{4}{*}{Random}                        & 97.7 $\pm$ 0.4             & 64.3 $\pm$ 0.9          & 73.4 $\pm$ 0.5             & 74.6 $\pm$ 1.3             & 77.5 $\pm$ 0.4             \\
             & Coral\cite{coral216aaai}                           & N/A                       &                                                & 98.3 $\pm$ 0.1             & 66.1 $\pm$ 1.2          & 73.4 $\pm$ 0.3             & 77.5 $\pm$ 1.2             & 78.8 $\pm$ 0.6             \\
             & SWAD \cite{cha2021wad}                             & Overfit-aware             &                                                & 98.8 $\pm$ 0.1             & 63.3 $\pm$ 0.3          & 75.3 $\pm$ 0.5             & \underline{79.2} $\pm$ 0.6 & 79.1 $\pm$ 0.1             \\
             & MA \cite{arpit2021ensemble}                        & Uniform                   &                                                & \textbf{99.0} $\pm$ 0.2    & 63.0 $\pm$ 0.2          & 74.5 $\pm$ 0.3             & 76.4 $\pm$ 1.1             & 78.2 $\pm$ 0.2             \\
             & DENS \cite{Lakshminarayanan2017,arpit2021ensemble} & Uniform: $M=6$            &                                                & 98.7                       & 64.5                    & 72.1                       & 78.9                       & 78.5                       \\
            \midrule
            \multirow{13}{*}{\begin{turn}{90}Our runs\end{turn}}
             & ERM                                                & N/A                       & \multirow{6}{*}{Random}                        & 97.9 $\pm$ 0.5             & 64.2 $\pm$ 0.3          & 73.5 $\pm$ 0.5             & 74.9 $\pm$ 1.2             & 77.6 $\pm$ 0.2             \\
             & MA \cite{arpit2021ensemble}                        & Uniform                   &                                                & 98.5 $\pm$ 0.2             & 63.5 $\pm$ 0.2          & 74.4 $\pm$ 0.8             & 77.3 $\pm$ 0.3             & 78.4 $\pm$ 0.1             \\
             & ENS                                                & Uniform: $M=20$           &                                                & 98.6 $\pm$ 0.1             & \textbf{64.9} $\pm$ 0.2 & 73.5 $\pm$ 0.3             & 77.7 $\pm$ 0.3             & 78.7 $\pm$ 0.1             \\
             & DiWA                                               & Restricted: $M \leq 20$   &                                                & 98.3 $\pm$ 0.1             & 63.9 $\pm$ 0.2          & \underline{75.6} $\pm$ 0.2 & 79.1 $\pm$ 0.3             & \underline{79.2} $\pm$ 0.1 \\
             & DiWA                                               & Uniform: $M=20$           &                                                & 98.4 $\pm$ 0.1             & 63.4 $\pm$ 0.1          & 75.5 $\pm$ 0.3             & 78.9 $\pm$ 0.6             & 79.1 $\pm$ 0.2             \\
             & DiWA$^{\dagger}$                                   & Uniform: $M=60$           &                                                & 98.4                       & 63.3                    & \textbf{76.1}              & \textbf{79.6}              & \textbf{79.4}              \\
            \cmidrule{2-9}
             & ERM                                                & N/A                       & \multirow{7}{*}{LP \cite{kumar2022finetuning}} & 98.1 $\pm$ 0.3             & 64.4 $\pm$ 0.3          & 72.5 $\pm$ 0.5             & 77.7 $\pm$ 1.3             & 78.1 $\pm$ 0.5             \\
             & MA \cite{arpit2021ensemble}                        & Uniform                   &                                                & \underline{98.9} $\pm$ 0.0 & 62.9 $\pm$ 0.5          & 73.7 $\pm$ 0.3             & 78.7 $\pm$ 0.6             & 78.5 $\pm$ 0.4             \\
             & ENS                                                & Uniform: $M=20$           &                                                & 98.5 $\pm$ 0.1             & \textbf{64.9} $\pm$ 0.1 & 73.4 $\pm$ 0.4             & 77.2 $\pm$ 0.4             & 78.5 $\pm$ 0.1             \\
             & DiWA                                               & Restricted: $M \leq 20$   &                                                & 98.4 $\pm$ 0.0             & 64.1 $\pm$ 0.2          & 73.3 $\pm$ 0.4             & 78.1 $\pm$ 0.8             & 78.5 $\pm$ 0.1             \\
             & DiWA                                               & Uniform: $M=5$            &                                                & 98.8 $\pm$ 0.0             & 63.8 $\pm$ 0.5          & 72.9 $\pm$ 0.2             & 77.6 $\pm$ 0.5             & 78.3 $\pm$ 0.3             \\
             & DiWA                                               & Uniform: $M=20$           &                                                & 98.8 $\pm$ 0.1             & 62.8 $\pm$ 0.2          & 73.9 $\pm$ 0.3             & 78.3 $\pm$ 0.1             & 78.4 $\pm$ 0.2             \\
             & DiWA$^{\dagger}$                                   & Uniform: $M=60$           &                                                & \underline{98.9}           & 62.4                    & 73.9                       & 78.9                       & 78.6                       \\
            \bottomrule
        \end{tabular}}
\end{table}

\begin{table}[h]%
    \caption{\textbf{Accuracy ($\%,\uparrow$) on OfficeHome with ResNet50} (best in \textbf{bold} and second best \underline{underlined}).}%
    \centering
    \adjustbox{max width=\textwidth}{%
        \begin{tabular}{llll|ccccc}
            \toprule
             & \textbf{Algorithm}                                 & \textbf{Weight selection} & \textbf{Init}                                  & \textbf{A}                 & \textbf{C}       & \textbf{P}                 & \textbf{R}                 & \textbf{Avg}               \\
            \midrule
             & ERM                                                & N/A                       & \multirow{4}{*}{Random}                        & 61.3 $\pm$ 0.7             & 52.4 $\pm$ 0.3   & 75.8 $\pm$ 0.1             & 76.6 $\pm$ 0.3             & 66.5 $\pm$ 0.3             \\
             & Coral\cite{coral216aaai}                           & N/A                       &                                                & 65.3 $\pm$ 0.4             & 54.4 $\pm$ 0.5   & 76.5 $\pm$ 0.1             & 78.4 $\pm$ 0.5             & 68.7 $\pm$ 0.3             \\
             & SWAD \cite{cha2021wad}                             & Overfit-aware             &                                                & 66.1 $\pm$ 0.4             & 57.7 $\pm$ 0.4   & 78.4 $\pm$ 0.1             & 80.2 $\pm$ 0.2             & 70.6 $\pm$ 0.2             \\
             & MA \cite{arpit2021ensemble}                        & Uniform                   &                                                & 66.7 $\pm$ 0.5             & 57.1 $\pm$ 0.1   & 78.6 $\pm$ 0.1             & 80.0 $\pm$ 0.0             & 70.6 $\pm$ 0.1             \\
             & DENS \cite{Lakshminarayanan2017,arpit2021ensemble} & Uniform: $M=6$            &                                                & 65.6                       & 58.5             & 78.7                       & 80.5                       & 70.8                       \\
            \midrule
            \multirow{13}{*}{\begin{turn}{90}Our runs\end{turn}}
             & ERM                                                & N/A                       & \multirow{6}{*}{Random}                        & 62.9 $\pm$ 1.3             & 54.0 $\pm$ 0.2   & 75.7 $\pm$ 0.9             & 77.0 $\pm$ 0.8             & 67.4 $\pm$ 0.6             \\
             & MA \cite{arpit2021ensemble}                        & Uniform                   &                                                & 65.0 $\pm$ 0.2             & 57.9 $\pm$ 0.3   & 78.5 $\pm$ 0.1             & 79.7 $\pm$ 0.1             & 70.3 $\pm$ 0.1             \\
             & ENS                                                & Uniform: $M=20$           &                                                & 66.1 $\pm$ 0.1             & 57.0 $\pm$ 0.3   & 79.0 $\pm$ 0.2             & 80.0 $\pm$ 0.1             & 70.5 $\pm$ 0.1             \\
             & DiWA                                               & Restricted: $M \leq 20$   &                                                & 66.7 $\pm$ 0.1             & 57.0 $\pm$ 0.3   & 78.5 $\pm$ 0.3             & 79.9 $\pm$ 0.3             & 70.5 $\pm$ 0.1             \\
             & DiWA                                               & Uniform: $M=20$           &                                                & 67.3 $\pm$ 0.2             & 57.9 $\pm$ 0.2   & 79.0 $\pm$ 0.2             & 79.9 $\pm$ 0.1             & 71.0 $\pm$ 0.1             \\
             & DiWA$^{\dagger}$                                   & Uniform: $M=60$           &                                                & 67.7                       & \underline{58.8} & 79.4                       & 80.5                       & 71.6                       \\
            \cmidrule{2-9}
             & ERM                                                & N/A                       & \multirow{7}{*}{LP \cite{kumar2022finetuning}} & 63.9 $\pm$ 1.2             & 54.8 $\pm$ 0.6   & 78.7 $\pm$ 0.1             & 80.4 $\pm$ 0.2             & 69.4 $\pm$ 0.2             \\
             & MA \cite{arpit2021ensemble}                        & Uniform                   &                                                & 67.4 $\pm$ 0.4             & 57.3 $\pm$ 0.9   & 79.7 $\pm$ 0.1             & \underline{81.7} $\pm$ 0.6 & 71.5 $\pm$ 0.3             \\
             & ENS                                                & Uniform: $M=20$           &                                                & 67.0 $\pm$ 0.1             & 57.9 $\pm$ 0.4   & \underline{80.0} $\pm$ 0.2 & \underline{81.7} $\pm$ 0.3 & 71.7 $\pm$ 0.1             \\
             & DiWA                                               & Restricted: $M \leq 20$   &                                                & 67.8 $\pm$ 0.5             & 57.2 $\pm$ 0.5   & 79.6 $\pm$ 0.1             & 81.4 $\pm$ 0.4             & 71.5 $\pm$ 0.2             \\
             & DiWA                                               & Uniform: $M=5$            &                                                & 68.4 $\pm$ 0.4             & 57.4 $\pm$ 0.5   & 79.2 $\pm$ 0.2             & 80.9 $\pm$ 0.4             & 71.5 $\pm$ 0.3             \\
             & DiWA                                               & Uniform: $M=20$           &                                                & \underline{68.4} $\pm$ 0.2 & 58.2 $\pm$ 0.5   & \underline{80.0} $\pm$ 0.1 & \underline{81.7} $\pm$ 0.3 & \underline{72.1} $\pm$ 0.2 \\
             & DiWA$^{\dagger}$                                   & Uniform: $M=60$           &                                                & \textbf{69.2}              & \textbf{59.0}    & \textbf{80.6}              & \textbf{82.2}              & \textbf{72.8}              \\
            \bottomrule
        \end{tabular}}
\end{table}

\begin{table}[h]%
    \caption{\textbf{Accuracy ($\%,\uparrow$) on TerraIncognita with ResNet50} (best in \textbf{bold} and second best \underline{underlined}).}%
    \centering
    \adjustbox{max width=\textwidth}{%
        \begin{tabular}{llll|ccccc}
            \toprule
             & \textbf{Algorithm}                                 & \textbf{Weight selection} & \textbf{Init}                                  & \textbf{L100}              & \textbf{L38}               & \textbf{L43}            & \textbf{L46}               & \textbf{Avg}               \\
            \midrule
             & ERM                                                & N/A                       & \multirow{4}{*}{Random}                        & 49.8 $\pm$ 4.4             & 42.1 $\pm$ 1.4             & 56.9 $\pm$ 1.8          & 35.7 $\pm$ 3.9             & 46.1 $\pm$ 1.8             \\
             & Coral\cite{coral216aaai}                           & N/A                       &                                                & 51.6 $\pm$ 2.4             & 42.2 $\pm$ 1.0             & 57.0 $\pm$ 1.0          & 39.8 $\pm$ 2.9             & 47.6 $\pm$ 1.0             \\
             & SWAD \cite{cha2021wad}                             & Overfit-aware             &                                                & 55.4 $\pm$ 0.0             & 44.9 $\pm$ 1.1             & 59.7 $\pm$ 0.4          & 39.9 $\pm$ 0.2             & 50.0 $\pm$ 0.3             \\
             & MA \cite{arpit2021ensemble}                        & Uniform                   &                                                & 54.9 $\pm$ 0.4             & 45.5 $\pm$ 0.6             & 60.1 $\pm$ 1.5          & 40.5 $\pm$ 0.4             & 50.3 $\pm$ 0.5             \\
             & DENS \cite{Lakshminarayanan2017,arpit2021ensemble} & Uniform: $M=6$            &                                                & 53.0                       & 42.6                       & 60.5                    & 40.8                       & 49.2                       \\
            \midrule
            \multirow{13}{*}{\begin{turn}{90}Our runs\end{turn}}
             & ERM                                                & N/A                       & \multirow{6}{*}{Random}                        & 56.3 $\pm$ 2.9             & 43.1 $\pm$ 1.6             & 57.1 $\pm$ 1.0          & 36.7 $\pm$ 0.7             & 48.3 $\pm$ 0.8             \\
             & MA \cite{arpit2021ensemble}                        & Uniform                   &                                                & 53.2 $\pm$ 0.4             & 46.3 $\pm$ 1.0             & 60.1 $\pm$ 0.6          & 40.2 $\pm$ 0.8             & 49.9 $\pm$ 0.2             \\
             & ENS                                                & Uniform: $M=20$           &                                                & 56.4 $\pm$ 1.5             & 45.3 $\pm$ 0.4             & \textbf{61.0} $\pm$ 0.3 & \underline{41.4} $\pm$ 0.5 & 51.0 $\pm$ 0.5             \\
             & DiWA                                               & Restricted: $M \leq 20$   &                                                & 55.6 $\pm$ 1.5             & 47.5 $\pm$ 0.5             & 59.5 $\pm$ 0.5          & 39.4 $\pm$ 0.2             & 50.5 $\pm$ 0.5             \\
             & DiWA                                               & Uniform: $M=20$           &                                                & 52.2 $\pm$ 1.8             & 46.2 $\pm$ 0.4             & 59.2 $\pm$ 0.2          & 37.8 $\pm$ 0.6             & 48.9 $\pm$ 0.5             \\
             & DiWA$^{\dagger}$                                   & Uniform: $M=60$           &                                                & 52.7                       & 46.3                       & 59.0                    & 37.7                       & 49.0                       \\
            \cmidrule{2-9}
             & ERM                                                & N/A                       & \multirow{7}{*}{LP \cite{kumar2022finetuning}} & \textbf{59.9} $\pm$ 4.2    & 46.9 $\pm$ 0.9             & 54.6 $\pm$ 0.3          & 40.1 $\pm$ 2.2             & 50.4 $\pm$ 1.8             \\
             & MA \cite{arpit2021ensemble}                        & Uniform                   &                                                & 54.6 $\pm$ 1.4             & 48.6 $\pm$ 0.4             & 59.9 $\pm$ 0.7          & \textbf{42.7} $\pm$ 0.8    & 51.4 $\pm$ 0.6             \\
             & ENS                                                & Uniform: $M=20$           &                                                & 55.6 $\pm$ 1.4             & 45.4 $\pm$ 0.4             & \textbf{61.0} $\pm$ 0.4 & 41.3 $\pm$ 0.3             & 50.8 $\pm$ 0.5             \\
             & DiWA                                               & Restricted: $M \leq 20$   &                                                & \underline{58.5} $\pm$ 2.2 & 48.2 $\pm$ 0.3             & 58.5 $\pm$ 0.3          & 41.1 $\pm$ 1.2             & \underline{51.6} $\pm$ 0.9 \\
             & DiWA                                               & Uniform: $M=5$            &                                                & 56.0 $\pm$ 2.5             & 48.9 $\pm$ 0.8             & 58.4 $\pm$ 0.2          & 40.6 $\pm$ 0.8             & 51.0 $\pm$ 0.7             \\
             & DiWA                                               & Uniform: $M=20$           &                                                & 56.3 $\pm$ 1.9             & \underline{49.4} $\pm$ 0.7 & 59.9 $\pm$ 0.4          & 39.8 $\pm$ 0.5             & 51.4 $\pm$ 0.6             \\
             & DiWA$^{\dagger}$                                   & Uniform: $M=60$           &                                                & 57.2                       & \textbf{50.1}              & 60.3                    & 39.8                       & \textbf{51.9}              \\
            \bottomrule
        \end{tabular}}
\end{table}

\begin{table}[h]%
    \caption{\textbf{Accuracy ($\%,\uparrow$) on DomainNet with ResNet50} (best in \textbf{bold} and second best \underline{underlined}).}%
    \centering
    \adjustbox{max width=\textwidth}{%
        \begin{tabular}{llll|ccccccc}
            \toprule
             & \textbf{Algorithm}                                 & \textbf{Weight selection} & \textbf{Init}                                  & \textbf{clip}              & \textbf{info}           & \textbf{paint}             & \textbf{quick}             & \textbf{real}              & \textbf{sketch}            & \textbf{Avg}            \\
            \midrule
             & ERM                                                & N/A                       & \multirow{4}{*}{Random}                        & 58.1 $\pm$ 0.3             & 18.8 $\pm$ 0.3          & 46.7 $\pm$ 0.3             & 12.2 $\pm$ 0.4             & 59.6 $\pm$ 0.1             & 49.8 $\pm$ 0.4             & 40.9 $\pm$ 0.1          \\
             & Coral\cite{coral216aaai}                           & N/A                       &                                                & 59.2 $\pm$ 0.1             & 19.7 $\pm$ 0.2          & 46.6 $\pm$ 0.3             & 13.4 $\pm$ 0.4             & 59.8 $\pm$ 0.2             & 50.1 $\pm$ 0.6             & 41.5 $\pm$ 0.1          \\
             & SWAD \cite{cha2021wad}                             & Overfit-aware             &                                                & 66.0 $\pm$ 0.1             & 22.4 $\pm$ 0.3          & 53.5 $\pm$ 0.1             & 16.1 $\pm$ 0.2             & 65.8 $\pm$ 0.4             & 55.5 $\pm$ 0.3             & 46.5 $\pm$ 0.1          \\
             & MA \cite{arpit2021ensemble}                        & Uniform                   &                                                & 64.4 $\pm$ 0.3             & 22.4 $\pm$ 0.2          & 53.4 $\pm$ 0.3             & 15.4 $\pm$ 0.1             & 64.7 $\pm$ 0.2             & 55.5 $\pm$ 0.1             & 46.0 $\pm$ 0.1          \\
             & DENS \cite{Lakshminarayanan2017,arpit2021ensemble} & Uniform: $M=6$            &                                                & \textbf{68.3}              & 23.1                    & 54.5                       & \underline{16.3}           & 66.9                       & \textbf{57.0}              & \textbf{47.7}           \\
            \midrule
            \multirow{13}{*}{\begin{turn}{90}Our runs\end{turn}}
             & ERM                                                & N/A                       & \multirow{6}{*}{Random}                        & 62.6 $\pm$ 0.4             & 21.6 $\pm$ 0.3          & 50.4 $\pm$ 0.1             & 13.8 $\pm$ 0.2             & 63.6 $\pm$ 0.4             & 52.5 $\pm$ 0.4             & 44.1 $\pm$ 0.1          \\
             & MA \cite{arpit2021ensemble}                        & Uniform                   &                                                & 64.5 $\pm$ 0.2             & 22.7 $\pm$ 0.1          & 53.8 $\pm$ 0.1             & 15.6 $\pm$ 0.1             & 66.0 $\pm$ 0.1             & 55.7 $\pm$ 0.1             & 46.4 $\pm$ 0.1          \\
             & ENS                                                & Uniform: $M=20$           &                                                & \underline{67.3} $\pm$ 0.4 & 22.9 $\pm$ 0.1          & 54.2 $\pm$ 0.2             & 15.5 $\pm$ 0.2             & 67.7 $\pm$ 0.2             & \underline{56.7} $\pm$ 0.2 & 47.4 $\pm$ 0.2          \\
             & DiWA                                               & Restricted: $M \leq 20$   &                                                & 65.2 $\pm$ 0.3             & 23.0 $\pm$ 0.3          & 54.0 $\pm$ 0.1             & 15.9 $\pm$ 0.1             & 66.2 $\pm$ 0.1             & 55.5 $\pm$ 0.1             & 46.7 $\pm$ 0.1          \\
             & DiWA                                               & Uniform: $M=20$           &                                                & 63.4 $\pm$ 0.2             & 23.1 $\pm$ 0.1          & 53.9 $\pm$ 0.2             & 15.4 $\pm$ 0.2             & 65.5 $\pm$ 0.2             & 55.1 $\pm$ 0.2             & 46.1 $\pm$ 0.1          \\
             & DiWA$^{\dagger}$                                   & Uniform: $M=60$           &                                                & 63.5                       & \textbf{23.3}           & 54.3                       & 15.6                       & 65.7                       & 55.3                       & 46.3                    \\
            \cmidrule{2-11}
             & ERM                                                & N/A                       & \multirow{7}{*}{LP \cite{kumar2022finetuning}} & 63.4 $\pm$ 0.2             & 21.1 $\pm$ 0.4          & 50.7 $\pm$ 0.3             & 13.5 $\pm$ 0.4             & 64.8 $\pm$ 0.4             & 52.4 $\pm$ 0.1             & 44.3 $\pm$ 0.2          \\
             & MA \cite{arpit2021ensemble}                        & Uniform                   &                                                & 64.8 $\pm$ 0.1             & 22.3 $\pm$ 0.0          & 54.2 $\pm$ 0.1             & 16.0 $\pm$ 0.1             & 67.4 $\pm$ 0.0             & 55.2 $\pm$ 0.1             & 46.6 $\pm$ 0.0          \\
             & ENS                                                & Uniform: $M=20$           &                                                & 66.7 $\pm$ 0.4             & 22.2 $\pm$ 0.1          & 54.1 $\pm$ 0.2             & 15.1 $\pm$ 0.2             & \underline{68.4} $\pm$ 0.1 & 55.7 $\pm$ 0.2             & 47.0 $\pm$ 0.2          \\
             & DiWA                                               & Restricted: $M \leq 20$   &                                                & 66.7 $\pm$ 0.2             & \textbf{23.3} $\pm$ 0.2 & \underline{55.3} $\pm$ 0.1 & \underline{16.3} $\pm$ 0.2 & 68.2 $\pm$ 0.0             & 56.2 $\pm$ 0.1             & \textbf{47.7} $\pm$ 0.1 \\
             & DiWA                                               & Uniform: $M=5$            &                                                & 65.7 $\pm$ 0.5             & 22.6 $\pm$ 0.2          & 54.4 $\pm$ 0.4             & 15.5 $\pm$ 0.5             & 67.7 $\pm$ 0.0             & 55.5 $\pm$ 0.4             & 46.9 $\pm$ 0.3          \\
             & DiWA                                               & Uniform: $M=20$           &                                                & 65.9 $\pm$ 0.4             & 23.0 $\pm$ 0.2          & 55.0 $\pm$ 0.3             & 16.1 $\pm$ 0.2             & \underline{68.4} $\pm$ 0.1 & 55.7 $\pm$ 0.4             & 47.4 $\pm$ 0.2          \\
             & DiWA$^{\dagger}$                                   & Uniform: $M=60$           &                                                & 66.2                       & \textbf{23.3}           & \textbf{55.4}              & \textbf{16.5}              & \textbf{68.7}              & 56.0                       & \textbf{47.7}           \\
            \bottomrule
        \end{tabular}}
\end{table}

\FloatBarrier

\section{Failure of WA under correlation shift on ColoredMNIST}
\label{app:failure_corr_shift}
Based on \Cref{eq:b_var_cov}, we explained that WA is efficient when variance dominates; we showed in \Cref{subsec:expression_ood_var} that this occurs under diversity shift.
This is confirmed by our state-of-the-art results in \Cref{table:db_all_training} and \Cref{app:full_results} on PACS, OfficeHome, VLCS, TerraIncognita and DomainNet.
In contrast, we argue that WA is inefficient when bias dominates, \ie in the presence of correlation shift (see \Cref{subsec:expression_ood_bias}).
We verify this failure on the ColoredMNIST \cite{arjovsky2019invariant} dataset, which is dominated by correlation shift \cite{peng2019moment}.

Colored MNIST is a colored variant of the MNIST handwritten digit classification dataset \cite{lecun2010mnist} where the correlation strengths between color and label vary across domains.
We follow the protocol described in \Cref{app:domainbed_description} except that (1) we used the convolutional neural network architecture introduced in DomainBed \cite{gulrajani2021in} for MNIST experiments and (2) we used the test-domain model selection in addition to the train-domain model selection. Indeed, as stated in \cite{ye2021odbench}, \enquote{it may be improper to apply training-domain validation to datasets dominated by correlation shift since under the influence of spurious correlations, achieving excessively high accuracy in the training environments often leads to low accuracy in novel test environments}.

In \Cref{table:db_cmnist_training,table:db_cmnist}, we observe that DiWA-uniform and MA both perform poorly compared to ERM.
Note that DiWA-restricted does not degrade ERM as it selects only a few models for averaging (low $M$).
This confirms that our approach is useful to tackle diversity shift but not correlation shift, for which invariance-based approaches as IRM \cite{arjovsky2019invariant} or Fishr \cite{rame_fishr_2021} remain state-of-the-art.%
\begin{table}[h]%
    \caption{\textbf{Accuracy ($\%,\uparrow$) on ColoredMNIST}. WA does not improve performance under correlation shift. Random initialization of the classifier. Training-domain model selection.}%
    \centering
    \vspace{0.5em}
    \adjustbox{width=0.8\textwidth}{
        \begin{tabular}{lll|cccc}
            \toprule
                                                       & \textbf{Algorithm}               & \textbf{Weight selection} & \textbf{+90\%}             & \textbf{+80\%}             & \textbf{-90\%}             & \textbf{Avg}              \\
            \midrule
                                                       & ERM                              & N/A                       & 71.7 $\pm$ 0.1             & 72.9 $\pm$ 0.2             & 10.0 $\pm$ 0.1             & 51.5  $\pm$ 0.1           \\
                                                       & Coral \cite{coral216aaai}        & N/A                       & 71.6 $\pm$ 0.3             & 73.1 $\pm$ 0.1             & 9.9 $\pm$ 0.1              & 51.5  $\pm$ 0.1           \\
                                                       & IRM \cite{arjovsky2019invariant} & N/A                       & \textbf{72.5} $\pm$ 0.1    & \underline{73.3} $\pm$ 0.5 & 10.2 $\pm$ 0.3             & \textbf{52.0} $\pm$ 0.1   \\
                                                       & Fishr \cite{rame_fishr_2021}     & N/A                       & \underline{72.3} $\pm$ 0.9 & \textbf{73.5} $\pm$ 0.2    & 10.1 $\pm$ 0.2             & \textbf{52.0}   $\pm$ 0.2 \\
            \midrule
            \multirow{6}{*}{\begin{turn}{90}Our runs\end{turn}} & ERM                              & N/A                       & 71.5 $\pm$ 0.4             & 73.3 $\pm$ 0.2             & \underline{10.3} $\pm$ 0.2 & 51.7 $\pm$ 0.2            \\
                                                       & MA \cite{arpit2021ensemble}      & Uniform                   & 68.9 $\pm$ 0.0             & 71.8 $\pm$ 0.1             & 10.0 $\pm$ 0.1             & 50.3 $\pm$ 0.0            \\
                                                       & ENS                              & Uniform: $M=20$           & 71.0 $\pm$ 0.2             & 72.9 $\pm$ 0.2             & 9.9 $\pm$ 0.0              & 51.3  $\pm$ 0.1           \\
                                                       & DiWA                             & Restricted: $M \leq 20$   & 71.3 $\pm$ 0.2             & 72.9 $\pm$ 0.1             & 10.0 $\pm$ 0.1             & 51.4 $\pm$ 0.1            \\
                                                       & DiWA                             & Uniform: $M=20$           & 69.1 $\pm$ 0.8             & 72.6 $\pm$ 0.4             & \textbf{10.6} $\pm$ 0.1    & 50.8 $\pm$ 0.4            \\
                                                       & DiWA$^{\dagger}$                 & Uniform: $M=60$           & 69.3                       & 72.3                       & \underline{10.3}           & 50.6                      \\
            \bottomrule
        \end{tabular}}
    \label{table:db_cmnist_training}%
\end{table}

\begin{table}[h]%
    \caption{\textbf{Accuracy ($\%,\uparrow$) on ColoredMNIST}. WA does not improve performance under correlation shift. Random initialization of the classifier. Test-domain model selection.}%
    \centering
    \vspace{0.5em}
    \adjustbox{width=0.8\textwidth}{
        \begin{tabular}{lll|cccc}
            \toprule
                                                       & \textbf{Algorithm}               & \textbf{Weight selection} & \textbf{+90\%}             & \textbf{+80\%}             & \textbf{-90\%}             & \textbf{Avg}                \\
            \midrule
                                                       & ERM                              & N/A                       & 71.8 $\pm$ 0.4             & 72.9 $\pm$ 0.1             & 28.7 $\pm$ 0.5             & 57.8 $\pm$ 0.2              \\
                                                       & Coral \cite{coral216aaai}        & N/A                       & 71.1 $\pm$ 0.2             & 73.4 $\pm$ 0.2             & 31.1 $\pm$ 1.6             & 58.6 $\pm$ 0.5              \\
                                                       & IRM \cite{arjovsky2019invariant} & N/A                       & \underline{72.0} $\pm$ 0.1 & 72.5 $\pm$ 0.3             & \underline{58.5} $\pm$ 3.3 & \underline{67.7}  $\pm$ 1.2 \\
                                                       & Fishr \cite{rame_fishr_2021}     & N/A                       & \textbf{74.1} $\pm$ 0.6    & 73.3 $\pm$ 0.1             & \textbf{58.9} $\pm$ 3.7    & \textbf{68.8} $\pm$ 1.4     \\
            \midrule
            \multirow{6}{*}{\begin{turn}{90}Our runs\end{turn}} & ERM                              & N/A                       & 71.5 $\pm$ 0.3             & \textbf{74.1} $\pm$ 0.4    & 21.5 $\pm$ 1.9             & 55.7 $\pm$ 0.4              \\
                                                       & MA \cite{arpit2021ensemble}      & Uniform                   & 68.8 $\pm$ 0.2             & 72.1 $\pm$ 0.2             & 10.2 $\pm$ 0.0             & 50.4 $\pm$ 0.1              \\
                                                       & ENS                              & Uniform: $M=20$           & 71.0 $\pm$ 0.2             & 72.9 $\pm$ 0.2             & 9.9 $\pm$ 0.0              & 51.3  $\pm$ 0.1             \\
                                                       & DiWA                             & Restricted: $M \leq 20$   & 71.9 $\pm$ 0.4             & \underline{73.6} $\pm$ 0.2 & 21.5 $\pm$ 1.9             & 55.7 $\pm$ 0.8              \\
                                                       & DiWA                             & Uniform: $M=20$           & 69.1 $\pm$ 0.8             & 72.6 $\pm$ 0.4             & 10.6 $\pm$ 0.1             & 50.8 $\pm$ 0.4              \\
                                                       & DiWA$^{\dagger}$                 & Uniform: $M=60$           & 69.3                       & 72.3                       & 10.3                       & 50.6                        \\
            \bottomrule
        \end{tabular}}
    \label{table:db_cmnist}%
\end{table}


\FloatBarrier
\section{\reb{Last layer retraining when some target data is available}}
\label{app:llr}

\reb{The traditional OOD generalization setup does not provide access to target samples (labelled or unlabelled). The goal is to learn a model able to generalize to any kind of distributions. This is arguably the most challenging generalization setup: under these strict conditions, we showed that DiWA outperforms other approaches on DomainBed. Yet, in real-world applications, some target data is often available for training; moreover, last layer retraining on these target samples was shown highly efficient in \cite{kirichenko2022last,rosenfeld2022domain}. The complete analysis of DiWA for this new scenario should be properly addressed in future work; yet, we now hint that a DiWA strategy could be helpful.}

\reb{Specifically, in \Cref{table:db_home0_llr}, we consider that after a first training phase on the \enquote{Clipart}, \enquote{Product} and \enquote{Photo} domains, we eventually have access to some samples from the target \enquote{Art} domain (20\% or 80\% of the whole domain). Following \cite{kirichenko2022last}, we re-train only the last layer of the network on these samples before testing.
We observe improved performance when the (frozen) feature extractor was obtained via DiWA (from the first stage) rather than from ERM. It suggests that features extracted by DiWA are more adapted to last layer retraining/generalization than those of ERM. In conclusion, we believe our DiWA strategy has great potential for many real-world applications, whether some target data is available for training or not.}
\begin{table}[h]%
    \caption{\reb{\textbf{Accuracy ($\uparrow$)} on domain \enquote{Art} from OfficeHome when some target samples are available for last layer retraining (LLR) \cite{kirichenko2022last}. The feature extractor is either pre-trained only on ImageNet (\xmark), fine-tuned on the source domains \enquote{Clipart}, \enquote{Product} and \enquote{Photo} (ERM), or obtained by averaging multiple runs on these source domains (DiWA-uniform $M=20$).}}
    \centering
    \vspace{0.5em}
    \adjustbox{width=0.8\textwidth}{
        \reb{\begin{tabular}{l|lll}
            \toprule
            \multirow{2}{*}{Training on source domains} & \multicolumn{3}{c}{LLR on target domain (\% domain in training)}                                                     \\
                                                                                                  & \xmark (0\%)                                                    & \checkmark (20\%)       & \checkmark (80\%)       \\
            \midrule
            \xmark                                                                                & -                                                               & 61.2 $\pm$ 0.6          & 74.4 $\pm$ 1.2          \\
            ERM                                                                                   & 62.9 $\pm$ 1.3                                                  & 68.0 $\pm$ 0.7          & 74.7 $\pm$ 0.6          \\
            DiWA & \textbf{67.3} $\pm$ 0.3                                         & \textbf{70.4} $\pm$ 0.1 & \textbf{78.1} $\pm$ 0.6 \\
            \bottomrule
        \end{tabular}}}
    \label{table:db_home0_llr}%
\end{table}


\end{appendices}
\end{document}